\documentclass[11pt,reqno,twoside]{article}

\usepackage[hang]{footmisc}
\usepackage{lipsum}

\usepackage{cmap} 

\usepackage[T1]{fontenc}
\usepackage[utf8]{inputenc}
\usepackage{graphicx}
\usepackage{placeins}
\usepackage{enumerate}
\usepackage{algcompatible}
\usepackage{comment}

\usepackage{verbatim}
\newcommand{\comments}[1]{}

\usepackage{soul}

\usepackage{setspace}

\let\counterwithin\relax  
\usepackage{lmodern} 

\usepackage{bm} 

\usepackage{bbold}
\usepackage[normalem]{ulem}

\usepackage{amsmath,amsbsy,amsgen,amscd,amsthm,amsfonts,amssymb} 

\usepackage[centering,top=1in,bottom=1in,left=0.85in,right=0.85in]{geometry}

\newcommand{\dlat}{d_{\rm lat}}
\newcommand{\din}{d_{\rm in}}
\newcommand{\dout}{d_{\rm out}}

\newcommand{\doutput}{{d_*}}

\newcommand{\dimz}{d_z}
\newcommand{\mInd}{j} 

\newcommand{\ndim}{\tilde{d}} 
\newcommand{\DcoefficientBound}{C_{\mathcal{D}}}
\newcommand{\potential}{\mathcal{W}}
\newcommand{\potentialDer}{w}
\newcommand{\potentialNonlin}{\tilde{\potentialDer}}
\usepackage[sf,bf,compact]{titlesec}

\usepackage[usenames,dvipsnames]{xcolor}

\definecolor{dark-gray}{gray}{0.3}
\definecolor{dkgray}{rgb}{.4,.4,.4}
\definecolor{dkblue}{rgb}{0,0,.5}
\definecolor{medblue}{rgb}{0,0,.75}
\definecolor{rust}{rgb}{0.5,0.1,0.1}

\usepackage{url}
\usepackage[colorlinks=true]{hyperref}
\hypersetup{linkcolor=dkblue}    
\hypersetup{citecolor=rust}      
\hypersetup{urlcolor=rust}     

\usepackage[final]{microtype} 

\newtheoremstyle{myThm} 
    {\topsep}                    
    {\topsep}                    
    {\itshape}                   
    {}                           
    {\sffamily\bfseries}                   
    {.}                          
    {.5em}                       
    {}  

\newtheoremstyle{myRem} 
    {\topsep}                    
    {\topsep}                    
    {}                   
    {}                           
    {\sffamily}                   
    {.}                          
    {.5em}                       
    {}  

\newtheoremstyle{myDef} 
    {\topsep}                    
    {\topsep}                    
    {}                   
    {}                           
    {\sffamily\bfseries}                   
    {.}                          
    {.5em}                       
    {}  

\theoremstyle{myThm}
\newtheorem{theorem}{Theorem}[section]
\newtheorem{lemma}[theorem]{Lemma}
\newtheorem{proposition}[theorem]{Proposition}
\newtheorem{corollary}[theorem]{Corollary}

\newtheorem{condition}[theorem]{Condition}
\newtheorem{definition}[theorem]{Definition}

\theoremstyle{myRem}
\newtheorem{remarkx}[theorem]{Remark}
 \newenvironment{remark}
  {\pushQED{\qed}\remarkx}
  {\popQED\endremarkx}
\usepackage{fancyhdr}
\let\originalleft\left
\let\originalright\right
\renewcommand{\left}{\mathopen{}\mathclose\bgroup\originalleft}
\renewcommand{\right}{\aftergroup\egroup\originalright}

\newcommand{\timestep}{\tau}
\newcommand{\linearOp}{\mathcal{B}}
\newcommand{\linearOpClass}{\mathsf{B}}
\newcommand{\nonlinearity}{\mathcal{G}}
\newcommand{\nonlinearClass}{\mathsf{G}}
\newcommand{\coefficientBound}{C_{\nonlinearity}}
\newcommand{\forcingClass}{\mathsf{F}}
\newcommand{\forcingBound}{F}
\newcommand{\forcing}{f}

\newcommand{\indices}{\mathsf{J}} 
\newcommand{\findices}{\mathsf{K}} 
\newcommand{\mlpBound}{\zeta} 
\newcommand{\inputBall}{U} 
\newcommand{\cutoff}{\mathscr{C}}
\newcommand{\samples}{M} 
\newcommand{\sample}{m} 
\newcommand{\outputBall}{V}
\newcommand{\spectralStep}{\mathcal{S}_{\timestep,N}}
\newcommand{\spectralClass}{\mathsf{S}_{\timestep,N}}
\newcommand{\pdeClass}{\mathsf{S}}
\newcommand{\polyClass}{\mathsf{POLY}}
\newcommand{\PDE}{\mathcal{S}}
\newcommand{\inductionBound}{\alpha}

\newcommand{\smoothness}{a}

\newcommand{\sobolevBall}{\mathsf{U}_\smoothness}
\newcommand{\risk}{\mathcal{L}}
\newcommand{\prob}{\mathbb{P}}
\newcommand{\covering}{\mathsf{M}}
\newcommand{\confidence}{t}
\newcommand{\pLH}{\mathcal{P}^{ \rm  LH} }
\newcommand{\metric}{\mathsf{d}}

\newcommand{\smoothClass}{W^{\alpha,\infty}}

\newcommand{\dissipation}{\beta}
\newcommand{\inverseBound}{c}
\newcommand{\power}{\eta}
\newcommand{\differentialClass}{\mathsf{D}}
\newcommand{\extension}{\tilde{\potentialDer}}
\usepackage{mathtools}
\mathtoolsset{centercolon}  

\newcommand{\overbar}[1]{\mkern 1.5mu\overline{\mkern-1.5mu#1\mkern-1.5mu}\mkern 1.5mu}

\renewcommand{\phi}{\varphi}
\newcommand{\eps}{\varepsilon}

\providecommand{\mathbbm}{\mathbb} 

\newcommand{\R}{\mathbbm{R}}
\newcommand{\E}{\mathbbm{E}}
\newcommand{\C}{\mathbbm{C}}
\newcommand{\N}{\mathbbm{N}}
\newcommand{\Z}{\mathbbm{Z}}

\newcommand{\G}{\mathcal{G}}
\newcommand{\F}{\mathcal{F}}
\renewcommand{\L}{\mathcal{L}}
\newcommand{\B}{\mathcal{B}}

\newcommand{\D}{\mathcal{D}}

\newcommand{\red}{\color{red}}

\definecolor{mygreen}{rgb}{0.13,0.55,0.13}

\newcommand{\nc}{\normalcolor}

\newcommand{\T}{\mathbbm{T}}

\newcommand{\Nc}{\mathcal{N}}

\usepackage[font = small, margin=30pt]{caption}

\usepackage[]{algorithm}
\usepackage{enumerate}

\usepackage{graphicx}

\usepackage{authblk}
\usepackage{chngcntr}
\usepackage{mathrsfs}
\counterwithin{table}{section}
\counterwithin{algorithm}{section}
\counterwithin{equation}{section}
\usepackage{enumitem}

\usepackage{tikz}
\usetikzlibrary{shapes,arrows,decorations.pathmorphing,backgrounds,positioning,fit,petri}
\usepackage{tikz}
\usetikzlibrary{arrows.meta, positioning, decorations.pathreplacing}
\usetikzlibrary{calc}
\usetikzlibrary{matrix}
\usetikzlibrary{backgrounds}
\usetikzlibrary{shapes.geometric}
\usepackage{pgfplots}
\pgfplotsset{compat=newest}
\usepgfplotslibrary{groupplots}
\pgfdeclarelayer{background}
\pgfdeclarelayer{bBack}
\pgfsetlayers{bBack,background,main}
\usetikzlibrary{fit}

\newcommand{\nathan}[1]{{\color{red} #1}}

\newcommand{\iid}{\stackrel{\mathsf{i.i.d.}}{\sim}}

\title{\huge From Spectral Methods to Sample Complexity Bounds for Fourier Neural Operators}

\author{N. Chandramoorthy, D. Sanz-Alonso, and N. Waniorek}

\date{University of Chicago}

\vspace{.25in}

\makeatletter\@addtoreset{section}{part}\makeatother%

\DeclareMathOperator*{\argmin}{arg\,min}

\begin{document}
\maketitle

\begin{abstract}
We establish approximation and learning guarantees for Fourier neural operators (FNOs) applied to time-$T$ solution operators of dissipative evolution equations. 
The analysis builds on the premise that FNOs can efficiently approximate and learn solution operators whenever these operators admit stable and accurate spectral discretizations.
To formalize this idea, we introduce classes of evolution operators defined through spectral methods and derive FNO approximation bounds and polynomial sample complexity guarantees for these classes. 
For equations with polynomial nonlinearities, the learning rates depend primarily on the smoothness of the input space and the dimension of the physical domain. Our results hold uniformly over broad families of  dissipative equations, rather than for a single fixed PDE, and apply in particular to the Navier--Stokes, Allen--Cahn, and Cahn--Hilliard equations. For equations with non-polynomial smooth nonlinearities, we prove that polynomial sample complexity still holds with rates that now additionally depend on the smoothness of the nonlinear terms and the dissipation strength. Overall, we connect classical spectral approximation theory with modern operator learning and explain when FNOs can learn nonlinear evolution operators efficiently.
\end{abstract}

\section{Introduction}
Fourier neural operators (FNOs) have emerged as a powerful architecture for operator learning, where the goal is to learn maps between infinite-dimensional function spaces from data. Many important operators in scientific computing arise as solution maps of evolution equations: an input function $u_0$, representing the initial state of the system, is mapped to the corresponding state $u(T)$ at a later time $T$. In this paper, we focus on time-$T$ solution maps $u_0\mapsto u(T)$ 
associated with dissipative evolution equations. Despite the empirical success of FNOs, a key theoretical question remains: under what structural conditions can FNOs approximate and learn such solution operators efficiently?

This paper addresses this question through a guiding principle rooted in numerical analysis: we show that if a solution operator admits a stable and accurate Fourier spectral discretization, then it can also be represented efficiently by an FNO and learned from data with polynomial sample complexity. Thus, rather than treating solution operators as arbitrary maps between function spaces, we exploit the structure inherited from numerical schemes for the underlying evolution equation. This perspective identifies stable spectral approximability as a broad structural condition that enables FNOs to learn nonlinear evolution operators efficiently.

To make this principle precise, we study solution operators generated by dissipative evolution equations on the $d$-dimensional torus $\T^d$ of the form
$$
\frac{du}{dt}+\mathcal{A}u+\mathcal{D}_s\bigl(\nonlinearity(u)\bigr)=\forcing,\qquad t>0, 
$$
with initial condition $u(0) = u_0$. Here $\mathcal{A}$ is a dissipative linear operator diagonalizable in the Fourier basis, $\mathcal{D}_s$ is a differential operator, $\nonlinearity$ is a nonlinear reaction or transport term, and $\forcing$ is a forcing term. For fixed $T>0$, the associated time-$T$ solution operator is the map $\PDE: u_0 \mapsto u(T)$. Our goal is to prove quantitative approximation and learning guarantees for such operators using FNOs.

The analysis proceeds in four steps. First, we introduce classes of semi-implicit Fourier spectral schemes that discretize these evolution equations in space and time. Second, we define classes of solution operators that can be approximated by iterating stable and accurate spectral steps. Third, we construct FNOs that emulate these spectral steps and compose them to approximate the full time-$T$ solution map. This yields quantitative FNO approximation bounds. Finally, we combine these approximation estimates with metric-entropy bounds to establish learning guarantees for empirical risk minimization over suitable FNO classes.

A key feature of our results is that they hold uniformly over operator classes, rather than for a single fixed partial differential equation (PDE). For equations with polynomial nonlinearities, the learning rates depend primarily on the Sobolev smoothness of the input space and the dimension of the physical domain. We verify that the framework applies to the time-$T$ solution maps of the Navier--Stokes, Allen--Cahn, and Cahn--Hilliard equations. We also extend the theory to nonlinearities satisfying general smoothness assumptions. In this more general setting, polynomial sample complexity still holds, and the rates additionally depend on the smoothness of the nonlinearity and the strength of dissipation.

\subsection{Related Works}
For overviews of the operator learning literature, we refer to \cite{kovachki2024operator,boulle2024mathematical,subedi2025operator,nelsen2025operator} and the references therein. Here, we review relevant literature on FNOs and statistical guarantees for operator learning.
\paragraph{Fourier Neural Operators}
The FNO architecture was introduced in \cite{DBLP:conf/iclr/LiKALBSA21} for learning PDE solution operators over the $d$-dimensional torus. Further variants have been proposed to handle more complicated geometries \cite{bonev2023spherical,li2023fourier,liu2023domain}. FNOs have been used effectively in applications such as numerical weather prediction \cite{pathak2022fourcastnet}, biomedical design \cite{zhou2024ai}, plasma modeling \cite{gopakumar2023fourier}, and carbon capture and storage \cite{wen2022u}. On the theory side, FNOs were shown to have the universal approximation property and to efficiently approximate the Darcy flow and Navier--Stokes solution operators in \cite{kovachki2021universal}. Approximation error bounds for general operator learning architectures, including FNOs, were shown for nonlinear parabolic PDEs in \cite{de2022generic}. The work \cite{furuyaquantitative} derives quantitative approximation guarantees with neural operators for a class of nonlinear evolution equations defined via Picard iteration. The discretization error of FNOs was characterized in \cite{lanthaler2024discretization}, and bounds on the Rademacher complexity of FNOs were established in \cite{kim2024bounding}. 
In the special case of linear FNOs, \cite{subedi2024controlling} decomposes the error into statistical, discretization, and truncation components and bounds each component for a particular least-squares estimator.

\paragraph{Statistical Guarantees for Operator Learning} 
The sample complexity of operator learning tasks has received significant attention in a variety of settings. For holomorphic operators, it has been shown that the number of samples required to approximate to $\eps$ accuracy scales algebraically in $\eps^{-1}$ \cite{adcock2024optimal,adcock2025optimal}. In contrast, learning $k$-times Fr\'echet differentiable or Lipschitz operators requires exponentially many samples in the worst case. This negative result holds for input measures with support on a compact set \cite{kovachki2024data} and with Gaussian input measures \cite{adcock2024learning}. Consequently, understanding the success of operator learning requires developing assumptions beyond those of classical smoothness. Operator Barron spaces \cite{korolev2022two} may be one such assumption; however, sample complexity guarantees for such operators have yet to be developed.

Several positive results are known for more structured operator classes. Optimal rates for learning linear operators between Sobolev spaces are established in \cite{chen2026optimal}. Restricting to the case of diagonalizable linear operators, \cite{de2023convergence} characterizes the sample complexity and provides generalization bounds under distribution shifts. Random feature methods have been shown to achieve Monte Carlo estimation rates if the target operator lies in a reproducing kernel Hilbert space \cite{lanthaler2023error}. In \cite{liu2024deep}, the authors again consider learning Lipschitz operators and demonstrate that classes of encoder-decoder networks achieve the optimal, but slow, logarithmic convergence rate. They further show that, if the input function space is restricted to be finite-dimensional, then the sample complexity scales polynomially in the dimension of the input space.

Other work focuses on guarantees specifically for learning PDEs. For elliptic PDEs, \cite{boulle2023elliptic} proves that the required sample complexity is logarithmic, meaning that the error decreases exponentially fast with the sample size. It is important to note, however, that \cite{boulle2023elliptic} does not consider noisy observational data. Instead, the authors assume access to PDE queries with deliberately chosen randomized inputs and corresponding exact solution outputs. For linear parabolic PDEs and in a similar setting, \cite{boulle2022learning} demonstrates an algorithm achieving a polynomial sample complexity that is increasing in the dimension of the physical domain. In the more challenging setting of linear hyperbolic PDEs, \cite{wang2023operator} proves that polynomially many samples can still suffice.
Convergence rates for learning pseudo-differential operators from noisy input-output data are derived in \cite{chen2026convergence}.
There has been less work on nonlinear PDEs. In \cite{chen2023deep}, the authors demonstrate that the solution operator for viscous Burgers' equation, along with several linear PDEs, can be learned with polynomial sample complexity. The work most similar to our own is \cite{chen2025error}, which considers learning time-stepping schemes for evolution equations with an encoder-decoder neural network. In contrast to our work, \cite{chen2025error} assumes that the training data is generated from the numerical methods, not from the true PDE.

\paragraph{Spectral Methods} 
We refer to \cite{guo1998spectral} for a classical reference on spectral methods. More recently, and of particular relevance to our results, there has been analysis of so-called ``large time-stepping'' implicit-explicit (IMEX) methods for dissipative PDEs that do not require restrictive CFL conditions. For the 2D Navier--Stokes equations, the IMEX Euler scheme was shown to be convergent with a step-size requirement independent of the spatial discretization for initial data in $H^2$ \cite{he2005stability}, and with mild dependence on the spatial discretization for initial data in $L^2$ or $H^1$ \cite{he2008euler}. Similar convergence guarantees were derived for a stabilized IMEX Euler scheme for the Cahn--Hilliard equation in \cite{he2008stability}. However, these bounds are conditional on the stabilization term being large enough relative to the $L^{\infty}$ norm of the numerical solution, which is not known a priori. 

In \cite{li2021stability}, the authors demonstrate that an IMEX Euler scheme for the Cahn--Hilliard equation with logarithmic potentials is unconditionally energy stable. They develop a bootstrap technique to obtain a priori control over the $L^{\infty}$ norm of the numerical solution, which in turn yields unconditional large time-stepping guarantees. Their analysis leverages recent tools from harmonic analysis \cite{bourgain2015strong}. These techniques have subsequently been used to obtain unconditional large time-stepping guarantees for the standard Cahn--Hilliard equation \cite{li2022stability} and the Allen--Cahn equation \cite{cheng2025energy}.


\subsection{Main Contributions and Outline}
We now summarize the main contributions of the paper and indicate where they are established. All required background on FNOs will be provided in Section \ref{sec:FNOs}.
\begin{itemize}
\item \textbf{A spectral-method-based operator class.}
In Section~\ref{sec:polynomialnonlinearity}, we introduce a class of time-$T$ solution operators defined through stable and accurate semi-implicit Fourier spectral approximations. This class is designed to capture the numerical structure shared by many dissipative evolution equations, while remaining broad enough to include several standard nonlinear PDEs.

\item \textbf{Uniform FNO approximation bounds for polynomial nonlinearities.}  
For equations whose nonlinearities are polynomial, we prove that every operator in the resulting class can be uniformly approximated by an FNO. The construction emulates one step of the spectral method and then composes the resulting FNOs to approximate the time-$T$ solution map. The resulting architecture sizes are quantified in terms of the desired accuracy and the parameters defining the operator class. This is proved in Theorem~\ref{thm: polynomial nonlinearity approximation theorem}.

\item \textbf{Polynomial sample complexity for FNO empirical risk minimization.}  
We combine the constructive approximation bounds with metric-entropy estimates to prove learning guarantees for empirical risk minimization over FNO classes. In the polynomial-nonlinearity setting, the resulting rates depend on the input space smoothness and the physical dimension. This yields polynomial sample complexity for learning all operators in the class; see Theorem~\ref{thm: learning}. 

\item \textbf{Applications to classical dissipative PDEs.}  
In Section~\ref{sec:Illustrative Examples}, we verify that the time-$T$ solution maps of the Navier--Stokes, Allen--Cahn, and Cahn--Hilliard equations fit into the abstract framework and apply Theorem \ref{thm: learning}. This gives concrete learning guarantees for these PDEs and illustrates how the assumptions arise from standard spectral stability and convergence estimates.

\item \textbf{Extension to general smooth nonlinearities.}  
In Section~\ref{sec:smoothnonlinearity}, we extend the framework beyond polynomial nonlinearities. Under non-parametric smoothness assumptions on $\nonlinearity$, we again obtain FNO approximation bounds and polynomial sample complexity. The resulting rates are generally slower and depend additionally on the smoothness of the nonlinearity and the strength of the dissipation; see Theorems~\ref{thm: smooth nonlinearity approximation theorem} and~\ref{thm: smooth learning bound}. In Section~\ref{ssec:examplegeneralnonlinearity}, we apply this theory to learn the time-$T$ solution map of the Cahn-Hilliard equation with a logarithmic potential. This equation presents additional difficulties for approximation and learning due to the possibility of singularities in the potential.

\end{itemize}

\section{Fourier Neural Operators}\label{sec:FNOs}
This section introduces the class of FNOs considered in this paper. Intuitively, an FNO $\Psi$ with $L$ hidden layers  is a nonlinear map $\Psi: L_N^2 (\T^d, \R^{\din}) \to L_N^2(\T^d, \R^{\dout})$ of the form 
  \begin{align*}
        \Psi = \mathcal{Q} \circ \L_L\circ \cdots \circ \L_1 \circ \mathcal{R}, 
    \end{align*}
where the operator $\mathcal{R}$ lifts the space $L_N^2 (\T^d, \R^{\din})$ of $\R^{\din}$-valued degree-$N$ trigonometric polynomials on the torus $\T^d = [0,2\pi]^d$   to a latent space $L_N^2 (\T^d, \R^{\dlat}),$ the hidden layers $\mathcal{L}_\ell$, $1 \le \ell \le L$, map the latent space to itself, and $\mathcal{Q}$ maps the latent space to $L_N^2(\T^d, \R^{\dout}).$ In Section~\ref{ssec:preliminaries}, we introduce the function spaces and operators involved in the lifting and projection layers $\mathcal{R}$ and $\mathcal{Q}$, as well as the Fourier hidden layers $\mathcal{L}_\ell$. We then formally introduce the class of FNOs in Section~\ref{ssec:FNOs}.

\subsection{Background: Function Spaces and Operators}\label{ssec:preliminaries}
FNOs perform finite-dimensional function approximation through finite Fourier representations. The basic idea is to represent functions by finitely many Fourier modes, perform operations on these finite representations, and then move between Fourier coefficients and grid values using the discrete Fourier transform. Hence, we begin by defining finite-dimensional function spaces consisting of finite Fourier series. Throughout, we use $N \in \N$ to denote the Fourier cutoff for these spaces.
For $N\in \N$, let $\findices_N :=\{k\in \Z^d:\|k\|_{\infty}\leq N\}$ and let
\begin{align}\label{eq:trigpoly}
   L^2_N(\T^d) : = \biggl\{v_N:\T^d\to \R \, \Big\vert \,   v_N(x)=\sum_{k \in \findices_N}c_ke^{i\langle k,x\rangle} \, 
   \forall x \in \T^d,\,  c_{-k}=c_k^* \,  \,  \forall k \in \findices_N \biggr\}
\end{align}
 be the space of real-valued trigonometric polynomials of degree $N$ equipped with the usual $L^2$ inner product. The condition $c_{-k}=c_k^*$ in \eqref{eq:trigpoly} guarantees that $v_N(x) \in \R$ for all $x\in \T^d.$ The space $L_N^2(\T^d, \R^{\doutput})$ of $\R^{\doutput}$-valued  trigonometric polynomials is defined as the space of functions from $\T^d$ to $\R^\doutput$  with components in $L^2_N(\T^d).$
 
Consider the set of indices $\indices_N=\{0,1,\hdots,2N\}^d$ and the corresponding uniform grid on the torus $\{x_j\}_{j\in \indices_N}$, where $x_j=\frac{2\pi j}{2N+1}$.
Note that $|\indices_N|=|\findices_N|=(2N+1)^d.$ The discrete Fourier transform (DFT) is the map $\F_N:L^2_N(\T^d,\R^{\doutput})\to \C^{(2N+1)^d\times \doutput},$ which assigns to 
$v_N \in L^2_N(\T^d,\R^{\doutput})$ the matrix $\bigl(\F_N v_N(k) \bigr)_{k \in \findices_N} \in \C^{(2N+1)^d\times \doutput}$ with rows  
\begin{align}\label{eq:DFT}
    (\F_N v_N)(k):=\frac{1}{(2N+1)^d}\sum_{j\in \indices_N}v_N(x_j)e^{-i \langle x_j , k\rangle} =: \hat{v}_N(k) \in \C^\doutput, \quad \forall\: k \in \findices_N.
\end{align}
Denoting $\vec{v}_N :=\{v_N(x_j)\}_{j\in  \indices_N},$ the DFT can be written as a  map (matrix)  ${\bf F} :\R^{(2N+1)^d\times \doutput}\to \C^{(2N+1)^d\times \doutput}$ from function values at the grid points to Fourier coefficients, given row-wise by
\begin{align*}
    ({\bf F} \vec{v}_N)(k) = \frac{1}{(2N+1)^d}\sum_{j\in \indices_N}v_N(x_j)e^{-i \langle x_j , k\rangle} =:\hat{v}_N(k) \in \C^\doutput, \quad \forall k \in \findices_N.
\end{align*}
The inverse DFT $\F_N^{-1}:\C^{(2N+1)^d\times \doutput}\to L^2_N(\T^d,\C^{\doutput})$ is given by
\begin{align*}
    (\F_N^{-1} \hat{v}_N)(x)=\sum_{k\in \findices_N}\hat{v}_N(k)e^{i \langle k,x\rangle}, \quad \forall x\in \T^d.
\end{align*}
As a map from conjugate-symmetric Fourier coefficients to function values on the uniform grid $\{x_j\}_{j\in \indices_N}$, we write the inverse DFT as ${\bf F} ^{-1}:\C^{(2N+1)^d\times \doutput}\to \R^{(2N+1)^d\times \doutput}.$ We have that, row-wise,  
\begin{align*}
    ({\bf F} ^{-1}\hat{v}_N)(j) :=\sum_{k\in \findices_N}\hat{v}_N(k)e^{i \langle k,x_j \rangle} \in \R^\doutput, \quad \forall j\in \indices_N.
\end{align*}

We next introduce an interpolation operator to formalize the usual pseudo-spectral identification between a function and its values on the uniform grid. Specifically, we let  $\mathcal{I}_N:C(\T^d)\to L_N^2(\T^d)$ be the pseudo-spectral interpolation operator mapping $v \in C(\T^d)$ to the unique trigonometric polynomial $\mathcal{I}_N v\in L^2_N(\T^d)$ that satisfies
\begin{align}\label{eq:interpolationoperator}
    (\mathcal{I}_N v)(x_j)=v(x_j), \quad \forall j\in \indices_N.
\end{align}
Extending the DFT to $C(\T^d)$ by the same formula as in \eqref{eq:DFT}, we can write, for $v \in C(\T^d),$ 
\begin{equation*}
    \mathcal{I}_Nv=\F_N^{-1}\F_Nv.
\end{equation*}
For vector-valued functions, 
the operator $\mathcal{I}_N$
is applied component-wise. Note that for $v_N\in L^2_N(\T^d,\R^{\doutput})$, $\mathcal{I}_Nv_N=v_N.$ Hence, we can identify functions $v_N\in L^2_N(\T^d,\R^{\doutput})$ with their values at the grid points $\{x_j\}_{j \in \indices_N}.$
See \cite{guo1998spectral} for a detailed discussion of Fourier pseudo-spectral approximation.

\subsection{Fourier Neural Operators}\label{ssec:FNOs}

    Let $\din,\dout,\dlat \in \N$ be given.
    A \emph{Fourier neural operator} (FNO) with $L\in \N$ hidden layers 
    is a nonlinear map 
    $\Psi:L^2_N(\T^d,\R^{\din})\to L^2_N(\T^d,\R^{\dout})$
    of the form
    \begin{align*}
        \Psi = \mathcal{Q} \circ \L_L\circ \cdots \circ \L_1 \circ \mathcal{R}, 
    \end{align*}
    where:

    \begin{enumerate}
        \item $\mathcal{R}: L^2_N(\T^d,\R^{\din})\to L^2_N(\T^d,\R^{\dlat})$ and $\mathcal{Q}:L^2_N(\T^d,\R^{\dlat})\to L^2_N(\T^{d}, \R^{\dout})$ are linear ``lifting'' and ``projection'' layers mapping $u \in L^2_N(\T^d,\R^{\din})$ and $v \in L^2_N(\T^d,\R^{\dlat})$ as follows: 
    \begin{align*}
        (\mathcal{R}u)(x)= {\bf R} u(x), \quad \forall x\in \T^d, \\ (\mathcal{Q}v)(x)={\bf{Q}} v(x), \quad \forall x\in \T^d,
    \end{align*}
where ${\bf R}\in \R^{\dlat\times \din}$ and ${\bf Q}\in \R^{\dout\times \dlat}$ are learnable matrices.
\item   $\L_{\ell}:L^2_N(\T^d,\R^{\dlat})\to L^2_N(\T^d,\R^{\dlat})$, for $1 \le \ell \le L,$ are nonlinear hidden layers mapping $v \in L^2_N(\T^d,\R^{\dlat})$ as follows:
    $$
    \L_{\ell}(v)(x)=\mathcal{I}_{N}\sigma\bigl( \mathcal{W}_{\ell}v+ \mathcal{K}_{\ell}(v)+b_{\ell}\bigr)(x), \quad \forall x \in \T^d.
    $$
    Here: 
    \begin{enumerate}
    \item $ \mathcal{W}_{\ell}:L^2_N(\T^d,\R^{\dlat})\to L^2_N(\T^d,\R^{\dlat})$ is linear and maps $v \in L^2_N(\T^d,\R^{\dlat})$ as follows:
    \begin{align*}
        (\mathcal{W}_{\ell} v)(x)={\bf W}_{\ell}v(x), \quad \forall x\in \T^d,
    \end{align*}
    where ${\bf W}_{\ell}\in \R^{\dlat\times \dlat}$ is a learnable matrix. 
    \item $\mathcal{K}_{\ell}:L^2_N(\T^d,\R^{\dlat})\to L^2_N(\T^d,\R^{\dlat})$ is a nonlocal linear operator mapping $v \in L^2_N(\T^d,\R^{\dlat})$ as follows:
    $$
    (\mathcal{K}_{\ell}v)(x)=\Bigl(\F^{-1}_N\bigl(\hat{P}_{\ell}(\F_N v)\bigr)\Bigr)(x)=\sum_{k\in \findices_N}\hat{P}_{\ell}(k)\hat{v}(k)e^{i \langle k,x\rangle},
    $$
    where $\F_N$ is the DFT and $\F_N^{-1}$ its inverse, and the $\hat{P}_\ell(k)\in \C^{\dlat\times \dlat}$ for $k\in \findices_N$ are a learnable sequence of matrices required to satisfy $\hat{P}_{\ell}(-k)=\overbar{\hat{P}_{\ell}(k)}$ for all $k\in \findices_N.$ 
    \item $b_{\ell}\in L^2_N(\T^d,\R^{\dlat})$ is a learnable function. 
    \item  $\sigma:L^2_N(\T^d,\R^{\dlat})\to C(\T^d,\R^{\dlat})$ is a ReLU activation function, which is applied component-wise and point-wise; that is, for $v \in L^2_N(\T^d,\R^{\dlat}),$
    \begin{align*}
        \sigma(v)(x)=\begin{bmatrix}
            \max\{0,v_1(x)\}\\ 
            \vdots \\
            \max\{0,v_{\dlat}(x)\}
        \end{bmatrix} \,, \qquad \forall x\in \T^d.
    \end{align*}
     \item $\mathcal{I}_N:C(\T^d,\R^\dlat)\to L_N^2(\T^d,\R^\dlat)$ is the pseudo-spectral projection defined in Section~\ref{ssec:preliminaries}.
    \end{enumerate}
    \end{enumerate}
    We have defined FNOs as acting on functions in $L^2_N(\T^d,\R^{\din}).$ For a general continuous function, an FNO acts on its trigonometric interpolant: $\Psi(u)=\Psi(\mathcal{I}_N u)$ for $u\in C(\T^d,\R^{\din})$. 
We characterize the size of an FNO $\Psi$ through the following quantities:
\begin{enumerate}
    \item The \textit{depth} $\mathscr{D}(\Psi):=L;$
    \item The \textit{channel width} $\mathscr{W}(\Psi):=\max \bigl\{\din,\dout,\dlat \bigr\};$
    \item The \textit{weight magnitude} $\mathscr{S}(\Psi):=\max_{k\in \findices_N,1\leq \ell \leq L}\{\|{\bf R}\|_{\infty},\|{\bf Q}\|_{\infty}, \|{\bf W}_{\ell}\|_{\infty},\|\hat{P}_k^{\ell}\|_{\infty},\,\|\hat{b}^{\ell}_k\|_{\infty}\}$; 
    \item The \textit{Fourier cutoff}, $\cutoff(\Psi) = N$.
\end{enumerate}
We define $\mathsf{FNO}(\mathscr{D},\mathscr{W},\mathscr{S},\mathscr{C})$ to be the set of FNOs with depth $\mathscr{D}$, width $\mathscr{W}$, and Fourier cutoff $\mathscr{C}$ with weights constrained to be bounded by $\mathscr{S}$.  We will refer to the vector of all trainable parameters in an FNO as $\theta=\text{vec}({\bf R}, {\bf W}_1,\hat{P}_1^1,\hat{b}_1^1,\hdots,\hat{P}_{ (2N+1)^d }^L,\hat{b}_{(2N+1)^d}^L,{\bf Q})\in \R^{d_{\theta}}$, where $d_{\theta}\leq 5(2N+1)^dL\dlat^2$. 

\section{Dissipative Models with Polynomial Nonlinearity}\label{sec:polynomialnonlinearity}
This section investigates how well FNOs can approximate and learn the time-$T$ solution operator for evolution equations on the $d$-dimensional torus $\T^d$ of the form
\begin{align}\label{eq:evolution PDE again}
\begin{split}
    \frac{du}{dt}+\mathcal{A}u+\mathcal{D}_s\bigl(\nonlinearity (u)\bigr)
    &=\forcing, \qquad  t>0. 
\end{split}
\end{align}
  We assume $\mathcal{A}$ to be a \textit{dissipative} linear operator diagonalizable in the Fourier basis,
  \begin{align}
      \mathcal{A}e^{i\langle k,x\rangle}=\lambda_k(\mathcal{A})e^{i\langle k, x\rangle}, \quad \forall k\in \Z^{d},
  \end{align}
  with eigenvalues scaling as 
  \begin{align}\label{eq:dissipation scaling}
        \lambda_k(\mathcal{A})\asymp \|k\|_{2}^{\dissipation}, \quad \forall k\in \Z^{d},  
    \end{align}
    for some $\dissipation>0.$ For instance, $\mathcal{A}$ may be a negative Laplacian $(\dissipation = 2)$, a bilaplacian $(\dissipation = 4)$, or a fractional Laplacian.  We assume that the nonlinearity $\nonlinearity$ is a degree-$p$ \textit{polynomial} acting on $u$, and that $\mathcal{D}_s$ is a differential operator with each coordinate derivative order bounded by $s.$ General smooth nonlinearities, not necessarily of polynomial type, will be considered in Section~\ref{sec:smoothnonlinearity}.

    A key observation is that equations of the form  \eqref{eq:evolution PDE again} can be efficiently solved by spectral methods. This motivates us to introduce in Section~\ref{ssec:spectralpoly} a class of operators that admit accurate spectral approximations. For this operator class, which includes solution maps of many equations of the form \eqref{eq:evolution PDE again},
    we establish FNO approximation bounds in Section~\ref{ssec: polynomial approximation section} and learning bounds in Section~\ref{ssec:mainresultslearning}. 
Applying these general results, we obtain in Section~\ref{sec:Illustrative Examples} learning bounds for the solution operators of three important evolution PDEs: the Navier--Stokes, Allen--Cahn, and Cahn--Hilliard equations.

\subsection{Spectral Methods and Class of Operators}\label{ssec:spectralpoly}
In this section, we introduce the general class of operators for which we will obtain FNO approximation and learning bounds. 
This class, defined in Subsection~\ref{ssec:operatorclasspoly},  consists of operators that can be well approximated by the family of spectral methods introduced in Subsection~\ref{ssec:spectralclasspoly}. We begin with background on spectral methods in Subsection~\ref{ssec:background}. 

\subsubsection{Background on Spectral Methods}\label{ssec:background}
    To solve an equation of the form \eqref{eq:evolution PDE again} numerically, one must discretize the problem in both space and time.
    Semi-implicit Fourier spectral methods are widely used and broadly applicable to dissipative equations of the form \eqref{eq:evolution PDE again}. A first-order-in-time semi-implicit method is given by
        \begin{align}\label{eq:generic spectral method unstabilized}
\frac{1}{\tau}\bigl(u_N^{j+1}-u_N^j\bigr)+\mathcal{A}u_N^{j+1}+ \mathcal{P}_N\mathcal{D}_s\bigl(\nonlinearity (u_N^j)\bigr)-f_N=0,
\end{align}
for a spatial resolution $N\in \N$ and timestep $\tau>0.$ Here $\mathcal{P}_N$ denotes an $L^2$-orthogonal projection onto $L^2_N(\T^d,\R^{\doutput})$ or onto a shift-invariant subspace of $L^2_N(\T^d,\R^{\doutput})$. The subspaces used in our examples in Section~\ref{sec:Illustrative Examples} include spaces of zero-mean functions and divergence-free functions. Our learning-theoretic guarantees in Section~\ref{ssec:mainresultslearning} do not require a priori knowledge of the subspace. Finally, we define $f_N=\mathcal{P}_Nf$ to be the $L^2$-orthogonal projection of the forcing function $f$ onto the same subspace. 

Notice that in \eqref{eq:generic spectral method unstabilized}, the linear dissipation term is treated implicitly, while the nonlinearity is treated explicitly. Semi-implicit methods typically have better stability than explicit methods, while avoiding the fixed-point iterations required by fully implicit methods \cite{chen1998applications}. For both theoretical and practical reasons, semi-implicit methods are often implemented with an additional stabilization term controlled by parameters $\gamma\geq 0$ and $\power\in [0,1),$ leading to the update
    \begin{align}\label{eq:generic spectral method}
\frac{1}{\tau}\bigl(u_N^{j+1}-u_N^j\bigr)+\gamma\mathcal{A}^{\power}\bigl(u_N^{j+1}-u_N^j\bigr)+\mathcal{A}u_N^{j+1}+ \mathcal{P}_N\mathcal{D}_s\bigl(\nonlinearity (u_N^j)\bigr)-\forcing_N=0.
\end{align}
Equation~\ref{eq:generic spectral method} defines a mapping $u_N^j\mapsto u_N^{j+1}=:\spectralStep (u_N^j).$
Thus, for any $u_N \in L^2_N(\T^d,\R^{\doutput}),$ we define 
\begin{align}\label{eq:spec}
\begin{split}
    \spectralStep (u_N) :&=\Bigl(\frac{1}{\timestep}I+\gamma\mathcal{A}^{\power}+\mathcal{A}\Bigr)^{-1}\Bigl((\frac{1}{\timestep}I+\gamma\mathcal{A}^{\power})u_N-\mathcal{P}_N\mathcal{D}_s\bigl(\nonlinearity (u_N)\bigr)+f_N\Bigr)\\
    & = \linearOp  \Bigl(\bigl(\frac{1}{\timestep}I+\gamma \mathcal{A}^{\power}\bigr)u_N- \mathcal{D}_s\bigl(\nonlinearity (u_N)\bigr)+\forcing_N\Bigr),
\end{split}
\end{align}
where $\linearOp := \Bigl(\frac{1}{\timestep}I +\gamma\mathcal{A}^{\power}+
    \mathcal{A}\Bigr)^{-1} \mathcal{P}_N.$
We refer to $\spectralStep$ as a single step of the spectral method. To approximate the time-$T$ solution operator of the PDE, one can choose $J \in \N $ and $\tau>0$ such that 
 $J=T/\timestep,$ and then iterate this single-step map:
\begin{align*}
    u(T) = \mathcal{S}(u) \approx \spectralStep^J(\mathcal{I}_Nu(0))=\underbrace{\spectralStep\circ \cdots \circ \spectralStep}_{J \text{ times}}(u_N^0).
\end{align*}
In the next subsection, we introduce the class of single-step spectral methods considered in this work. Our operator class, introduced in Subsection~\ref{ssec:operatorclasspoly}, consists of operators that can be well approximated by iterating these single-step spectral methods.

\subsubsection{Family of Single-Step Spectral Methods}\label{ssec:spectralclasspoly}  
Recall the generic single-step spectral method $\spectralStep$ introduced in \eqref{eq:spec}. In this subsection, we define a family of single-step spectral methods by imposing conditions on the linear operator $\mathcal{B},$ the lower-order differential operator $\mathcal{D}_s,$ the nonlinearity $\nonlinearity,$ and the forcing $\forcing_N$.



For timestep $\timestep>0$ and degree $N \in \N,$ we introduce parameter $\vartheta:=( \dissipation,\gamma,\power,\inverseBound, s,\DcoefficientBound,p, \coefficientBound, \forcingBound) $ and define the family of single-step spectral methods 
    \begin{align*} 
\spectralClass(\vartheta)
    :=\Biggl\{\spectralStep:L^2_N(\T^d,\R^{\doutput})\to L^2_N(\T^d,\R^{\doutput}) \,  \bigg| &   \, \spectralStep(u_N)=\linearOp  \Bigl(\bigl(\frac{1}{\timestep}I+\gamma\mathcal{A}^{\power} \bigr)u_N-\mathcal{D}_s\bigl(\nonlinearity(u_N)\bigr)+\forcing_N\Bigr),  \\   &\hspace{-1.3cm}\linearOp \in \linearOpClass(\dissipation,\gamma,\power,\inverseBound), \, 
    \mathcal{D}_s\in \differentialClass(s,\DcoefficientBound), \,
    \nonlinearity\in \nonlinearClass(p,\coefficientBound), \,  \, \forcing_N\in \forcingClass(\forcingBound) \Biggr\}.
\end{align*}  

In the above, $\linearOpClass(\dissipation,\gamma,\power,\inverseBound)$ is a class of linear operators, $\mathsf{D}(s,\DcoefficientBound)$ is a class of differential operators, $\nonlinearClass(p,\coefficientBound)$ is a class of polynomial nonlinearities, and $\forcingClass(\forcingBound)$ is a class of forcings. Before formally defining these classes, we provide an intuitive interpretation of their parameters. 
The parameters $(\dissipation,\gamma,\power,\inverseBound)$ in $\linearOpClass$ control respectively the rate of dissipation, the strength of the stabilization term, the order of the stabilization term, and an upper bound on the inverse term; the parameters $(s,\DcoefficientBound)$ in $\mathsf{D}$ represent the order of derivatives and a bound on their coefficients; the parameters $(p,\coefficientBound)$ in $\nonlinearClass$ represent the degree of the polynomial nonlinearity and a bound on its coefficients; and the parameter $\forcingBound$ in $\forcingClass$ represents a pointwise bound on the forcing term. The following table summarizes the values of these parameters for our PDE examples in Section~\ref{sec:Illustrative Examples}. 
\bigskip

\begin{table}[H]
	\begin{center}
		\begin{tabular}{| c || c | c |  c | c || c | c || c | c || c |}
			\hline
			PDE & $\dissipation$ & $\gamma$ & $\power$ & $\inverseBound$ & $s$ & $\DcoefficientBound$ & $p$ & $\coefficientBound$ & $\forcingBound$ \\ \hline \hline
			Navier--Stokes & $2$ & 0 & 0 &$\nu$ & 1 & 1 & 2 & 1 & $\|f\|_{L^{\infty}}$ \\ \hline
            Allen--Cahn & $2$ & $r( \inputBall^2+\nu^{-1}|\log(\nu)|^2)$ & $0$ & $\nu$  & 0 & $1$ & 3 & 1 & 0 \\ \hline
            Cahn--Hilliard & $4$ & $r( \inputBall^2+|\log(\nu)|^2) $ & $\frac{1}{2}$ & $\nu$ & 2 & 1 & 3 & 1 & 0 \\ \hline
		\end{tabular}
	\end{center}
\caption{Values of the parameter $\vartheta:=( \dissipation,\gamma,\power,\inverseBound, s,\DcoefficientBound,p, \coefficientBound, \forcingBound) $  for the PDE examples in Section~\ref{sec:Illustrative Examples}. In these equations, $\mathcal{A}$ is a negative Laplacian or a bilaplacian operator scaled by a viscosity parameter $\nu>0.$ The interpretation of $r$ and $U$ will be discussed below.    }
\end{table}

\bigskip

We now define the classes $\linearOpClass, \nonlinearClass,$ $\differentialClass,$ and $\forcingClass$:  

\begin{enumerate}
    \item  $\linearOpClass(\dissipation,\gamma,\power,\inverseBound)$ is a class of linear operators defined for $\dissipation>0,$ $\gamma \ge 0,$ $\power \in [0,1)$ and $\inverseBound>0.$ It contains operators
    $\linearOp:L^2(\T^d,\R^{\doutput})\to L^2_N(\T^d,\R^{\doutput})$
    of the form 
    $$  
     \linearOp= \Bigl(\frac{1}{\timestep}I+\gamma \mathcal{A}^{\power} +
    \mathcal{A}\Bigr)^{-1} \mathcal{P}_N,$$
    where:
    \begin{enumerate}
        \item $\mathcal{P}_N : L^2(\mathbb{T}^d,\mathbb{R}^{d_*})
\to
 L^2_{N}(\mathbb{T}^d,\mathbb{R}^{d_*})$ is an $L^2$-orthogonal projection onto any shift invariant subspace of $L^2_N(\T^{d},\R^{\doutput})$, so that for any $u\in L^2(\T^d,\R^{\doutput}),$
 \begin{align*}
     (\mathcal{P}_Nu)(x)=\sum_{k\in \findices_N}\hat{P}_{\mathcal{P}_N}(k)\hat{u}(k)e^{i\langle k,x\rangle},
 \end{align*}
where $\hat{P}_{\mathcal{P}_N}(k)\in \C^{\doutput\times \doutput}$  are orthogonal projections $\hat{P}_{\mathcal{P}_N}(k)\in \C^{\doutput\times \doutput}$ satisfying $\hat{P}_{\mathcal{P}_N}(k)=\overbar{\hat{P}_{\mathcal{P}_N}(-k)}$ for all $k\in \findices_N.$
 \item  $\mathcal{A}:\text{Range}(\mathcal{P}_N)\subseteq L^2_{N}(\T^d,\R^{\doutput})\to L^2_{N}(\T^d,\R^{\doutput})$ is  positive semi-definite, diagonalizable in Fourier space, and its eigenvalues satisfy \eqref{eq:dissipation scaling} with parameter $\dissipation,$ so that
 \begin{align*}
     (\mathcal{A}u)(x)=\sum_{k\in \findices_N}\hat{P}_{\mathcal{A}}(k)\hat{u}(k)e^{i\langle k,x\rangle},
 \end{align*}
 for positive definite matrices $\hat{P}_{\mathcal{A}}(k)\in \C^{\doutput\times \doutput}$ satisfying $\hat{P}_{\mathcal{A}}(k)=\overbar{\hat{P}_{\mathcal{A}}(-k)}$, $\hat{P}_{\mathcal{A}}(k)\hat{P}_{\mathcal{P}_N}(k)=\hat{P}_{\mathcal{P}_N}(k)\hat{P}_{\mathcal{A}}(k)$, and 
 \begin{align}\label{eq:A scaling assumption}
     0\prec \|k\|_2^{\dissipation}\inverseBound  \preceq \hat{P}_{\mathcal{A}}(k)\preceq \|k\|_2^{\dissipation}I,
 \end{align}
 for some $\inverseBound\in (0,1]$ for all $k\neq 0 $ in $\findices_N$ and $\hat{P}_{\mathcal{A}}(0)=0.$
    \end{enumerate}
\item $\differentialClass(s,\DcoefficientBound),$ with $\dissipation/2 \ge s \in \mathbb{Z}_+,$  denotes the class of differential operators 
with each coordinate derivative order bounded by $s$
given by 
\begin{align*}
    \differentialClass(s,\DcoefficientBound)& := \Biggl\{\mathcal{D}_s:L^2_{pN}(\T^d,\R^{\ndim})\to L^2_{pN}(\T^d,\R^{\doutput})\bigg|\\ &\hspace{0.5cm}  
    \mathcal{D}_s(v)(x)=\begin{bmatrix}
       \sum_{\ell=1}^{\ndim} \sum_{j=1}^{d}c_{j\ell 1}\partial_{x_j}^{s_{j\ell 1}}v_{\ell}(x)\\
        \vdots \\
        \sum_{\ell=1}^{\ndim}\sum_{j=1}^dc_{j\ell \doutput}\partial_{x_j}^{s_{j\ell\doutput}}  v_{\ell}(x) 
    \end{bmatrix}, \, \|\vec{c}\|_{\infty}\leq \DcoefficientBound, \,\|\vec{s}\|_{\infty}\leq s \,\Biggr\} \, ,
\end{align*}
where $\vec{c}=(c_{111},\hdots, c_{d\ndim\doutput})\in \R^{d\ndim \doutput}$ and  $\vec{s}=(s_{111},\hdots,s_{d\ndim \doutput})\in \Z_{+}^{d\ndim\doutput}.$ Here $\|\vec{c}\|_{\infty}$ denotes the largest coefficient and $\|\vec{s}\|_{\infty}$ denotes the highest order derivative in $\mathcal{D}_s.$

\item $\nonlinearClass(p,\coefficientBound)$ is the space of degree-$p$ polynomial nonlinearities (applied component-wise and point-wise) of $\doutput$ variables with coefficients bounded by $\coefficientBound$; that is, 
\begin{align*}
\nonlinearClass(p,\coefficientBound)&:=\Biggl\{\nonlinearity:L^2_{N}(\T^d,\R^{\doutput})\to L^2_{pN}(\T^d,\R^{\ndim}) \, \bigg\vert \,    \\
&\hspace{0.5cm} \nonlinearity(u_N)(x)=\begin{bmatrix}
        g_{1}(u_N(x))\\
         \vdots \\
         g_{\ndim}(u_N(x)) \end{bmatrix} \,\,\, \forall x\in \T^d, 
    g_{j}\in \polyClass(p,\coefficientBound), 1 \le j \le \ndim  \Biggr\},
\end{align*}
where
\begin{align*}
\polyClass(p,\coefficientBound):=\biggl\{g:\R^{\doutput}\to \R \, \Big|\,  g(z)= c_{0}+\sum_{i_1=1}^{\doutput}c_{i_1}z_{i_1}+\cdots+ \hspace{-0.3cm}\sum_{i_1,\ldots, i_p=1}^{\doutput}c_{i_1,\ldots, i_p}z_{i_1}\ldots z_{i_p},\|\vec{c}\|_{\infty} \leq \coefficientBound\biggr\}
\end{align*}
and $\|\vec{c}\|_{\infty}$ denotes the magnitude of the largest coefficient. 
\item  $\forcingClass(\forcingBound)$ denotes the class of forcing terms in the range of $\mathcal{P}_N$ that are bounded pointwise by $\forcingBound:$\nc 
\begin{align*}
    \forcingClass(\forcingBound)&:=\biggl\{\forcing_N\in \text{range}(\mathcal{P}_N) \Big\vert \sup_{x\in \T^d}\|\forcing_N(x)\|_2\leq \forcingBound\biggr\}.
\end{align*}
\end{enumerate}


\begin{remark}
    The conditions $0\preceq \|k\|_2^{\dissipation}\inverseBound I \preceq \hat{P}_{\mathcal{A}}(k)\preceq \|k\|_2^{\dissipation}$, and $\hat{P}_{\mathcal{A}}(k)\hat{P}_{\mathcal{P}_N}(k)=\hat{P}_{\mathcal{P}_N}(k)\hat{P}_{\mathcal{A}}(k)$ for all $k\in \findices_N$ imply that for any mean-zero $u\in \text{Range}(\mathcal{P}_N)\subseteq L^2_{N}(\T^d,\R^{\doutput})$, $\mathcal{A}$ satisfies the Poincaré inequality
    \begin{align*}
        \inverseBound \|u\|_{L^2}\leq \|\mathcal{A}u\|_{L^2}. 
    \end{align*} 
\end{remark}
\begin{remark}
    It is straightforward to extend the class $\differentialClass(s,\DcoefficientBound)$ to include mixed partial derivatives, and our main results in Theorems~\ref{thm: polynomial nonlinearity approximation theorem} and \ref{thm: learning} would continue to hold with only minor modifications to the proofs. For ease of notation, we do not pursue this extension, since the examples in Section~\ref{sec:Illustrative Examples} do not involve mixed derivatives.
\end{remark}
 \subsubsection{Class of Operators}\label{ssec:operatorclasspoly}
Next, we introduce the class of operators that can be well approximated by spectral methods and then show that this class can also be well approximated by FNOs. We consider operators acting on a Sobolev ball of radius $\inputBall,$ denoted by 
$$\sobolevBall :=\bigl\{u\in L^2(\T^d,\R^{\doutput}):\|u\|_{H^a}\leq\inputBall \bigr\},$$
where the Sobolev norm is defined via Fourier series as
\begin{align*}
    \|u\|_{H^{a}}^2:=\sum_{k\in \Z^d}\bigl(1+\|k\|_2^2\bigr)^a\|\hat{u}(k)\|_{2}^2.
\end{align*}
    For fixed $T>0$ and parameter $\bar{\vartheta}:=(\vartheta,a,U,V),$ with $a>d/2,$  we define $\pdeClass^{\textup{poly}}(\bar{\vartheta})$ to be the class of operators
    $\PDE$ for which the following condition holds:
    \begin{condition}\label{conditionsec3}
    The operator $\PDE:L^2(\T^d,\R^{\doutput})\to L^2(\T^d,\R^{\doutput})$ satisfies the following properties:
    \begin{enumerate}[label=(\Roman*)]
        \item \label{item:boundedness}  (Boundedness) $\sup_{u\in \sobolevBall}\|\PDE(u)\|_{L^2}\leq \outputBall$. 
        \item \label{item:spectral method}There exists a spectral method 
        $\spectralStep\in \spectralClass(\vartheta)$ satisfying:
        \begin{itemize}
            \item (Stability) For all $j\in \mathbb{N}$ sufficiently large, 
            $$\sup_{u_N\in \sobolevBall\cap L^2_N(\T^d,\R^{\doutput})}\|\spectralStep^j(u_N)\|_{L^2}\leq \outputBall.$$
            \item (Accuracy) For all $j\in \mathbb{N}$ sufficiently large, taking $\tau=T/j,$ we have $$\sup_{u\in \sobolevBall}\|\spectralStep^j(\mathcal{I}_Nu)-\PDE(u)\|_{L^2}\leq r_1\exp(r_2T)\Bigl(N^{-a}+\timestep\Bigr). $$
        \end{itemize}
    \end{enumerate}
    \end{condition}
    The parameter $\smoothness>d/2$ will play a key role in our approximation and learning bounds.
  One can interpret this parameter as the Sobolev smoothness of the solution to a (semi)-dynamical system of the form \eqref{eq:evolution PDE again}. For dissipative evolution equations, the dynamics are typically asymptotically bounded, so Condition~\ref{conditionsec3} \ref{item:boundedness} is natural. We assume, for simplicity of notation, that the solution upper bound in Condition~\ref{conditionsec3} \ref{item:boundedness} is the same as the spectral method stability bound appearing in the first item of Condition~\ref{conditionsec3} \ref{item:spectral method}; if this is not the case, one can take the maximum of the two upper bounds. The second item in Condition~\ref{conditionsec3} \ref{item:spectral method} captures the rate at which the spectral method converges to the true solution operator, where the spatial rate is determined by the smoothness parameter $\smoothness$. 

  We emphasize that semi-implicit spectral methods $\spectralStep$ typically do not require a restrictive CFL condition, as is often the case for explicit time-stepping schemes. For this reason, the accuracy bound in Condition~\ref{conditionsec3} \ref{item:spectral method} does not require $\timestep$ to scale with $N$.
We will show in Section~\ref{sec:Illustrative Examples} that the time-$T$ solution maps of the Navier--Stokes, Allen--Cahn, and 
Cahn--Hilliard equations belong to the class $\pdeClass^{\textup{poly}}(\bar{\vartheta})$ for appropriate choices of the parameter $\bar{\vartheta}.$

\subsection{Approximation Bound}\label{ssec: polynomial approximation section}
In this section, we state our first main result: operators in 
$\pdeClass^{\textup{poly}}(\bar{\vartheta})$
can be uniformly approximated to accuracy $\eps$ by sufficiently large FNOs. In particular, we provide quantitative estimates on the required network size. We show that any $\PDE\in \pdeClass^{\textup{poly}}(\bar{\vartheta})$ can be approximated to accuracy $\eps$ within the following class of FNOs:

\begin{align}\label{eq:sigmaclass}
\begin{split}
    \Sigma^{\textup{poly}}(\eps,\bar{\vartheta})=\biggl\{\Psi\in  \mathsf{FNO}(\mathscr{D},\mathscr{W},\mathscr{S},\mathscr{C}): \mathscr{D}&=r\eps^{-2}\log (\eps^{-1}),\\
        \mathscr{W}&=r\ndim{p+\doutput-1\choose \doutput-1},\\
        \mathscr{S}&=r\eps^{-{\max\{1,\frac{\beta \eta}{\smoothness},\frac{s}{\smoothness}\}}},\\
        \mathscr{C}&=r\eps^{-\frac{1}{\smoothness}},\\
         & \hspace{-0.5cm} \sup_{u\in \sobolevBall}\|\Psi(u)\|_{L^2}\leq 2\outputBall\biggr\}.
\end{split}
\end{align}
Here, $r$ is a positive constant, increasing in $\dissipation^{-1},\gamma,\power,\inverseBound^{-1}, s,p, \coefficientBound, \forcingBound,U$ and $V$, which can be traced through the proof of Theorem~\ref{thm: polynomial nonlinearity approximation theorem} in Section~\ref{ssec:polynomial approximation theorem proof}. Note that we also constrain the output of the FNO to lie in an $L^2$-ball of radius $2\outputBall.$ We are now ready to state the first main result.

 \begin{theorem}[Approximation Bound]\label{thm: polynomial nonlinearity approximation theorem}
     Let $\PDE\in\pdeClass^{\textup{poly}}(\bar{\vartheta})$.
     Then, for any $0<\eps<1/2$, there exists an FNO $\Psi\in \Sigma^{\textup{poly}}(\eps,\bar{\vartheta})$ such that
     \begin{align*}
        \sup_{u\in \sobolevBall} \|\Psi(u)-\PDE(u)\|_{L^2}\leq \eps.
     \end{align*}
 \end{theorem}
 Theorem~\ref{thm: polynomial nonlinearity approximation theorem} provides a \textit{uniform} approximation guarantee over the class of operators $\pdeClass^{\textup{poly}}(\bar{\vartheta}),$ with the required size of the FNO architecture
controlled quantitatively in terms of the desired accuracy $\eps$, as specified in the definition of $\Sigma^{\textup{poly}}(\eps,\bar{\vartheta})$. While quantitative FNO approximation bounds have been shown for specific PDEs, such as stationary Darcy flow and the Navier--Stokes equations \cite{kovachki2021universal}, the constructions have been specific to the equations at hand. To our knowledge, Theorem~\ref{thm: polynomial nonlinearity approximation theorem} is the first approximation result to hold uniformly over a broad class of operators, including those arising from a range of practically relevant nonlinear evolution equations. We discuss particular examples of PDEs covered by our theory in Section~\ref{sec:Illustrative Examples}. This approximation result is essential to the learning bound in Section~\ref{ssec:mainresultslearning}, where we establish polynomial sample complexity for learning operators in $\pdeClass^{\textup{poly}}(\bar{\vartheta})$ with FNOs.

 We postpone the proof of
Theorem~\ref{thm: polynomial nonlinearity approximation theorem} to
Section~\ref{ssec:polynomial approximation theorem proof}. The idea is to
construct an FNO that efficiently approximates one step of the
spectral method for an appropriately chosen timestep $\timestep$ and spatial
resolution $N$, and then to compose $j=\lceil T/\timestep\rceil$ such
single-step approximations. The resulting composition can be implemented as a
single FNO with $j$ times as many layers.
The construction is inspired by, but more general than, that in \cite{kovachki2021universal}, which is specific to the Navier--Stokes equations. The error analysis differs more significantly. In particular, we measure the single-step error in the nonlinearity in an $\mathcal{A}^{-1/2}$-weighted norm, which is both well suited to analyzing the error of $j$ compositions and allows for faster approximation rates by exploiting the dissipation induced by $\mathcal{A}$.

The proof proceeds as follows. In Lemmas~\ref{lemma:B FNO implementation}, \ref{lemma:Ds FNO implementation}, and \ref{lemma:polynomial part single step}, we construct FNOs to approximate $\linearOp\in \linearOpClass(\dissipation,\gamma,\power,\inverseBound),$ $\mathcal{D}_s \in \differentialClass(s,\DcoefficientBound)$, and $\nonlinearity\in \nonlinearClass(p,\coefficientBound)$, respectively. Combining these FNOs allows us to efficiently approximate any spectral method in $\spectralClass(\vartheta),$ and in turn approximate operators in $\pdeClass^{\textup{poly}}(\bar{\vartheta})$. Our proof technique for controlling the error incurred by iterating the single-step FNO differs from that of \cite{kovachki2021universal} for the Navier--Stokes case, and more closely resembles stability arguments for spectral methods, such as those in \cite{he2005stability}.

\subsection{Learning Bound}\label{ssec:mainresultslearning}
We assume that we are given training data
$\bigl\{u^{(m)},\PDE(u^{(m)})\bigr\}_{m=1}^M$ for some unknown
$\PDE\in \pdeClass^{\textup{poly}}(\bar{\vartheta})$, where
$u^{(1)},\hdots,u^{(\samples)}\overset{\textup{i.i.d.}}{\sim} \mu,$ and $\mu$
is a measure with $\operatorname{supp}(\mu)\subseteq \sobolevBall$.
We consider learning with an FNO trained by empirical risk minimization. In particular, we consider the squared-error empirical risk functional
\begin{align}\label{eq:empirical risk}
    \hat{\risk}(\Psi,\PDE):=\frac{1}{\samples}\sum_{\sample=1}^{\samples}\|\Psi(  u^{(\sample)}  )-\PDE(   u^{(\sample)}   )\|_{L^2}^2,
\end{align}
and the corresponding population risk
\begin{align}\label{eq:population risk}
\risk(\Psi,\PDE):=\E_{u\sim \mu}\|\Psi(u)-\PDE(u)\|_{L^2}^2.
\end{align}
Our learned estimator is given by a minimizer of \eqref{eq:empirical risk} over the class of FNOs $\Sigma^{\textup{poly}}(\eps,\bar{\vartheta})$, defined  in the previous section, for some choice of $\eps$: 
    \begin{align}\label{eq:ERMfortheorem}
        \hat{\Psi}_{\PDE}\in \argmin_{\Psi\in \Sigma^{\textup{poly}}(\eps,\bar{\vartheta})}\hat{\risk}(\Psi,\PDE).
    \end{align}
We refer to any such $\hat{\Psi}_{\PDE}$ as an empirical risk minimizer (ERM). In the following remark, we briefly explain why such empirical risk minimizers exist.
\begin{remark}
    We denote by $C(\mathsf{U}_{\smoothness},L^2(\T^{d},\R^{\doutput}))$ the set of continuous operators from $\mathsf{U}_{\smoothness}$ to $L^2(\T^{d},\R^{\doutput})$ equipped with the supremum norm. By definition of  $\Sigma^{\textup{poly}}(\eps,\bar{\vartheta})\subset C(\mathsf{U}_{\smoothness},L^2(\T^{d},\R^{\doutput}))$, the trainable parameters of the FNO satisfy $\theta\in[-r\eps^{-1},r\eps^{-1}]^{d_{\theta}}$, which is compact. Moreover, by \cite[Lemma E.6]{kovachki2024data}, the parameter-to-operator map $\theta\mapsto \Psi_{\theta}$ is Lipschitz continuous. Thus, $\Sigma^{\textup{poly}}(\eps,\bar{\vartheta})$ is compact in $C(\mathsf{U}_{\smoothness},L^2(\T^{d},\R^{\doutput})),$ and a minimizer of \eqref{eq:empirical risk} is guaranteed to exist.
\end{remark}

The main result of this section, Theorem~\ref{thm: learning}, shows that
FNOs obtained by minimizing the empirical risk \eqref{eq:empirical risk} over
$\Sigma^{\textup{poly}}(\eps,\bar{\vartheta})$, for an appropriate choice of
$\eps$, learn operators in $\pdeClass^{\textup{poly}}(\bar{\vartheta})$ with polynomial sample
complexity.
\begin{theorem}[Learning Bound]\label{thm: learning}
Let $\samples\in\N$ be sufficiently large and let 
$u^{(1)},\hdots,u^{(\samples)}\overset{\textup{i.i.d.}}{\sim}\mu.$
For each $\PDE\in \pdeClass^{\textup{poly}}(\bar{\vartheta})$, let
$\hat{\Psi}_{\PDE}$ 
be the ERM defined in \eqref{eq:ERMfortheorem} with 
$\eps = r\samples^{-\frac{1}{2(3+\frac{d}{a})}}.$
Then, with probability at least $1-1/\samples$,
\begin{align*}
    \sup_{\PDE\in \pdeClass^{\textup{poly}}(\bar{\vartheta})}
    \risk(\hat{\Psi}_{\PDE},\PDE)
    \leq
    r\samples^{-\frac{1}{6+2\frac{d}{a}}}.
\end{align*}
\end{theorem}

This theorem demonstrates that operators in $\pdeClass^{\textup{poly}}(\bar{\vartheta})$ can be learned with polynomial sample complexity using FNOs. In contrast, \cite{kovachki2024data} shows that general $k$-times Fr\'echet differentiable or Lipschitz operators cannot be learned from a polynomial number of 
samples, regardless of the learning procedure. In our setting, operators in $\pdeClass^{\textup{poly}}(\bar{\vartheta})$ can be well approximated by spectral methods and consequently admit parameter-efficient representations that can be learned data-efficiently by FNOs. As we will discuss in detail in Section~\ref{sec:Illustrative Examples}, the class $\pdeClass^{\textup{poly}}(\bar{\vartheta})$ contains the solution maps of a range of nonlinear dissipative evolution equations. 

We present the proof of Theorem~\ref{thm: learning} in Section~\ref{ssec: polynomail learning proof}. The argument follows the general structure of \cite[Theorem 3.3]{kovachki2024data}, using ideas similar to those in \cite{kovachki2024data,grohs2024proof}. At a high level, the idea is classical. We first show that the population risk of an ERM can be controlled by the best approximation error and the worst-case empirical risk deviation over $\Sigma^{\textup{poly}}(\eps,\bar{\vartheta})$ and 
$\pdeClass^{\textup{poly}}(\bar{\vartheta}).$
Since we minimize over $\Sigma^{\textup{poly}}(\eps,\bar{\vartheta}),$ the approximation error is controlled by $\eps$ by construction. The remaining task is therefore to show that, with high probability, the worst-case empirical risk deviations are uniformly small. This is achieved by controlling the metric entropies of $\Sigma^{\textup{poly}}(\eps,\bar{\vartheta})$ and 
$\pdeClass^{\textup{poly}}(\bar{\vartheta}).$
Finally, $\eps$ is chosen to balance the approximation error against the metric entropy of $\Sigma^{\textup{poly}}(\eps,\bar{\vartheta})$.

\begin{remark}
    In the setting considered in this section, where  the nonlinearity is given by a degree-$p$ polynomial, the statistical rate achieved by the FNO ERM is determined by the smoothness parameter $\smoothness$ and by the dimension $d$ of the physical domain $\T^d$. This is because the depth required for an FNO to approximate a degree-$p$ polynomial to accuracy $\eps$ scales like $r\log(\eps^{-1})$, where only the constant factor $r$ increases with $p$; see Lemma~\ref{lemma: polynomial approximation}. The contribution of this part of the approximation to the metric entropy of $\Sigma^{\textup{poly}}(\eps,\bar{\vartheta})$ therefore only introduces polynomial factors in $r\log(\eps^{-1})$. Iterating this construction $j=\lceil T/\timestep\rceil$ times and taking the input domain to be sufficiently large contributes $r\eps^{-4}\text{poly}\log(\eps^{-1})$ to the metric entropy, where the parameters in $\bar{\vartheta}$ only enter through the constant $r$. The width contributes only a constant factor, and the weight bound only a logarithmic factor in $\eps^{-1}$ and a constant factor in $\bar{\vartheta}$ (see Lemma \ref{lemma:FNO metric entropy}). In contrast, the error incurred by truncating an $H^{\smoothness}$ function to its first $N$ Fourier modes scales like $N^{-\smoothness}$, so achieving accuracy $\eps$ requires $N\asymp \eps^{-1/\smoothness}$. The corresponding contribution to the metric entropy scales like $\eps^{-\frac{d}{\smoothness}}$, which is the primary parametric dependence in the bound. In Section~\ref{sec:smoothnonlinearity}, we consider nonlinearities $\G$ that satisfy general smoothness assumptions. The resulting rates for these operators depend explicitly on the smoothness of the nonlinearity and the strength of the dissipation, and are generally slower than the rates for polynomial nonlinearities. See Remark \ref{remark:comparison remark} for further discussion.
\end{remark}

\subsection{Examples: Learning the Solution Map of Dissipative PDEs}\label{sec:Illustrative Examples}  
In this section, we verify that several classical dissipative evolution
equations fit into the abstract framework developed above. We focus on the
Navier--Stokes, Allen--Cahn, and Cahn--Hilliard equations, which illustrate how
the assumptions defining $\pdeClass^{\textup{poly}}(\bar{\vartheta})$ arise naturally for
physically relevant PDEs. Together, these examples exhibit distinct nonlinear
structures, dissipation mechanisms, and stabilization regimes. They are
particularly well suited to FNO approximation because their dynamics admit efficient Fourier
spectral discretizations on periodic domains.

    \subsubsection{Navier--Stokes Equations}
    The incompressible Navier--Stokes equations describe the evolution of viscous incompressible fluid flows under the combined effects of advection and diffusion.
    On the two-dimensional torus $\T^2$ with periodic boundary conditions, the equations take the form
    \begin{align}\label{eq:NSE}
        \frac{\partial u}{\partial t}- \nu \pLH \Delta u+\pLH\bigl((u\cdot \nabla)u \bigr)=f \, .
    \end{align}
Here $u(x,t)\in\R^2$ denotes the velocity field, $\nu\in(0,1]$ is the viscosity,
$\Delta$ is the Laplacian, and
$f\in L^\infty(\T^2,\R^2)$ is a divergence-free, mean-zero forcing term.
The operator $\pLH$ denotes the Leray--Helmholtz projection onto
divergence-free vector fields, enforcing the incompressibility constraint.
The Laplacian induces second-order dissipation, corresponding to $\dissipation=2,$ while the quadratic transport term gives the degree-two polynomial nonlinearity appearing in our general framework. 

 Let $\PDE^{\textup{NS}}: L^2(\T^2,  \R^2)\to L^2(\T^2,\R^2)$  denote the time-$T$ solution map of the Navier--Stokes equations, which exists and is unique \cite{temam2012infinite}. We now verify Condition~\ref{conditionsec3} to show that $\PDE^{\textup{NS}}\in \pdeClass^{\textup{poly}}(\bar{\vartheta})$ with parameters 
 $$\vartheta=(\dissipation,\gamma,\power,\inverseBound, s,\DcoefficientBound,p, \coefficientBound, \forcingBound)=(2,0,0,\nu,1,1,2,1,F)$$
 and 
 $$\bar{\vartheta}=(\vartheta,\smoothness,\inputBall,\outputBall)=(\vartheta,2,\inputBall,\inputBall+\frac{\forcingBound}{\nu})
 $$
 for any desired $\inputBall,\forcingBound >0$ and $\nu\in (0,1]$.

  We begin with the boundedness of $\PDE^{\textup{NS}}$. Given a pointwise upper bound $F>0$ on the forcing function $f$, \cite[Proposition 12.1]{robinson2001infinite} shows that   
    \begin{align*}
        \sup_{u:\|u\|_{H^2}\leq \inputBall }\|\PDE^{\textup{NS}}(u)\|_{L^2}\leq \inputBall\exp(-\nu 2\pi^2T)+\frac{\forcingBound}{\nu}\leq \inputBall+\frac{\forcingBound}{\nu}:=\outputBall.    
        \end{align*}
    We consider the semi-implicit Euler scheme
    \begin{align}\label{eq:NS euler}
    \frac{1}{\timestep}\bigl(u_N^{j+1}-u_N^{j}\bigr)- \nu \Delta u_N^{j+1}+\mathcal{P}_N^{\textup{NS}} \nabla \cdot \bigl(  u_N^j \otimes u_N^j\bigr) =\mathcal{P}_N^{\textup{NS}} f,
\end{align}
where $\mathcal{P}_N^{\textup{NS}}:L^2(\T^2,\R^2)\to L^2_N(\T^2,\R^2)$ denotes the $L^2$-orthogonal projection onto  the space of degree-$N$, divergence-free, mean-zero trigonometric polynomials. This scheme can be written in the form of \eqref{eq:spec} with
\begin{enumerate}
    \item $\linearOp^{\textup{NS}}:L^2(\T^2,\R^2)\to L^2_N(\T^2,\R^2)$ given by
    \begin{align*}
        \linearOp^{\text{NS}}&=\Bigl(\frac{1}{\timestep} I+ \mathcal{A}^{\text{NS}}\Bigr)^{-1}\mathcal{P}_N,
    \end{align*}
    which corresponds to $\gamma=\power=0$, where
    \begin{enumerate}
        \item $\mathcal{P}_N^{\textup{NS}}:L^2(\T^2,\R^2)\to L^2_N(\T^2,\R^2)$ is given by
        \begin{align*}
            (\mathcal{P}_N^{\textup{NS}}u)(x)=\sum_{ k\in \findices_N\backslash\{0\}}\Bigl(I-\frac{k\otimes k}{\|k\|_2^2}\Bigr)\hat{u}(k)e^{i\langle k,x\rangle},
        \end{align*}
        for any $u\in L^2(\T^2,\R^2).$
        \item $\mathcal{A}^{\textup{NS}}:\text{Range}(\mathcal{P}_N^{\textup{NS}})\subset L^2_N(\T^2,\R^2)\to L^2_N(\T^2,\R^2)$ is given by
        \begin{align*}
            (\mathcal{A}^{\textup{NS}}u_N)(x)=(-\nu\Delta u_N)(x)=\nu\sum_{ k\in \findices_N\backslash\{0\}}\|k\|_2^2\hat{u}_N(k)e^{i\langle k,x\rangle},
        \end{align*}
        for any $u_N\in \text{Range}(\mathcal{P}_N^{\textup{NS}})$, which implies $\dissipation=2$ and $\inverseBound=\nu.$
    \end{enumerate}
    Consequently, $\linearOp^{\textup{NS}}\in \linearOpClass(2,0,0,\nu)$.
    \item $\D_s^{\textup{NS}}:L^2_{2N}(\T^2,\R^3)\to L^2_{2N}(\T^2,\R^2)$ given by 
\begin{align*}
    \D_s^{\textup{NS}}(v_{2N})=\begin{bmatrix}
        \partial_{x_1}v_{2N}^1+\partial_{x_2}v_{2N}^2\\
        \partial_{x_1}v_{2N}^2+\partial_{x_2}v_{2N}^3
    \end{bmatrix},
\end{align*}
for any $v_{2N}\in L^2_{2N}(\T^2,\R^3),$ which implies $s=1$ and $\DcoefficientBound=1$, so $\D_s^{\textup{NS}}\in \differentialClass(1,1).$
\item $\nonlinearity^{\textup{NS}}:L^2_{N}(\T^2,\R^2)\to L^2_{2N}(\T^2,\R^3)$ given by 
\begin{align*}
    \nonlinearity^{\textup{NS}}(u_N)=\begin{bmatrix}
        (u_N^1)^2\\
        u_N^1 u_N^2\\
        (u_N^2)^2
    \end{bmatrix},
\end{align*}
for any $u_N\in L^2_{N}(\T^2,\R^2),$ which implies $p=2$ and $\coefficientBound=1$. Consequently, $\nonlinearity^{\textup{NS}}\in \nonlinearClass(2,1).$
\end{enumerate}

It remains to verify the stability and accuracy requirements in
Condition~\ref{conditionsec3}. In \cite{he2005stability}, this scheme is shown to satisfy a uniform $L^2$ stability bound, verifying the first part of Condition~\ref{conditionsec3} (II). In particular, they show that for any $u_0\in \mathsf{U}_2$, if $\timestep$ is chosen sufficiently small relative to $\nu$, $F$, and $\inputBall,$ then, for all $j\geq 1,$\footnote{In \cite{he2005stability}, they show the stability and error bound for the spectral method applied to initial condition $\mathcal{P}_N u_0,$ not $\mathcal{I}_Nu_0,$ however the proofs carry through exactly as written with this change, provided that $u_0\in \mathsf{U}_\smoothness$ for $\smoothness>\frac{d}{2}$. This is also the case for the methods for the Allen--Cahn and Cahn--Hilliard equations.} 
\begin{align*}
    \|\bigl(\spectralStep^{\text{NS}}\bigr)^j(\mathcal{I}_N u_0)\|_{L^2}\leq \inputBall+\frac{\forcingBound}{\nu}=\outputBall.
\end{align*}
Further, \cite{he2005stability} shows the method yields the error bound
\begin{align*}
    \|u(j\timestep)-\bigl(\spectralStep^{\text{NS}}\bigr)^j(\mathcal{I}_N u_0)\|_{L^2}\leq C_1e^{C_2 j\timestep}\Bigl(N^{-2}+\timestep\Bigr),
\end{align*}
for $\|u_0\|_{H^2}\leq \inputBall$, if $\timestep$ is sufficiently small relative to a constant depending only on $\inputBall$, $\forcingBound$ and $\nu$, where $C_1$ and $C_2$ also only depend on $\nu$, $\forcingBound$, and $\inputBall$. 
This verifies the second part of
Condition~\ref{conditionsec3} \ref{item:spectral method} with $\smoothness=2$.
Consequently, the Navier--Stokes solution operator belongs to
$\pdeClass^{\textup{poly}}(\bar{\vartheta})$, and therefore satisfies the approximation
and learning guarantees established in Theorem~\ref{thm: learning}. We thus obtain the following corollary:

\begin{corollary}[Learning Navier--Stokes Solution Map]
    Let $u^{(1)},\hdots,u^{(\samples)}\overset{\textup{i.i.d.}}{\sim}\mu$ with $\textup{supp}(\mu)\subset \mathsf{U}_2.$ Let $\eps=r\samples^{-\frac{1}{8}}$ and set
    \begin{align*}
        \hat{\Psi}\in \argmin_{\Psi\in \Sigma^{\textup{poly}}(\eps,\bar{\vartheta})}\frac{1}{\samples}\sum_{\sample=1}^{\samples}\|\Psi(  u^{(\sample)}  )-\PDE^{\textup{NS}}(   u^{(\sample)}   )\|_{L^2}^2.
    \end{align*}
    Then, with probability at least $1-1/\samples$, it holds that
    \begin{align*}
       \E_{u\sim \mu}\|\hat{\Psi}(u)-\PDE^{\textup{NS}}(u)\|_{L^2}^2\leq r \samples^{-\frac{1}{8}}.
    \end{align*}
\end{corollary}
The rate $\samples^{-1/8}$ results from the spatial smoothness parameter and the dimension of the physical space, which are $a=2$ and $d=2$, respectively. The constant $r$, which we do not specify here but can be traced through the proof of Theorem~\ref{thm: learning}, is increasing in $\nu^{-1},F$ and $\inputBall,$ capturing the natural intuition that the Navier--Stokes dynamics become harder to learn as the viscosity decreases or the forcing strength increases.

\subsubsection{Allen--Cahn Equation}
The Allen--Cahn equation is a reaction-diffusion model describing phase separation dynamics driven by interfacial energy minimization. On the two-dimensional torus $\T^2$ and with periodic boundary conditions, it is given by  
\begin{align}\label{eq:Allen Cahn}
    \frac{\partial u}{\partial t}-\nu \Delta u+u^3-u=0,
\end{align}
where $u(x,t)$ represents a mixture of two
phases \cite{allen1979microscopic}. For initial data satisfying $-1\leq u_0\leq 1$, the maximum principle
ensures that the solution remains in this range. The Laplacian induces second-order dissipation corresponding to $\dissipation=2,$
while the cubic reaction term drives phase separation and produces the
degree-three polynomial nonlinearity appearing in our general framework.

  We let $\PDE^{\textup{AC}}:L^2(\T^2, \R)\to L^2(\T^2,\R)$ denote the time-$T$ solution map of the Allen--Cahn equation, mapping $u_0\mapsto u(T)=\PDE^{\textup{AC}}(u_0).$ We now verify that $\PDE^{\textup{AC}}\in \pdeClass^{\textup{poly}}(\bar{\vartheta})$ for parameters
  $$\vartheta=(\dissipation,\gamma,\power,\inverseBound, s,\DcoefficientBound,p, \coefficientBound, \forcingBound)=(2,r(\inputBall^2+\nu^{-1}|\log(\nu)|^2),0,\nu,0,1,3,1,0)$$ and 
  $$\bar{\vartheta}=(\vartheta,\smoothness,\inputBall,\outputBall)=(\vartheta,4,\inputBall,\sqrt{
2\nu\inputBall^2+r^4\inputBall^4+9\pi^2})$$ for any $\inputBall>0$ and $\nu\in (0,1].$ The Allen--Cahn equation is an $L^2$-gradient flow of the energy functional 
\begin{align}\label{eq:energy}
    E(u):=\int_{ \T^2 }\frac{1}{2}\nu \|\nabla u\|_2^2+\frac{1}{4}(u^2-1)^2 \, dx,
\end{align}
which implies that $E(u(t))\leq E(u(0))$ for all $t>0$ \cite{cheng2025energy}. Since for all $u\in \R$, $u^2\leq(u^2-1)^2+\frac{5}{4},$ we have that $\|u\|_{L^2}^2\leq \int_{\T^2}(u^2-1)^2 \, dx+5\pi^2\leq 4E(u)+5\pi^2$. Consequently,
\begin{align*}
    \sup_{u: \|u\|_{H^4}\leq \inputBall}\|\PDE^{\textup{AC}}(u)\|_{L^2}\leq \sqrt{4E(u)+5\pi^2}\leq \sqrt{
2\nu\inputBall^2+r^4\inputBall^4+9\pi^2}:=\outputBall,
\end{align*}
where the constant $r>0$ is from the Sobolev embedding $\|u\|_{L^{\infty}}\leq r\|u\|_{H^4},$ which holds since $\smoothness=4> 1=\frac{d}{2}.$ We consider the stabilized semi-implicit Euler scheme
\begin{align}\label{eq:AC Euler}
    \frac{1}{\tau}(u_N^{j+1}-u_N^{j})
    =\nu \Delta u_N^{j+1}
    -\gamma(u_N^{j+1}-u_N^j)
 -  \mathcal{P}_N\Bigl((u_N^j)^3-u_N^j\Bigr)
\end{align}
for $\gamma= r(\inputBall^2+\nu^{-1}|\log(\nu)|^2)$, where $\mathcal{P}_N$ is the $L^2$-orthogonal projection onto $L^2_N(\T^2 ,\R)$. This scheme can be written in the form of \eqref{eq:spec} with
\begin{enumerate}
    \item $\linearOp^{\textup{AC}}:L^2(\T^2,\R)\to L^2_N(\T^2,\R)$ is given by
    \begin{align*}
        \linearOp^{\textup{AC}}&=\Bigl(\frac{1}{\timestep} I+\gamma I+ \mathcal{A}^{\textup{AC}}\Bigr)^{-1}\mathcal{P}_N,
    \end{align*}
    which implies $\power=0$, where
    \begin{enumerate}
        \item $\mathcal{P}_N:L^2(\T^2,\R)\to L^2_N(\T^2,\R)$ is given by
        \begin{align*}
            (\mathcal{P}_Nu)(x)=\sum_{k\in \findices_N}\hat{u}(k)e^{i\langle k,x\rangle},
        \end{align*}
        for any $u\in L^2(\T^2,\R);$
        \item $\mathcal{A}^{\textup{AC}}:L^2_N(\T^2,\R)\to L^2_N(\T^2,\R)$ is given by
        \begin{align*}
            (\mathcal{A}^{\textup{AC}}u_N)(x)=(-\nu\Delta u_N)(x)=\nu\sum_{k\in \findices_N}\|k\|_2^2\hat{u}_N(k)e^{i\langle k,x\rangle},
        \end{align*}
        for any $u_N\in L^2_N(\T^2,\R)$, which implies $\dissipation=2$ and $\inverseBound=\nu.$
    \end{enumerate}
    Consequently, 
    $\linearOp^{\textup{AC}}\in \linearOpClass(2,r(\inputBall^2+\nu^{-1}|\log(\nu)|^2),0,\nu).$
    \item $\D_s^{\textup{AC}}:L^2_{3N}(\T^2,\R)\to L^2_{3N}(\T^2,\R)$ by
\begin{align*}
    \D_s^{\textup{AC}}(v_{3N})= v_{3N},   
\end{align*}
for any $v_{3N}\in L^2_{3N}(\T^2,\R^3),$ which implies $s=0$ and $\DcoefficientBound=1$, so $\D_s^{\textup{AC}}\in \differentialClass(0,1).$
\item $\nonlinearity^{\textup{AC}}:L^2_{N}(\T^2,\R)\to L^2_{3N}(\T^2,\R)$ given by 
\begin{align*}
    \nonlinearity^{\textup{AC}}(u_N)= u_N^3-u_N,
\end{align*}
for any $u_N\in L^2_{N}(\T^2,\R),$ which implies $p=3$ and $\coefficientBound=1$. Consequently, $\nonlinearity^{\textup{AC}}\in \nonlinearClass(3,1)$.
\end{enumerate}

It is shown in \cite{cheng2025energy} that with this choice of $\gamma$, this method is energy-stable, meaning that for any $j\geq 1,$
\begin{align*}
    E\Bigl(\bigl(\spectralStep^{\textup{AC}}\bigr)^j(\mathcal{I}_Nu_0)\Bigr)\leq  E\Bigl(\bigl(\spectralStep^{\textup{AC}}\bigr)^{j-1}(\mathcal{I}_Nu_0)\Bigr),
\end{align*}
for $\|u_0\|_{H^{4}}\leq \inputBall$ if $\timestep$ is chosen small enough relative to a constant depending only on $\inputBall$, $\gamma$, and $\nu$. This implies a uniform stability bound in $L^2$,
\begin{align*}
 \|\bigl(\spectralStep^{\textup{AC}}\bigr)^j(\mathcal I_N u_0)\|_{L^2} \leq \sqrt{4E(\mathcal{I}_Nu_0)+5\pi^2}\leq \sqrt{
2\nu\inputBall^2+r^4\inputBall^4+9\pi^2}=\outputBall,
\end{align*}
verifying the first part of Condition~\ref{conditionsec3} (II). Further, the iterates have error bounded by
\begin{align*}
    \|u(j\timestep)-\bigl(\spectralStep^{\textup{AC}}\bigr)^j(\mathcal{I}_Nu_0)\|_{L^2}\leq C_1 e^{C_2j\timestep }\Bigl(N^{-4}+\timestep\Bigr),
\end{align*}
for $\|u_0\|_{H^{4}}\leq \inputBall$, where $C_1$ and $C_2$ also only depend on $\nu$, $\gamma$, and $\inputBall$. This implies that $\smoothness=4,$ and verifies the second part of Condition~\ref{conditionsec3} (II).
Consequently, the Allen--Cahn solution operator belongs to
$\pdeClass^{\textup{poly}}(\bar{\vartheta})$, and therefore satisfies the approximation
and learning guarantees established in Theorem~\ref{thm: learning}. We thus obtain the following corollary:

\begin{corollary}[Learning Allen--Cahn Solution Map]
    Let $u^{(1)},\hdots,u^{(\samples)}\overset{\textup{i.i.d.}}{\sim}\mu$ with $\textup{supp}(\mu)\subset \mathsf{U}_4.$ Let $\eps=r\samples^{-\frac{1}{7}}$ and set
    \begin{align*}
        \hat{\Psi}\in \argmin_{\Psi\in \Sigma^{\textup{poly}}(\eps,\bar{\vartheta})}\frac{1}{\samples}\sum_{\sample=1}^{\samples}\|\Psi(  u^{(\sample)}  )-\PDE^{\textup{AC}}(   u^{(\sample)}   )\|_{L^2}^2.
    \end{align*}
   Then, with probability at least $1-1/\samples$, it holds that
    \begin{align*}
       \E_{u\sim \mu}\|\hat{\Psi}(u)-\PDE^{\textup{AC}}(u)\|_{L^2}^2\leq r  \samples^{-\frac{1}{7}}. 
    \end{align*}
\end{corollary}
The rate $M^{-1/7}$ is determined by the spatial smoothness parameter and the dimension of the physical space, which are $a=4$ and $d=2,$
respectively. The constant $r$ is increasing in $\nu^{-1}$ and $\inputBall.$

\subsubsection{Cahn--Hilliard Equation}
The Cahn--Hilliard equation is a mass-conserving phase-field model describing phase separation and coarsening in binary mixtures through fourth-order diffusion. 
On the two-dimensional torus $\T^2$ and with periodic boundary conditions, it is given by
\begin{align*}
    \frac{du}{dt}- \Delta \bigl((- \nu\Delta) u+(u^3-u)\bigr)=0,
\end{align*}
where $u(x,t)$ is a real-valued function. 
The biharmonic operator induces fourth-order dissipation corresponding to
$\dissipation=4,$ while the cubic potential produces the degree-three
polynomial nonlinearity governing phase separation dynamics.
In contrast to the Allen--Cahn equation, whose dynamics are governed by second-order diffusion, the Cahn--Hilliard equation involves fourth-order diffusion, making stabilization particularly important for large-timestep numerical schemes.

We let $\PDE^{\textup{CH}}:L^2(\T^2,\R)\to L^2(\T^2,\R)$ denote the time-$T$ solution map of the Cahn--Hilliard equation, mapping $u_0\mapsto u(T)=\PDE^{\text{CH}}(u_0).$ We now verify that $\PDE^{\textup{CH}}\in \pdeClass^{\textup{poly}}(\bar{\vartheta})$ for 
$$\vartheta=(\dissipation,\gamma,\power,\inverseBound, s,\DcoefficientBound,p, \coefficientBound, \forcingBound)=(4,r(\inputBall^2+\nu^{-1}|\log(\nu)|^2),\frac{1}{2},\nu,2,1,3,1,0)$$
and 
$$\bar{\vartheta}=(\vartheta,\smoothness,\inputBall,\outputBall)=(\vartheta,4,\inputBall,\sqrt{
2\nu\inputBall^2+r^4\inputBall^4+9\pi^2})$$ for any $\inputBall>0$ and $\nu\in (0,1].$ 
By the energy dissipation estimate in \cite[Theorem 3.1]{he2008stability}, if $u_0\in H^4(\T^2)$, then the solution satisfies
\[
\|\PDE^{\textup{CH}}(u_0)\|_{L^2}
\leq \sqrt{4E(u_0)+5\pi^2}.
\]
Consequently,
\[
\sup_{\|u_0\|_{H^4}\leq \inputBall}
\|\PDE^{\textup{CH}}(u_0)\|_{L^2}
\leq
\sqrt{2\nu\inputBall^2+r^4\inputBall^4+9\pi^2}
=:\outputBall.
\]
This verifies Condition~\ref{conditionsec3} (I). We consider the stabilized semi-implicit Euler scheme
\begin{align} \label{eq:CH euler}
    \frac{u_N^{j+1}-u_N^j}{\timestep}
    =-\nu \Delta^2 u^{j+1}_N
    +\gamma\sqrt{\nu}\,\Delta (u_N^{j+1}-u_N^j)
    +\Delta \mathcal{P}_N\Bigl((u_N^j)^3-u_N^j\Bigr),
\end{align}
for $\gamma\geq r(\inputBall^2+\nu^{-1}|\log(\nu)|^2)$, where $\mathcal{P}_N$ is the $L^2$-orthogonal projector onto the space of mean-zero degree-$N$ trigonometric polynomials. 
This scheme can be written in the form of \eqref{eq:spec} with
\begin{enumerate}
    \item $\linearOp^{\textup{CH}}:L^2(\T^2,\R)\to L^2_N(\T^2,\R)$ given by
    \begin{align*}
        \linearOp^{\text{CH}}&=\Bigl(\frac{1}{\timestep} I+\gamma \bigl(\mathcal{A}^{\textup{CH}}\bigr)^{1/2}+ \mathcal{A}^{\text{CH}}\Bigr)^{-1}\mathcal{P}_N,
    \end{align*}
    which implies $\power=1/2$, where
    \begin{enumerate}
        \item $\mathcal{P}_N:L^2(\T^2,\R)\to L^2_N(\T^2,\R)$ is given by
        \begin{align*}
            (\mathcal{P}_Nu)(x)=\sum_{k\in \findices_N\backslash\{0\}}\hat{u}(k)e^{i\langle k,x\rangle},
        \end{align*}
        for any $u\in L^2(\T^2,\R);$
        \item $\mathcal{A}^{\textup{CH}}:L^2_N(\T^2,\R)\to L^2_N(\T^2,\R)$  is given by  
        \begin{align*}
            (\mathcal{A}^{\textup{CH}}u_N)(x)=(\nu\Delta^2 u_N)(x)=\nu
            \sum_{ k\in \findices_N \backslash\{0\} }\|k\|_2^4\hat{u}_N(k)e^{i\langle k,x\rangle},
        \end{align*}
        for any $u_N\in L^2_N(\T^2,\R)$, which implies $\dissipation=4$ and $\inverseBound=\nu.$
    \end{enumerate}
    Consequently, $\linearOp^{\textup{CH}}\in \linearOpClass(4,r(\inputBall^2+\nu^{-1}|\log(\nu)|^2),\frac{1}{2},\nu)$.
    \item $\D_s^{\textup{CH}}:L^2_{3N}(\T^2,\R)\to L^2_{3N}(\T^2,\R)$ given by
\begin{align*}
    \D_s^{\textup{CH}}(v_{3N})=-\Delta v_{3N}=-\sum_{j=1}^2\partial^2_{x_j}v_{3N},
\end{align*}
for any $v_{3N}\in L^2_{3N}(\T^2,\R),$  which implies $s=2$ and $\DcoefficientBound=1$, so $\D_s^{\textup{CH}}\in \differentialClass(2,1).$
\item $\nonlinearity^{\textup{CH}}:L^2_{N}(\T^2,\R)\to L^2_{3N}(\T^2,\R)$ given by 
\begin{align*}
    \nonlinearity^{\textup{CH}}(u_N)=-u_N+u_N^3,
\end{align*}
for any $u_N\in L^2_{N}(\T^2,\R),$ which implies $p=3$ and $\coefficientBound=1$. Consequently, $\nonlinearity^{\textup{CH}}\in \nonlinearClass(3,1)$.
\end{enumerate}

It is shown in \cite{cheng2025energy} that with this choice of $\gamma$, the scheme is energy-stable, so for any $j\geq 1,$
\begin{align*}
    E\Bigl( \bigl(\spectralStep^{\text{CH}}\bigr)^j(\mathcal{I}_Nu_0)\Bigr)\leq  E\Bigl( \bigl(\spectralStep^{\text{CH}}\bigr)^{j-1}(\mathcal{I}_Nu_0)\Bigr),
\end{align*}
for $\|u_0\|_{H^{4}}\leq \inputBall$ if $\timestep$ is chosen small enough relative to a constant depending only on $\inputBall$, $\gamma$, and $\nu$. Here $E$ is the energy functional defined in \eqref{eq:energy}. This implies a uniform stability bound in $L^2$,
\begin{align*}
    \|\bigl(\spectralStep^{\textup{CH}}\bigr)^j( \mathcal{I}_N \nc u_0 )\|_{L^2}\leq \sqrt{4E(\mathcal{I}_Nu_0)+5\pi^2}\leq \sqrt{
2\nu\inputBall^2+r^4\inputBall^4+9\pi^2}=\outputBall,
\end{align*}
verifying the first part of Condition~\ref{conditionsec3} (II). As shown in \cite{li2016characterizing}, this scheme admits an error bound given by 
\begin{align*}
    \|u(j\timestep)-\bigl(\spectralStep^{\text{CH}}\bigr)^j(\mathcal{I}_Nu_0)\|_{L^2}\leq r_1e^{r_2 j \timestep}\Bigl(N^{-4}+\timestep\Bigr)
\end{align*}
for $\|u_0\|_{H^4}\leq \inputBall$ with $\timestep$ small enough. This implies that $\smoothness=4,$ and verifies the second part of Condition~\ref{conditionsec3} (II).
Consequently, the Cahn--Hilliard solution operator belongs to
$\pdeClass^{\textup{poly}}(\bar{\vartheta})$, and therefore satisfies the approximation
and learning guarantees established in Theorem~\ref{thm: learning}. We thus obtain the following corollary:
\begin{corollary}[Learning Cahn--Hilliard Solution Map]
    Let $u^{(1)},\hdots,u^{(\samples)}\overset{\textup{i.i.d.}}{\sim}\mu$ with $\textup{supp}(\mu)\subset \mathsf{U}_4.$ Let $\eps=r\samples^{-\frac{1}{7}}$ and set
    \begin{align*}
        \hat{\Psi}\in \argmin_{\Psi\in \Sigma^{\textup{poly}}(\eps,\bar{\vartheta})}\frac{1}{\samples}\sum_{\sample=1}^{\samples}\|\Psi(  u^{(\sample)}  )-\PDE^{\textup{CH}}(   u^{(\sample)}   )\|_{L^2}^2.
    \end{align*}
     Then, with probability at least $1-1/\samples$, it holds that
    \begin{align*}
       \E_{u\sim \mu}\|\hat{\Psi}(u)-\PDE^{\textup{CH}}(u)\|_{L^2}^2\leq r \samples^{-\frac{1}{7}}.
    \end{align*}
\end{corollary}
As in the previous examples, the rate $M^{-1/7}$ 
is determined by the spatial smoothness parameter and the dimension of the physical space, which are $a=4$ and $d=2$,
respectively. The constant $r$ is increasing in $\nu^{-1}$ and $\inputBall.$

\section{Dissipative Models with Smooth Nonlinearity}\label{sec:smoothnonlinearity}
In this section, we develop a theory for approximating and learning 
the time-$T$ solution operator for evolution equations of the form
\begin{align}\label{eq:evolution PDEsmooth}
\begin{split}
    \frac{du}{dt}+\mathcal{A}u+\mathcal{D}_s\nonlinearity( u)
    &=\forcing, \qquad  t>0,
\end{split}
\end{align}
using FNOs under non-parametric smoothness assumptions on the nonlinearity $\G.$
This contrasts with Section~\ref{sec:polynomialnonlinearity}, where $\G$ is assumed to have polynomial structure. Under these general smoothness assumptions, we still obtain polynomial sample complexity for learning the time-$T$ solution operator; however, the rate is slower than in the polynomial case. In direct analogue with Section~\ref{sec:polynomialnonlinearity}, we begin by defining a family of spectral methods and then define a corresponding class of operators that can be approximated by these methods. For this operator class, which includes more general reaction-diffusion equations, we establish FNO approximation bounds in Section~\ref{ssec: smooth approximation section} and learning bounds in Section~\ref{ssec:smooth learning bound}. In Section~\ref{ssec:examplegeneralnonlinearity}, we apply this theory to show that reaction-diffusion equations can be efficiently learned by FNOs.

\subsection{Spectral Methods and Class of Operators}\label{ssec:spectralsmooth}

\subsubsection{Background on Spectral Methods}
Following the ideas in Subsection~\ref{ssec:background}, we consider first-order-in-time semi-implicit spectral methods of the form
\begin{align}\label{eq:generic spectral method smooth}
\frac{1}{\tau}\bigl(u_N^{j+1}-u_N^j\bigr)+\gamma\mathcal{A}^{\power}\bigl(u_N^{j+1}-u_N^j\bigr)+\mathcal{A}u_N^{j+1}+ \mathcal{P}_N\mathcal{D}_s\bigl(\nonlinearity(u_N^j)\bigr)-f_N=0,
\end{align}
where $\gamma\geq0$ and  $\power\in [0,1)$ control the strength of the stabilization term. As in the previous section, the mapping $u_N^{j}\mapsto u_N^{j+1}$ is uniquely defined, and we write a single step of the spectral method $\spectralStep:L^2_N(\T^d,\R^{d'})\to L^2_N(\T^d,\R^{d'})$ as 
\begin{align}\label{eq:spec smooth}
\begin{split}
\spectralStep (u_N) :=& \, \, \Bigl(\frac{1}{\timestep}I+\gamma\mathcal{A}^{\power}+\mathcal{A}\Bigr)^{-1}\Bigl((\frac{1}{\timestep}I+\gamma\mathcal{A}^{\power})u_N-\mathcal{P}_N\mathcal{D}_s\bigl(\nonlinearity(u_N)\bigr)+f_N\Bigr)\\
 =& \, \,   \linearOp \Bigl(\bigl(\frac{1}{\timestep}I+\gamma \mathcal{A}^{\power}\bigr)u_N- \mathcal{D}_s\bigl(\nonlinearity(u_N)\bigr)+\forcing_N\Bigr),
\end{split}
\end{align}
where $\linearOp := \Bigl(\frac{1}{\timestep}I +\gamma\mathcal{A}^{\power}+
    \mathcal{A}\Bigr)^{-1} \mathcal{P}_N.$ Once again, to approximate the time-$T$ solution operator of the PDE, one can choose $J \in \N $ and $\tau>0$ such that 
 $J=T/\timestep,$ and  then iterate the single-step spectral method:
\begin{align*}
    u(T) = \mathcal{S}(u) \approx \spectralStep^J(\mathcal{P}_Nu(0))=\underbrace{\spectralStep\circ \cdots \circ \spectralStep}_{J \text{ times}}(u_N^0).
\end{align*}
In the next subsection, we introduce a family of single-step spectral methods with general smooth nonlinearities. Then, in the following subsection, we introduce a class of operators that can be well approximated by these single-step spectral methods.

\subsubsection{Family of Single-Step Spectral Methods} \label{ssec:spectralclasssmooth}  
Recall the generic single-step spectral method $\spectralStep$ in \eqref{eq:spec smooth}. 
For timestep $\timestep>0$ and degree $N \in \N,$ we introduce the parameter $\vartheta:=( \dissipation,\gamma,\power,\inverseBound,s,\DcoefficientBound,\alpha, \coefficientBound, \forcingBound) $ and define the family of single-step spectral methods 

 \begin{align*} 
\spectralClass^{\text{smooth}}(\vartheta)
    :=\Biggl\{\spectralStep:L^2_N(\T^d,\R^{\doutput})\to L^2_N(\T^d,\R^{\doutput})& \,  \bigg| \,    \spectralStep(u_N)=\linearOp  \Bigl(\bigl(\frac{1}{\timestep}I+\gamma\mathcal{A}^{\power} \bigr)u_N-\mathcal{D}_s\nonlinearity(u_N)+\forcing_N\Bigr),  \\    & \hspace{-1cm} \linearOp\in \linearOpClass(\dissipation,\gamma,\power,\inverseBound), \,   \mathcal{D}_s\in \differentialClass(s,\DcoefficientBound), \,\nonlinearity\in \nonlinearClass(\alpha,\coefficientBound), \, \forcing_N\in \forcingClass(\forcingBound) \Biggr\}.
\end{align*}  
Here, the classes $\linearOpClass(\dissipation,\gamma,\power,\inverseBound)$, $\differentialClass(s,\DcoefficientBound)$,  and $ \forcingClass(\forcingBound)$ are exactly as defined in Subsection~\ref{ssec:spectralclasspoly}, 
   while  $\nonlinearClass(\alpha,\coefficientBound)$ now denotes the class of smooth nonlinearities (applied component-wise and point-wise) with $\alpha$ derivatives bounded by $\coefficientBound.$ That is,
\begin{align*}
\nonlinearClass(\alpha,\coefficientBound)&:=\Biggl\{\nonlinearity:L^2_{N}(\T^d,\R^{\doutput})\to H^{\alpha-1}(\T^d,\R^{\ndim}) \, \bigg\vert \\\,  &\hspace{5.5cm} \nonlinearity(v_N)(x)=\begin{bmatrix}
        g_{1}(v_N(x))\\
         \vdots \\
         g_{\ndim}(v_N(x))
    \end{bmatrix} \,\,\, \forall x\in \T^d, \, g_{j}\in \smoothClass(\coefficientBound), 1 \le j \le \ndim  \Biggr\},
\end{align*}
where
\begin{align*}
    \smoothClass(\coefficientBound):=\biggl\{g\in  L^{\infty}(\R^{\doutput}):\max_{\omega :|\omega|\leq \alpha}\|\partial^{\omega} g\|_{L^{\infty}(\R^{\doutput})}\leq \coefficientBound\biggr\}.
\end{align*}
For simplicity, we assume throughout this section that $\mathcal{P}_N$ projects onto a subspace of $L^2_N(\T^d,\R^{\doutput})$ consisting only of mean-zero functions, so that
\begin{align*}
     (\mathcal{P}_Nu)(x)=\sum_{k\in \findices_N \backslash \{0\}}\hat{P}_{\mathcal{P}_N}(k)\hat{u}(k)e^{i\langle k,x\rangle},
 \end{align*}
where $\hat{P}_{\mathcal{P}_N}(k)\in \C^{\doutput\times \doutput}$  are orthogonal projections  satisfying $\hat{P}_{\mathcal{P}_N}(k)=\overbar{\hat{P}_{\mathcal{P}_N}(-k)}$ for all $k\in \findices_N\backslash \{0\},$ and $\hat{P}_{\mathcal{P}_N}(0)=0.$ This implies that for any $u_N\in \text{Range}(\mathcal{P}_N)$---in particular, for any iterate of the spectral method, $\spectralStep^j(u_N)$---we have the Poincar\'e inequality
\begin{align*}
        \inverseBound \|u_N\|_{L^2}\leq \|\mathcal{A}u_N\|_{L^2}.
    \end{align*}
    We will assume throughout that $\alpha>\frac{d}{2}$ and $s\leq\alpha-1,$ so that $\D_s\nonlinearity(u_N)$ is defined as a classical derivative.

\subsubsection{Class of Operators}\label{ssec:operatorclasssmooth}
Next, we introduce the class of operators that can be approximated by spectral methods in $\spectralClass^{\text{smooth}}(\vartheta)$, defined analogously to the class considered in Subsection~\ref{ssec:operatorclasspoly}. We consider operators acting on a Sobolev ball of radius $\inputBall,$ denoted by  $\sobolevBall=\{u\in L^2(\T^d,\R^{\doutput}):\|u\|_{H^a}\leq\inputBall\}$. 
    For fixed $T>0$ and parameter $\bar{\vartheta}:=(\vartheta,a,U,V),$ with $a>d/2,$ we define $\pdeClass^{\text{smooth}}(\bar{\vartheta})$ to be the class of operators
    $\PDE$ for which the following holds:
     \begin{condition}\label{conditionsec4}
      The operator $\PDE:L^2(\T^d,\R^{\doutput})\to L^2(\T^d,\R^{\doutput})$ satisfies the following properties:
    \begin{enumerate}[label=(\Roman*)]
        \item \label{item:boundedness II}  (Boundedness) $\sup_{u\in \sobolevBall}\|\PDE(u)\|_{L^2}\leq \outputBall$. 
        \item \label{item:spectral method II} There exists a spectral method 
        $\spectralStep\in \spectralClass^{\text{smooth}}(\vartheta)$ satisfying:
        \begin{itemize}
            \item (Stability) For all $j\in \mathbb{N}$ sufficiently large, 
            $$\sup_{u_N\in \sobolevBall\cap L^2_N(\T^d,\R^{\doutput})}\|\spectralStep^j(u_N)\|_{L^2}\leq \outputBall.$$
            \item (Accuracy) For all $j\in \mathbb{N}$ sufficiently large, taking $\tau=T/j,$ we have $$\sup_{u\in \sobolevBall}\|\spectralStep^j(\mathcal{I}_Nu)-\PDE(u)\|_{L^2}\leq r_1\exp(r_2T)\Bigl(N^{-a}+\timestep\Bigr).$$
        \end{itemize}
    \end{enumerate}
    \end{condition}
As in Subsection~\ref{ssec:operatorclasspoly}, we assume that the dynamics are bounded, which is natural for dissipative equations, and that the solution operator $\PDE$ can be approximated by a stable and convergent spectral method, now  belonging to $\spectralClass^{\text{smooth}}(\vartheta)$. The smoothness parameter $\smoothness$ again plays a fundamental role in the approximation and learning rates in the following sections. However, we will see shortly that, under the weaker assumptions on the nonlinearity considered here, the rates also inherit a dependence on the smoothness of the nonlinearity, $\alpha$, and the rate of dissipation, $\dissipation$.

\begin{remark}\label{remark:comparison remark}
    In general, polynomial nonlinearities do not have uniformly bounded derivatives, and thus the class $\pdeClass^{\textup{poly}}$ is not a subset of $\pdeClass^{\text{smooth}}.$ However, if the PDE solution and spectral method iterates satisfy an a priori $L^{\infty}$ stability bound,
    then polynomial nonlinearities restricted to these bounded inputs have uniformly bounded derivatives of all orders and are contained in the class considered in this section.
   This is the case for the Allen--Cahn example in Section~\ref{sec:Illustrative Examples}.
\end{remark}

\subsection{Approximation Bound}\label{ssec: smooth approximation section}
We now state the first main result of this section: a quantitative estimate on the FNO size required to approximate any operator in $\pdeClass^{\text{smooth}}(\bar{\vartheta})$ to any desired accuracy. We show that for any $\PDE\in \pdeClass^{\text{smooth}}(\bar{\vartheta})$, there exists an FNO in the following class that can approximate $\PDE$ to accuracy $\eps$:
\begin{align}\label{eq:sigmaclass II}
\begin{split}
    \Sigma^{\text{smooth}}(\eps,\bar{\vartheta})=\biggl\{\Psi\in  \mathsf{FNO}(\mathscr{D},\mathscr{W},\mathscr{S},\mathscr{C}): \mathscr{D}&=r \eps^{-1}\log(\eps^{-1}),\\
        \mathscr{W}&=r\eps^{-\frac{\doutput d}{2\smoothness}-\frac{\doutput\max\{d+s-\dissipation/2,0\}}{\alpha \smoothness}-\frac{\doutput}{\alpha}}{\log(\eps^{-1})^{\frac{\doutput}{\alpha}}},\\
        \mathscr{S}&=r\eps^{-\frac{\max\{d+s-\dissipation/2,0\}}{\alpha \smoothness}-\frac{1+s}{\alpha}-\frac{(d+2)s}{2\smoothness}}{\log(\eps^{-1})^{\frac{1}{\alpha}}}\, ,\\
        \mathscr{C}&=r\eps^{-\frac{1}{\alpha}-\frac{2+d}{2\smoothness}},\\
        &\sup_{u\in \sobolevBall}\|\Psi(u)\|_{L^2}\leq 2\outputBall\biggr\} \, ,
\end{split}
\end{align}
where $r$ is a positive constant increasing in $\dissipation^{-1}$, $\gamma$, $\power$, $\inverseBound^{-1},$ $\alpha,$ $\coefficientBound$, $\forcingBound$, $\inputBall$, and $\outputBall,$ which can be traced through the proof of Theorem~\ref{thm: smooth nonlinearity approximation theorem}. Precisely, we have the following approximation result.
 \begin{theorem}[Approximation Bound]\label{thm: smooth nonlinearity approximation theorem}
     Let $\PDE\in \pdeClass^{\textup{smooth}}(\bar{\vartheta})$. Then, for any $0<\eps<1/2$, there exists an FNO $\Psi\in \Sigma^{\textup{smooth}}(\eps,\bar{\vartheta})$ such that
     \begin{align*}
      \sup_{u\in \sobolevBall}   \|\Psi(u)-\PDE(u)\|_{L^2}\leq \eps.
     \end{align*}
 \end{theorem}
Theorem~\ref{thm: smooth nonlinearity approximation theorem} provides a \textit{uniform} approximation guarantee over the class of operators $\pdeClass^{\text{smooth}}(\bar{\vartheta})$ with general smooth nonlinearities, and the required size of the FNO architecture is controlled quantitatively in terms of the desired accuracy $\eps$, as specified in the definition of $\Sigma^{\text{smooth}}(\eps,\bar{\vartheta})$. These quantitative estimates on the approximation error will be essential to our learning-theoretic result in Section~\ref{ssec:smooth learning bound}. 

The proof of Theorem~\ref{thm: smooth nonlinearity approximation theorem} is presented in Section~\ref{ssec:smooth approximation proof}. The construction of the FNO achieving $\eps$ accuracy mirrors that in Theorem~\ref{thm: polynomial nonlinearity approximation theorem}; however, the proof of the error bound differs in several substantial ways. The ReLU neural network approximation rates for functions in $W^{\alpha,\infty}(\coefficientBound)$ are slower than those for polynomials, and these bounds cannot be improved in general \cite{yarotsky2017error}. As such, we require a more delicate Gronwall argument than in the proof of Theorem~\ref{thm: polynomial nonlinearity approximation theorem}. Further, it is much more difficult to control the aliasing error incurred by approximating general smooth functions of trigonometric polynomials than it is to control polynomial functions of trigonometric polynomials.

\subsection{Learning Bound}\label{ssec:smooth learning bound}
In this section, we show that operators $\PDE\in\pdeClass^{\text{smooth}}(\bar{\vartheta})$ can be efficiently learned by FNOs trained via empirical risk minimization. As in Section~\ref{ssec:mainresultslearning}, we assume 
access to training data of the form $\{u^{(m)},\PDE(u^{(m)})\}_{m=1}^M,$ where the input functions $u^{(1)},\hdots,u^{(\samples)}$ are  sampled independently from a measure $\mu$ with $\text{supp}(\mu)\subseteq \sobolevBall$. As before, we consider the ERM estimator
\begin{align}\label{eq:empirical risk II}
\begin{split}
   \hat{\Psi}_{\PDE}&\in \argmin_{ \Psi \in \Sigma^{\textup{smooth}}(\eps,\bar{\vartheta})}\hat{\risk}(\Psi,\PDE), \\  \hat{\risk}(\Psi,\PDE)&:=\frac{1}{\samples}\sum_{\sample=1}^{\samples}\|\Psi(  u^{(\sample)}  )-\PDE(   u^{(\sample)}   )\|_{L^2}^2,
\end{split}   
\end{align}
for a choice of $\eps$ specified in Theorem~\ref{thm: smooth learning bound}. We measure error using the population risk
\begin{align}\label{eq:population risk II}
\risk(\Psi,\PDE):=\E_{u\sim \mu}\|\Psi(u)-\PDE(u)\|_{L^2}^2.
\end{align}
The main result of this section is the following:
    \begin{theorem}[Learning Bound]\label{thm: smooth learning bound}
    Let $M \in \N$ be sufficiently large and let $u^{(1)},\hdots,u^{(\samples)}\overset{\textup{i.i.d.}}{\sim} \mu.$
    Let $$\rho:= \frac{d}{\alpha}+\frac{2d+d^2}{\smoothness}+\frac{\doutput d}{\smoothness}+\frac{2\doutput\max\{d+s-\dissipation/2,0\}}{\alpha \smoothness}+
\frac{2\doutput}{\alpha}.$$
    For each $\PDE\in\pdeClass^{\text{smooth}}(\bar{\vartheta}),$ let  $\hat{\Psi}_{\PDE}$ be the ERM defined in \eqref{eq:empirical risk II} with $\eps\asymp\samples^{-\frac{1}{5+\rho}}.$
   Then, with probability at least
    $1 - 1/\samples,$ 
    \begin{align*}
        \sup_{\PDE\in \pdeClass^{\textup{smooth}}(\bar{\vartheta})}\risk(\hat{\Psi}_{\PDE},\PDE)\leq rM^{ -\frac{1}{5+\rho}}.
    \end{align*}
\end{theorem}
Theorem~\ref{thm: smooth learning bound} demonstrates that the operator class $ \pdeClass^{\text{smooth}}(\bar{\vartheta})$  can be learned with polynomial sample complexity using FNOs. As before, we contrast this result with the hardness result in \cite{kovachki2024data}, which shows that generic Fr\'echet differentiable or Lipschitz operators cannot be learned from only polynomially many samples. The structure imposed in $ \pdeClass^{\text{smooth}}(\bar{\vartheta})$ is enough to overcome this hardness, even though the spectral methods that approximate our target operators make only generic smoothness assumptions. The proof of Theorem~\ref{thm: smooth learning bound} is given in Section~\ref{sec:smooth learning proof} and proceeds similarly to the proof of Theorem~\ref{thm: learning}.

\subsection{Example: Learning the Solution Map of a Nonlinear Gradient Flow}\label{ssec:examplegeneralnonlinearity}
We consider the Cahn--Hilliard equation with a logarithmic Flory--Huggins potential on the two-dimensional periodic torus $\T^2$,
\begin{align}\label{eq:FH PDE}
    \frac{\partial u}{\partial t}- \Delta \bigl(-\nu \Delta u  +\potential'(u)\bigr)=0,
\end{align}
where $\potential:(-1,1)\to \R$ is referred to as the thermodynamic potential or homogeneous free energy, given by
\begin{align*}
    \potential(u)=\frac{\theta}{2}\Bigl((1+u)\log(1+u)+(1-u)\log(1-u)\Bigr)-\frac{\theta_c}{2}u^2.
\end{align*}
Here $\theta$ represents the absolute temperature, and $\theta_c$ represents the critical temperature. Under the standard condition that $0<\theta<\theta_c$, the potential $\potential$ has a double-well shape and phase separation can occur. The derivative of $\potential$ is given by
\begin{align*}
    \potential'(u)=\potentialDer(u)=\frac{\theta}{2}\log((1+u)/(1-u))-\theta_cu.
\end{align*}
The logarithmic potential considered here arises from first principles to model the entropy of mixing \cite{flory1942thermodynamics,huggins1942theory,elliott1991generalised}. The usual Cahn--Hilliard equation, as considered in Section~\ref{sec:Illustrative Examples}, can be derived as a quartic approximation to $\potential$,
\begin{align*}
    \potential(u)\approx \potential_{\text{quartic}}(u)=\frac{\theta}{2}\frac{u^4}{6}+\Bigl(\frac{\theta}{2}-\frac{\theta_c}{2}\Bigr)u^2,
\end{align*}
with the specific choice of $\frac{\theta}{\theta_c}=\frac{3}{4}$. From the perspective of physical modeling, the logarithmic potential considered here is more natural despite presenting additional mathematical and computational challenges due to the possibility of singularities in the potential.
We let $\PDE^{\textup{FH}}:L^2(\T^2,\R)\to L^2(\T^2,\R)$ denote the time-$T$ solution map of the Cahn--Hilliard equation with Flory--Huggins potential, mapping $u_0\mapsto u(T)=\PDE^{\text{FH}}(u_0).$ This map is only well-defined and physically meaningful for initial conditions satisfying $\|u_0\|_{L^{\infty}}\leq 1-\delta_0,$ for $\delta_0\in (0,1]$. As such, we restrict our attention to such inputs. We now verify that $\PDE^{\textup{FH}}\in \pdeClass^{\text{smooth}}(\bar{\vartheta})$ for 
$$\vartheta=(\dissipation,\gamma,\power,\inverseBound, s,\DcoefficientBound,\alpha, {\coefficientBound}, \forcingBound)=\Bigl(4,0,0,\nu,2,1,\alpha,C_{\alpha }(\theta_c + \theta(\alpha-1)!\delta_1^{-\alpha}),0 \Bigr)$$ and $\bar{\vartheta}=(\vartheta,\smoothness,\inputBall,\outputBall)=(\vartheta,5,\inputBall,2\pi(1-\delta_1))$ for any $\inputBall,\delta_0,\theta,\theta_c>0$ such that $0<\theta<\theta_c,$ $\nu\in (0,1]$, and $\alpha\geq 3.$ First we verify boundedness. From \cite[Corollary 1.1]{li2021stability}, there exists $\delta_1\in (0,1)$ depending on $\inputBall, \theta,\theta_c,$ and $\delta_0$ such that 
\begin{align*}
    \sup_{u:\|u\|_{H^5}\leq \inputBall}\|\PDE^{\textup{FH}}(u)\|_{L^{\infty}}\leq 1-\delta_1,
\end{align*}
implying that 
\begin{align*}
    \sup_{u:\|u\|_{H^5}\leq \inputBall}\|\PDE^{\textup{FH}}(u)\|_{L^{2}}\leq 2\pi(1-\delta_1):=\outputBall.
\end{align*}
This verifies Condition~\ref{conditionsec4} (I). We consider the semi-implicit Euler scheme
\begin{align}\label{eq:FH scheme}
    \frac{u^{j+1}_N-u^{j}_N}{\timestep}=-\nu \Delta^2u^{j+1}_N-\theta_c\Delta u_N^{j}+\Delta \mathcal{P}_N\frac{\theta}{2}\log\biggl(\frac{1+u_N^j}{1-u_N^j}\biggr),
\end{align}
where $\mathcal{P}_N:L^2(\T^2,\R)\to L^2_N(\T^2,\R)$ is the $L^2$-orthogonal projection onto the space of mean-zero degree-$N$ trigonometric polynomials. Theorem \ref{thm: FH stability} demonstrates that the iterates given by \ref{eq:FH scheme} remain bounded pointwise by a constant strictly less than $1$, such that  $\|u_N^j\|_{L^{\infty}}\leq 1-\delta_1$ for all $j\geq 0.$ To ensure the nonlinearity is globally well-defined on the entire domain space, as is required to apply Theorem \ref{thm: smooth learning bound}, we equivalently consider the extended scheme 
\begin{align}\label{eq:FH scheme 2}
    \frac{u^{j+1}_N-u^{j}_N}{\timestep}=-\nu \Delta^2u^{j+1}_N+\Delta \mathcal{P}_N \extension(u_N^j),
\end{align}
where $\extension:\mathbb{R}\rightarrow\mathbb{R}$ is a smooth extension of $\potentialDer:{[-1+\delta_1, 1-\delta_1]}\to \R$ to the entire real line. By classical extension results \cite[Chapter 6, Theorem 5]{stein1970singular}, $\extension$ can be constructed to satisfy $\extension(u) = \potentialDer(u)$ for all $u\in[-1+\delta_1, 1-\delta_1]$ and obey the uniform derivative bound
\begin{align}\label{eq:extension derivative bound}
  \max_{\omega:|\omega|\leq \alpha}\|\partial^{\omega} \extension\|_{L^{\infty}(\R)}\leq C_{\alpha}\max_{\omega:|\omega|\leq \alpha}\|\partial^{\omega} \potentialDer\|_{L^{\infty}([-1+\delta_1,1-\delta_1])},
\end{align}
where $C_{\alpha}$ is a universal constant depending solely on the smoothness $\alpha$. The iterates given by \eqref{eq:FH scheme} and \eqref{eq:FH scheme 2} coincide exactly provided $\|u_N^0\|_{L^{\infty}}\leq 1-\delta_0$ for $\delta_0<1$ and $u_N^0\in \mathsf{U}_5.$ 
We now demonstrate that the scheme \eqref{eq:FH scheme 2} can be written in the form of \eqref{eq:spec smooth} with
\begin{enumerate}
    \item $\linearOp^{\textup{FH}}:L^2(\T^2,\R)\to L^2_N(\T^2,\R)$ given by
    \begin{align*}
        \linearOp^{\text{FH}}&=\Bigl(\frac{1}{\timestep} I+ \mathcal{A}^{\text{FH}}\Bigr)^{-1}\mathcal{P}_N,
    \end{align*}
    which corresponds to $\gamma=\power=0$, where
    \begin{enumerate}
        \item $\mathcal{P}_N:L^2(\T^2,\R)\to L^2_N(\T^2,\R)$ is given by
        \begin{align*}
            (\mathcal{P}_Nu)(x)=\sum_{ k\in \findices_N\backslash\{0\}}\hat{u}(k)e^{i\langle k,x\rangle},
        \end{align*}
        for any $u\in L^2(\T^2,\R).$
        \item $\mathcal{A}^{\textup{FH}}:L^2_N(\T^2,\R)\to L^2_N(\T^2,\R)$ is given by
        \begin{align*}
            (\mathcal{A}^{\textup{FH}}u_N)(x)=(\nu\Delta^2 u_N)(x)=\nu\sum_{ k\in \findices_N\backslash\{0\}}\|k\|_2^4\:\hat{u}_N(k)e^{i\langle k,x\rangle},
        \end{align*}
        for any $u_N\in L^2_N(\T^2,\R)$, which implies $\dissipation=4$ and $\inverseBound=\nu.$
    \end{enumerate} 
    Consequently, $\linearOp^{\textup{FH}}\in \linearOpClass(4,0,0,\nu)$.
    \item $\mathcal{D}_s^{\textup{FH}}:H^{\alpha-1}(\T^{2},\R) \to L^2(\T^2,\R)$ given by
    \begin{align*}
        \mathcal{D}_s^{\textup{FH}}(u)=-\Delta u=-\sum_{j=1}^2\partial_{x_j}^2u,
    \end{align*}
    for any $u\in H^{\alpha-1}(\T^2,\R),$ which implies $s=2$ and $\DcoefficientBound=1$, so $\mathcal{D}_s^{\textup{FH}}\in \differentialClass(2,1).$ 
    \item $\nonlinearity^{\textup{FH}}:L^2_N(\T^2,\R)\to H^{\alpha-1}(\T^2,\R)$ given by
    \begin{align*}
        \nonlinearity^{\textup{FH}}(u_N)=\extension(u_N),
    \end{align*}
    where $\extension:\R\to \R$ is the smooth global extension of $\potentialDer$ introduced above. Since
    $$
\max_{\omega :|\omega|\leq \alpha}\|\partial^{\omega} \potentialDer\|_{L^{\infty}([-1+\delta_1,1-\delta_1])}\leq\theta_c + \theta(\alpha-1)!\frac{1}{\delta_1^\alpha},
$$
it follows from \ref{eq:extension derivative bound} that
$$
\max_{\omega:|\omega|\leq \alpha}\|\partial^{\omega}\extension\|_{L^{\infty}}\leq C_{\alpha}\bigl(\theta_c + \theta(\alpha-1)!\frac{1}{\delta_1^\alpha}\bigr),
$$
and thus $\nonlinearity^{\textup{FH}}\in \nonlinearClass(\alpha, C_{\alpha }(\theta_c + \theta(\alpha-1)!\delta_{1}^{-\alpha}))$ for any $\alpha\geq 3.$
    \end{enumerate}
    In Theorem~\ref{thm: FH stability}, we show that for $u_0\in \mathsf{U}_5$ with $\|u_0\|_{L^\infty}\leq 1-\delta_0$, 
    $$
    \sup_{j\geq 1} \left\|(\spectralStep^{\textup{FH}})^j(\mathcal{I}_N u_0) \right\|_{L^{\infty}}\leq 1-\delta_1,
    $$
    provided $\timestep$ is chosen small enough and $N$ is chosen large enough relative to $\inputBall,\delta_0,\nu,\theta,\theta_c.$ Thus 
    $$
    \sup_{j\geq 1}\|(\spectralStep^{\textup{FH}})^j(\mathcal{I}_N u_0)\|_{L^{2}}\leq 2\pi(1-\delta_1)=\outputBall.
    $$
    This verifies the first part of Condition~\ref{conditionsec4} (II). Then in Theorem~\ref{thm:FH accuracy}, we show that this scheme admits an error bound given by
    \begin{align*}
        \left\|u(\timestep j)-(\spectralStep^{\textup{FH}})^j(\mathcal{I}_Nu_0) \right\|_{L^2}\leq r_1e^{r_2j\timestep}\Bigl(N^{-5}+\timestep\Bigr)
    \end{align*}
    for $\|u_0\|_{H^5}\leq \inputBall$ with $\|u_0\|_{L^{\infty}}\leq 1-\delta_0$ for $\tau$ appropriately small and $N$ appropriately large relative to $\inputBall, \delta_0,
    \nu, \theta, \theta_c$. This implies that $\smoothness=5$ and verifies the second part of Condition~\ref{conditionsec4} (II). Consequently, the logarithmic Cahn-Hilliard solution operator belongs to $\pdeClass^{\text{smooth}}(\bar{\vartheta})$, and therefore satisfies the approximation and learning guarantees established in Theorem~\ref{thm: smooth learning bound}. We obtain the following corollary:
    \begin{corollary}[Learning Logarithmic Cahn--Hilliard Solution Map]
    Let $u^{(1)},\hdots,u^{(\samples)}\overset{\textup{i.i.d.}}{\sim}\mu$ with $\textup{supp}(\mu)\subset \mathsf{U}_5\cap \{u\in L^{\infty}:\|u\|_{L^{\infty}}\leq 1-\delta_0\}.$ Let $\eps=r\samples^{-\frac{1}{7+\frac{24}{5\alpha}}}$ and set
    \begin{align*}
        \hat{\Psi}\in \argmin_{\Psi\in \Sigma^{\textup{smooth}}(\eps,\bar{\vartheta})}\frac{1}{\samples}\sum_{\sample=1}^{\samples}\|\Psi(  u^{(\sample)}  )-\PDE^{\textup{FH}}(   u^{(\sample)}   )\|_{L^2}^2.
    \end{align*}
     Then, with probability at least $1-1/\samples$, it holds that
    \begin{align*}
       \E_{u\sim \mu} \bigl\|\hat{\Psi}(u)-\PDE^{\textup{FH}}(u) \bigr\|_{L^2}^2\leq r \samples^{-\frac{1}{7+\frac{24}{5\alpha}}}.
    \end{align*}
\end{corollary}
Here the rate $M^{\frac{-1}{7+\frac{24}{5\alpha}}}$
is determined by the spatial smoothness parameter $\smoothness=5$, the dimension of the physical space $d=2$, and the smoothness of the nonlinearity $\alpha$. While $\alpha$ could be taken to be arbitrarily large to decrease the rate, the constant $r$ rapidly blows up with $\alpha$ as $r\gtrsim  \exp\Bigl(\frac{T}{\nu} (\theta_c + \theta(\alpha-1)!\frac{1}{\delta_1^\alpha})^2\Bigr)$ because of the singularities of the logarithmic potential at $\pm1$. Consequently, taking $\alpha$ to be too large quickly leads to a vacuous bound. Given a constrained number of samples, one can in principle optimize this bound over $\alpha$; however we do not pursue this here.

\section{Conclusions}
We developed approximation and learning guarantees for FNOs applied to time-$T$ solution operators of dissipative evolution equations. The main principle underlying our results is that solution operators admitting stable and accurate Fourier spectral discretizations can be efficiently represented by FNOs and learned from data with polynomial sample complexity.

We formalized this principle by introducing classes of evolution operators defined through spectral methods and proving quantitative FNO approximation and learning bounds for these classes. For polynomial nonlinearities, the rates depend  on the Sobolev smoothness of the input space and the dimension of the physical domain. We verified that the framework applies to the Navier--Stokes, Allen--Cahn, and Cahn--Hilliard equations, and extended the theory to general smooth nonlinearities, where the rates also depend on the smoothness of the nonlinearity and the strength of dissipation.

These results identify stable spectral approximability as a mechanism leading to parameter-efficient FNO representations and data-efficient learning. Natural directions for future work include non-periodic domains, more general boundary conditions, and broader classes of numerical discretizations and operator-learning architectures.

Our learning guarantees for dissipative evolution equations are in part motivated by the use of data-driven surrogate forecasting models within data assimilation algorithms \cite{adrian2025data,kotsuki2025ensemble}. Current theory suggests that to fully couple learning theoretic guarantees for surrogate models with accuracy guarantees for data assimilation, the learning bounds should be in a higher-order Sobolev space than $L^2$ \cite{sanz2025long,li2026continuousdataassimilationlearned}. Pursuing learning guarantees in higher-order norms is therefore a natural direction for future work.
\section*{Acknowledgments}
DSA and NW were partly funded by NSF CAREER DMS-2237628. The authors are thankful to Jiaheng Chen for helpful feedback and discussions.

\bibliographystyle{siam}
\bibliography{references}

\appendix

\section{Background on Multilayer Perceptrons and FNOs}
This appendix overviews existing approximation results for multilayer perceptrons and FNOs. 
\subsection{Multilayer Perceptrons}
\begin{definition}[Multilayer Perceptron]
    Let $L\in \N$ and $d_0,\ldots,d_L\in \N$. A multilayer perceptron (MLP) is a mapping $\Phi:\R^{d_0}\to\R^{d_L}$ of the form  
    \begin{align}
        \Phi=\begin{cases}
            W_1, & L=1,\\
            W_2\circ \sigma \circ W_1, & L=2,\\
            W_L \circ \sigma \circ W_{L-1} \circ \cdots \circ \sigma \circ W_1, & L\geq 3,
        \end{cases}
    \end{align}
    where, for $\ell\in \{1,\hdots,L\}$, $W_{\ell}: \R^{d_{\ell-1}}\to \R^{d_{\ell}}$ is given by $W_{\ell}(z)={\bf A}_{\ell}z+b_{\ell}$   with ${\bf A}_{\ell}\in \R^{d_{\ell}\times d_{\ell-1}}$ and $b_{\ell}\in \R^{d_{\ell}}$. The activation function $\sigma$  is applied component-wise. 
\end{definition}
 We characterize the size of an MLP through the following quantities:
    \begin{itemize}
        \item The \textit{depth} $\mathscr{D}(\Phi):=L$;
        \item The \textit{width} $\mathscr{W}(\Phi):= \max_{\ell=0,\ldots,L}d_{\ell}$;
        \item The \textit{weight magnitude} $\mathscr{S}(\Phi):=\max_{\ell=1,\ldots,L}\max \bigl\{\| {\bf A}_{\ell}\|_{\infty}, \|b_{\ell}\|_{\infty} \bigr\}$.
    \end{itemize}
In the sequel, we restrict our attention to the rectified linear unit (ReLU) activation function, $\sigma(z)=\max\{0,z\}.$
We let $\Nc_{\din,\dout}$ denote the set of all ReLU networks with input dimension $\din$ and output dimension $\dout$.
The following lemmas from \cite{elbrachter2021deep} will be useful for constructing such networks.
\begin{lemma}[MLP Composition]\label{lemma:MLP Composition}
    Let $d_1,d_2,d_3\in \N$, $\Phi_1\in \Nc_{d_1,d_2}$, and $\Phi_2\in \Nc_{d_2,d_3}.$ Then, there exists a network $\Phi_3\in \Nc_{d_1,d_3}$ such that
    \begin{align*}
        \Phi_3(z)=(\Phi_2\circ \Phi_1)(z)=\Phi_2\bigl(\Phi_1(z)\bigr), \quad \forall z\in \R^{d_1},
    \end{align*}
    and moreover
        \begin{align*}
        \mathscr{D}(\Phi_3)&=\mathscr{D}(\Phi_1)+\mathscr{D}(\Phi_2),\\
        \mathscr{W}(\Phi_3) & =\max\bigl\{2d_2,\mathscr{W}(\Phi_1),\mathscr{W}(\Phi_2)\bigr\},\\
        \mathscr{S}(\Phi_3)&=\max\bigl\{\mathscr{S}(\Phi_1),\mathscr{S}(\Phi_2)\bigr\}.
    \end{align*}
\end{lemma}
\begin{lemma}[MLP Depth Augmentation]\label{lemma:MLP depth augmentation}
Let $d_1,d_2,K\in \N$ and $\Phi\in \Nc_{d_1,d_2}$ with $\mathscr{D}(\Phi)\leq K$. Then, there exists a network $\Phi'\in \Nc_{d_1,d_2}$ such that
\begin{align*}
    \Phi'(z)=\Phi(z), \quad \forall z\in \R^{d_1},
\end{align*}
and moreover
\begin{align*}
    \mathscr{D}(\Phi')&=K,\\
    \mathscr{W}(\Phi')&=\max\bigl\{2d_2,\mathscr{W}(\Phi)\bigr\},\\
    \mathscr{S}(\Phi')&=\max\bigl\{1,\mathscr{S}(\Phi)\bigr\}.
\end{align*}
\end{lemma}

\begin{lemma}[MLP Parallelization]\label{lemma:MLP parallelization} 
    Let $n,L\in \N$ and, for $i\in \{1,\hdots,n\},$ let $d_i,d_i'\in \N$ and $\Phi_i\in \Nc_{d_i,d_i'}$ with $\mathscr{D}(\Phi_i)=L$. Then, there exists a network $\Phi \in \Nc_{\sum_{i=1}^n d_i,\sum_{i=1}^n d_i'}$ such that, for all $z=(z_1,\ldots,z_n)\in \R^{\sum_{i=1}^n d_i}$ with $z_i\in \R^{d_i},$
    \begin{align*}
        \Phi(z)=\bigl(\Phi_1(z_1),\hdots,\Phi_n(z_n)\bigr)\in \R^{\sum_{i=1}^n d_i'},
    \end{align*}
     and moreover
    \begin{align*}
        \mathscr{D}(\Phi)&=L,\\
        \mathscr{W}(\Phi)&=\sum_{i=1}^n\mathscr{W}(\Phi_i),\\
        \mathscr{S}(\Phi)&=\max_i\mathscr{S}(\Phi_i).
    \end{align*}
\end{lemma}

\begin{lemma}[MLP Linear Combination]\label{lemma:MLP linear combination}
    Let $n,L,\doutput\in \N$ and, for $i\in \{1,\hdots,n\},$ let $d_i\in \N$, $a_i\in \R$, and $\Phi_i\in \mathcal{N}_{d_i,\doutput}$ with $\mathscr{D}(\Phi_i)=L.$ Then, there exists a network $\Phi\in \mathcal{N}_{\sum_{i=1}^nd_i,\doutput}$ such that, for all $z=(z_1,\hdots,z_n)\in \R^{\sum_{i=1}^nd_i}$ with $z_i\in \R^{d_i},$ $i\in \{1,\hdots, n\},$
    \begin{align*}
        \Phi(z)=\sum_{i=1}^na_i\Phi_i(z_i)\in \R^{\doutput},
    \end{align*}
    and moreover
    \begin{align*}
        \mathscr{D}(\Phi)&=L,\\
        \mathscr{W}(\Phi)&\leq \sum_{i=1}^n\mathscr{W}(\Phi_i),\\
        \mathscr{S}(\Phi)&=\max_{i}\bigl\{|a_i|\mathscr{S}(\Phi_i)\bigr\}.
    \end{align*}
\end{lemma}

\begin{lemma}[MLP Multiplication Approximation]\label{lemma:multiplication approximation}
    There exists a constant $r>0$ such that for all $\mlpBound\in \R_+$ and $\eps\in (0,1/2)$, there is a network $\Phi_{\mlpBound,\eps}\in  \Nc_{2,1}$ such that
    \begin{align*}
     \Phi_{\mlpBound,\eps}(z,0)=\Phi_{\mlpBound,\eps}(0,z)&=0, \qquad \forall z \in \R, \\
        \sup_{(z_1,z_2)\in [-\mlpBound,\mlpBound]^2} \bigl|\Phi_{\mlpBound,\eps}(z_1,z_2)-z_1z_2 \bigr| & \leq \eps,
    \end{align*}
and moreover
    \begin{align*}
        \mathscr{D}(\Phi_{\mlpBound,\eps})&\leq r\bigl(\log (\lceil \mlpBound \rceil)+\log(\eps^{-1})\bigr),\\
        \mathscr{W}(\Phi_{\mlpBound,\eps})&\leq 5,\\
        \mathscr{S}(\Phi_{\mlpBound,\eps})&=1.
    \end{align*} 
    \end{lemma}

Consider the set of polynomials of $n$ variables of degree at most $p$,
\begin{align*}
h(z)=c_0+\sum_{i_1}c_{i_1}z_{i_1}+\sum_{i_1,i_2}c_{i_1,i_2}z_{i_1}z_{i_2}+\cdots+\sum_{i_1,\hdots,i_p}c_{i_1,\hdots, i_p}z_{i_1}z_{i_2}\hdots z_{i_p},
\end{align*}
where the coefficient vector $c=(c_0,c_1,\hdots, c_n, c_{1,1},\hdots, c_{n,n},\hdots)$ is of length ${p+n\choose n}$. In the following lemma, we quantify the size of a ReLU network sufficient to approximate a generic polynomial of $n$ variables of degree at most $p$.

\begin{lemma}[MLP Approximation of Polynomials]\label{lemma: polynomial approximation}
    There exists a constant $r>0$ such that for all $\mlpBound\in \R_+$, $\dimz,p\in \N$, and $\eps\in (0,1/2)$, there is a network $\Phi_{\mlpBound,\dimz,p,\eps}\in \mathcal{N}_{\dimz,1}$ such that
    \begin{align*}
        \|\Phi_{\mlpBound,\dimz,p,\eps}(z)-h(z)\|_{L^{\infty}([-\mlpBound,\mlpBound]^{\dimz})}&\leq\eps,
    \end{align*}
and moreover
    \begin{align*}
        \mathscr{D}(\Phi_{\mlpBound,\dimz,p,\eps})& \le  r p \Big(p\log(\lceil \mlpBound\rceil  ) +\log(\eps^{-1})+p\log(\dimz)+\log(p)+\log(\|c\|_{\infty})\Bigr),\\
        \mathscr{W}(\Phi_{\mlpBound,\dimz,p,\eps})& \le 10 {p+\dimz-1 \choose \dimz-1}+2\dimz,\\
        \mathscr{S}(\Phi_{\mlpBound,\dimz,p,\eps})& \le \|c\|_{\infty}.
    \end{align*}
    \begin{proof}
        We follow \cite[Proposition III.5]{elbrachter2021deep}, and iteratively compose multiplication networks that can approximate the monomials in parallel with a network that implements linear combinations of the monomial terms. Recall that there are 
        \begin{align*}
            d_{\mInd} = {\mInd + \dimz-1 \choose\dimz-1}\leq \dimz^\mInd,
        \end{align*}
        degree $\mInd$ monomials of $\dimz$ variables. 
        We consider the composition of $p$ MLPs, where $\Phi_1\in \mathcal{N}_{\dimz,d_2+\dimz+1}$, $\Phi_\mInd\in \mathcal{N}_{d_{\mInd}+\dimz+1,d_{\mInd+1}+\dimz+1}$ for $2\leq \mInd \leq p-1$, and $\Phi_p\in \mathcal{N}_{d_p+\dimz+1,1}$, implementing the following mappings. We use the notation $\Phi_{i_1\hdots i_j}$ for the multiplication network associated with the degree-$j$ monomial $z_{i_1}\cdots z_{i_j}$. Thus, in the repeated-index entries below, the number of repeated indices is determined by the degree of the corresponding monomial; for example, in the line defining $z^{(p)}$, the outer $\Phi_{1\hdots 1}$ has $p$ copies of $1$ in its subscript, while the inner one has $p-1$ copies.
    
        \begin{align*}
    \Phi_1(z)&= z^{(2)}=\begin{pmatrix}
        z\\
        \Phi_{11}(z_1,z_1)\\
        \vdots \\
        \Phi_{i_1i_2}(z_{i_1},z_{i_2})\\
        \vdots \\
        \Phi_{\dimz\dimz}(z_{\dimz},z_{\dimz})\\
        c_0+\sum_{i=1}^{\dimz}c_iz_i
    \end{pmatrix}\in \R^{\dimz+d_{2}+1},\\
    \Phi_2(z^{(2)})&=z^{(3)} =\begin{pmatrix}
        z\\
        \Phi_{111}(\Phi_{11},z_1)\\
        \vdots \\
        \Phi_{i_1i_2i_3}(\Phi_{i_1i_2},z_{i_3})\\
        \vdots \\
        \Phi_{\dimz\dimz \dimz}(\Phi_{\dimz\dimz},z_{\dimz})\\
        c_0+\sum_{i=1}^{\dimz}c_iz_i+\sum_{i_1,i_2}c_{i_1,i_2}\Phi_{i_1i_2}(z_{i_1},z_{i_2})
    \end{pmatrix}\in\R^{\dimz+d_{3}+1},\\
    &\vdots \\
    \Phi_{p-1}(z^{(p-1)})&=z^{(p)}=\begin{pmatrix}
        z\\
        \Phi_{1\hdots 1}(\Phi_{1\hdots 1},z_1)\\
        \vdots \\
        \Phi_{i_1\hdots i_p}(\Phi_{i_1\hdots i_{p-1}},z_{i_p}) \\
        \vdots\\
        \Phi_{\dimz\hdots \dimz}(\Phi_{\dimz\hdots \dimz},z_{\dimz})\\
c_0+\sum_{i=1}^{\dimz}c_iz_i+\cdots +\sum_{i_1,\hdots, i_{p-1}}c_{i_1,\hdots , i_{p-1}}\Phi_{i_1\hdots i_{p-1}}(\Phi_{i_1\hdots i_{p-2}},z_{i_{p-1}})\end{pmatrix}\in \R^{\dimz+d_p+1},\\
\Phi_p(z^{(p)}) & =c_0+\sum_{i=1}^{\dimz}c_iz_i+\sum_{i_1,i_2}c_{i_1,i_2}\Phi_{i_1i_2}(z_{i_1},z_{i_2})+\cdots +\sum_{i_1,\hdots, i_p}c_{i_1,\hdots , i_p}\Phi_{i_1\hdots i_p}(\Phi_{i_1\hdots i_{p-1}},z_{i_p})\in \R,
\end{align*}
where, for $j\geq 3$, $\Phi_{i_1\hdots i_j}$ is obtained by applying the two-input multiplication network from Lemma~\ref{lemma:multiplication approximation} to $\Phi_{i_1\hdots i_{j-1}}$ and $z_{i_j}$:
$$
\Phi_{i_1\hdots i_j}
=
\Phi_{\mlpBound',\eps'}(\Phi_{i_1\hdots i_{j-1}},z_{i_j}),
$$
for some $\mlpBound'$ and $\eps'$ to be chosen later. For $j=2$, we set
$$
\Phi_{i_1i_2}=\Phi_{\mlpBound',\eps'}(z_{i_1},z_{i_2}).
$$
Thus, each multiplication block used in the construction satisfies
\begin{align*}
    \mathscr{D}(\Phi_{\mlpBound',\eps'})&\leq r\bigl(\log (\lceil \mlpBound' \rceil)+\log(\eps'^{-1})\bigr),\\
    \mathscr{W}(\Phi_{\mlpBound',\eps'})&\leq 5,\\
    \mathscr{S}(\Phi_{\mlpBound',\eps'})&=1 .
\end{align*}
    By Lemmas~\ref{lemma:MLP Composition}, \ref{lemma:MLP depth augmentation}, \ref{lemma:MLP parallelization}, and \ref{lemma:MLP linear combination}, there exists a network $\Phi\in \mathcal{N}_{\dimz,1}$ such that
    \begin{align*}
        \Phi(z)=(\Phi_p\circ \cdots \circ\Phi_1 )(z), \qquad \forall z\in \R^{\dimz},
    \end{align*}
    and moreover  
    \begin{align*}
        \mathscr{D}(\Phi)&\leq rp\bigl(\log (\lceil \mlpBound'\rceil)+\log(\lceil\eps'^{-1}\rceil)\bigr),\\
        \mathscr{W}(\Phi) & \leq10d_p+\dimz,\\
        \mathscr{S}(\Phi) & = \|c\|_{\infty}.
    \end{align*}
    It remains to choose $\eps'$ and $\mlpBound'$ appropriately to guarantee the desired accuracy. We define $\inductionBound_j=\lceil \mlpBound\rceil^j+\eps'\sum_{k=0}^{j-2}\lceil \mlpBound\rceil^k$, and assume that $\mlpBound'$ is taken to be at least as large as $\inductionBound_j$ for all $j\leq p$. We then show by induction that for any $(i_1,\hdots i_j)\in [\dimz]^j$,
    \begin{align*}
       \|\Phi_{i_1\hdots i_j}(z)-z_{i_1}\hdots z_{i_j}\|_{L^{\infty}([-\mlpBound,\mlpBound]^{\dimz})}\leq \eps' \sum_{k=0}^{j-2} \lceil \mlpBound\rceil^k
    \end{align*}
    for all $j\geq 2.$ The base case of $j=2$ follows directly from Lemma~\ref{lemma:multiplication approximation},
    \begin{align*}
        \|\Phi_{i_1 i_2}(z_{i_1},z_{i_2})-z_{i_1}z_{i_2}\|_{L^{\infty}([-\mlpBound,\mlpBound]^2)}\leq \eps'.
    \end{align*}
    We assume by induction that 
    \begin{align*}
        \|\Phi_{i_1\hdots i_{j-1}}(z)-z_{i_1}\hdots z_{i_{j-1}}\|_{L^\infty([-\mlpBound,\mlpBound]^{\dimz})}\leq \eps' \sum_{k=0}^{j-3}\lceil\mlpBound\rceil ^k.
    \end{align*}
    Since
    \begin{align*}
        \|\Phi_{i_1 \hdots i_{j-1}}\|_{L^\infty([-\mlpBound,\mlpBound]^{\dimz})}\leq \|z_{i_1}\hdots z_{i_{j-1}}\|_{L^\infty([-\mlpBound,\mlpBound]^{\dimz})}+\|\Phi_{i_1 \hdots i_{j-1}}(z)-z_{i_1}\hdots z_{i_{j-1}}\|_{L^\infty([-\mlpBound,\mlpBound]^{\dimz})}\leq \inductionBound_{j-1},
    \end{align*}
    we have that
    \begin{align*}
         \|\Phi_{i_1\hdots i_j}(z)-z_{i_1}\hdots z_{i_j}\|_{L^{\infty}([-\mlpBound,\mlpBound]^{\dimz})}&\leq \|\Phi_{i_1\hdots i_{j}}(z)-z_{i_j}\Phi_{i_1\hdots i_{j-1}}(z)\|_{L^\infty([-\mlpBound,\mlpBound]^{\dimz})}\\
         &+\|z_{i_j}\Phi_{i_1\hdots i_{j-1}}(z)-z_{i_1}\hdots z_{i_j}\|_{L^\infty([-\mlpBound,\mlpBound]^{\dimz})}\\
         & \leq\|\Phi_{i_1\hdots i_j}(\Phi_{i_1\hdots i_{j-1}}(z),z_{i_{j}})-z_{i_j}\Phi_{i_1\hdots i_{j-1}}(z)\|_{L^\infty([-\mlpBound,\mlpBound]^{\dimz})}\\
         & +\mlpBound \|\Phi_{i_1\hdots i_{j-1}}(z)-z_{i_1}\hdots z_{i_{j-1}}\|_{L^\infty([-\mlpBound,\mlpBound]^{\dimz})}\\
         & \leq\eps'+\mlpBound \eps' \sum_{k=0}^{j-3}\lceil \mlpBound\rceil^k\\
         & \leq \eps' \sum_{k=0}^{j-2}\lceil \mlpBound\rceil^k,
    \end{align*}
    where the second-to-last inequality uses Lemma~\ref{lemma:multiplication approximation} and the inductive hypothesis. This completes the induction. To conclude, we note that
    \begin{align*}
        \|\Phi(z)-h(z)\|_{L^{\infty}([-\mlpBound,\mlpBound]^{\dimz})}&\leq \sum_{i_1,\hdots, i_p}|c_{i_1, \hdots, i_p}|\|\Phi_{i_1\hdots i_{p}}(z)-z_{i_1}\hdots z_{i_p}\|_{L^\infty([-\mlpBound,\mlpBound])^{\dimz}}\\
         & \leq p d_p\|c\|_{\infty}\eps' \sum_{k=0}^{p-2}\lceil \mlpBound\rceil^k\\
          & \leq p^2\dimz^p\|c\|_{\infty}\lceil \mlpBound \rceil^p \eps'.
    \end{align*}
    Taking $\eps'=\frac{\eps}{\dimz^pp^2\|c\|_{\infty}\lceil\mlpBound\rceil^p}$ and $\mlpBound'=2\lceil\mlpBound\rceil^p$ gives the desired approximation guarantee. 
    \end{proof}
\end{lemma}
For any domain $\Omega \subseteq \R^d$, we define 
\begin{align*}
    \smoothClass_{\Omega}(C):=\biggl\{f\in  L^{\infty}(\Omega):\max_{\kappa :|\kappa|\leq \alpha}\|D^{\kappa} f\|_{L^{\infty}(\Omega)}\leq C\biggr\}.
\end{align*}
When $\Omega=\R^d$, we typically write $\smoothClass_{\R^d}(C)=\smoothClass(C).$

\begin{lemma}[MLP Approximation of Smooth Functions]\label{lemma: sobolev approximation}
    Fix $\alpha,\dimz\in \N$. There exists a constant $r>0$ depending only on $\dimz$ and $\alpha$ such that for any $h\in \smoothClass(C)$ and any $\mlpBound\in \R_+$, and any $\eps\in (0,1/2)$, there is a network $\Phi_{\mlpBound,\alpha,\eps}\in \mathcal{N}_{\dimz,1}$ such that 
    \begin{align*}
        \|\Phi_{\mlpBound,\alpha,\eps}(z)-h(z)\|_{L^{\infty}([-\mlpBound,\mlpBound]^{\dimz})}\leq\eps,
    \end{align*}
   and moreover
    \begin{align*}
       \mathscr{D}(\Phi_{h})& = r \bigl(\log(\dimz+\alpha)+\alpha\dimz\log(\eps^{-1}) \bigr),\\
        \mathscr{W}(\Phi_{h})&=r\mlpBound^{\dimz}\eps^{-\frac{\dimz}{\alpha}}\coefficientBound^{\dimz},\\
        \mathscr{S}(\Phi_{h})&= r\eps^{-\frac{1}{\alpha}}\mlpBound^{\alpha+1}\coefficientBound^2.
   \end{align*}
    
    \begin{proof}
    We use the same construction as in \cite[Theorem 1]{yarotsky2017error}, with only minor modifications to extend the result from $\smoothClass_{[0,1]^{\dimz}}(1)$ to $\smoothClass_{[-\mlpBound,\mlpBound]^{\dimz}}(C)$. For completeness, we will write out the argument in full detail. We will also keep track of an upper bound on the network weights when doing so. 
    
    First, note that for any $h\in \smoothClass(C)$, it also holds that $h\in \smoothClass_{[-\mlpBound,\mlpBound]^{\dimz}}(C)$. Further, making a change of variables, there exists $g\in \smoothClass_{[0,1]^{\dimz}}(2^{\alpha}\mlpBound^{\alpha}C)$ such that $g(y)=h(2\mlpBound y-\mlpBound)$ for $y\in [0,1]^{\dimz}$, and thus we proceed with the argument in \cite{yarotsky2017error} to approximate $g$, with the only difference being we have $\alpha$ weak derivatives bounded by $2^{\alpha}\mlpBound^{\alpha}C$ instead of by $1$. For $N\in \N$ to be chosen later, we consider a partition of unity by $(N+1)^d$ functions $\phi_{m}$ on $[0,1]^{\dimz}$ given by
    \begin{align*}
        \sum_{m}\phi_m(y)=1, \quad y\in [0,1]^{\dimz},
    \end{align*}
    where $m=(m_1,\hdots, m_{\dimz})\in \{0,1,\hdots, N\}^{\dimz}$ and $\phi_m$ is given by
    \begin{align*}
        \phi_m(y)=\prod_{k=1}^{\dimz}\psi \Bigl(3N \bigl(y_k-\frac{m_k}{N} \bigr) \Bigr),
    \end{align*}
    where
    \begin{align*}
        \psi(y)=\begin{cases}
            1,  & |y|\leq 1,\\
            0, & 2 <|y|,\\
            2-|y|, & 1\leq |y|\leq 2.
        \end{cases}
    \end{align*}
    Note that $\|\psi\|_{L^{\infty}}=1$ and $\|\phi_m\|_{L^{\infty}}=1$ for every $m,$ and moreover
    \begin{align*}
        \text{support}(\phi_m)\subset \left\{y: \left|y_k-\frac{m_k}{N} \right|\leq \frac{1}{N} \quad \forall 1\leq k\leq \dimz \right\}.
    \end{align*}
    For any $m\in \{0,\hdots,N\}^{\dimz}$ we consider the Taylor polynomial of degree $\alpha-1$ for $g$ at $y=\frac{m}{N}:$
    \begin{align*}
        P_m(y)=\sum_{d:|\omega|<\alpha}\frac{d^{\omega} g}{\omega!}\Big|_{y=\frac{m}{N}} \Bigl(y-\frac{m}{N}\Bigr)^\omega,
    \end{align*}
    where $\omega!=\prod_{k=1}^{\dimz}\omega_k!$ and $(y-\frac{m}{N})^\omega=\prod_{k=1}^{\dimz}(y_k-\frac{m_k}{N})^{\omega_k}.$ We then consider an approximation to $g$ given by
    \begin{align*}
        g_1=\sum_{m\in \{0,\hdots,N\}^{\dimz}}\phi_m P_m,
    \end{align*}
    which satisfies
    \begin{align*}
        \sup_{y\in [0,1]^{\dimz}}|g(y)-g_1(y)|&\leq \sup_{y\in [0,1]^{\dimz}} \sum_{m:|y_k-\frac{m_k}{N}|<\frac{1}{N} \forall k}|g(y)-P_m(y)|\\
        &\leq  \sup_{y\in [0,1]^{\dimz}} 2^{\dimz}\max_{m:|y_k-\frac{m_k}{N}|<\frac{1}{N} \forall k}|g(y)-P_m(y)|\\
        & \leq \frac{2^{\dimz+\alpha}\dimz^{\alpha}\mlpBound^{\alpha}C}{\alpha!}\times \frac{1}{N^\alpha},
    \end{align*}
    by the Taylor remainder formula. Taking $N=\lceil \eps^{-\frac{1}{\alpha}}2^{\dimz/\alpha+1+1/\alpha}\dimz\mlpBound\coefficientBound\rceil$ guarantees that $\|g-g_1\|_{L^{\infty}}\leq \frac{\eps}{2}.$
    It remains to show that $g_1$ can be approximated to uniform accuracy $\frac{\eps}{2}$ by a network of the desired size. Expanding $g_1$, we have
    \begin{align*}
        g_1(y)=\sum_{m\in \{0,\hdots,N\}^{\dimz}}\sum_{\omega:|\omega|\leq \alpha}a_{m,\omega}\phi_\omega(y) \Bigl(y-\frac{m}{N}\Bigr)^\omega,
    \end{align*}
    where $a_{m,\omega}=\frac{D^\omega g}{\omega!}\Big|_{y=\frac{m}{N}}\leq 2^{\alpha}\mlpBound^{\alpha}C$. We  see that $g_1$ is a linear combination of at most $\dimz^{\alpha}(N+1)^{\dimz}$ terms $\phi_m(y)(y-\frac{m}{N})^{\omega}$. Each  one of these terms is a product of $\dimz+\alpha-1$ piece-wise linear terms, including $\dimz$ functions $\psi(3Ny_k-3m_k)$ and $\alpha-1$ linear terms $y_k-\frac{m_k}{N}$. Note that each $\psi(3Ny_k-3m_k)$ is univariate, and by \cite[Proposition 1]{yarotsky2017error} can be implemented exactly by a neural network $\Phi_k$ with size bounded by
    \begin{align*}
        \mathscr{D}(\Phi_k)&=2,\\
        \mathscr{W}(\Phi_k)&=4,\\
        \mathscr{S}(\Phi_k)&\leq 3N.
    \end{align*}
 Let $\Phi^{\text{mult}}\in \mathcal{N}_{2,1}$ denote the multiplication network constructed in Lemma~\ref{lemma:multiplication approximation} with $\mlpBound=\dimz+\alpha$ and $\eps_1<1/2$ to be chosen. We consider approximating $\phi_m$ by the composition
   \begin{align*}
          \Phi_{m,\omega}'(y):=\Phi^{\text{mult}}(\Phi_{{\dimz}}(y_{\dimz}),\Phi^{\text{mult}}(\Phi_{\dimz-1}(y_{\dimz-1}),\ldots,\Phi_{\text{mult}}(\Phi_{2}(y_2),\Phi_{1}(y_1)))\hdots ).
   \end{align*}
   We approximate each $(y_k-\frac{m_k}{N})^{\omega_k}$ for $1\leq k \leq \dimz$ by composing multiplication networks to enumerate each of the terms
   \begin{align*}
       \Phi_{m_k,\omega_k}(y_k):=\underbrace{\Phi^{\text{mult}}(y_{k}-\frac{m_{k}}{N},\Phi^{\text{mult}}(y_k-\frac{m_k}{N},\Phi^{\text{mult}}\hdots,y_k-\frac{m_1}{N} )\hdots )}_{\omega_k \text{ times}}
   \end{align*}
   and composing these networks to approximate $(y-\frac{m}{N})^\omega$ as
   \begin{align*}
       \Phi_{m,\omega}''(y):=\Phi^{\text{mult}}(\Phi_{m_{\dimz},\omega_{\dimz}}(y),\Phi^{\text{mult}}(\Phi_{m_{\dimz-1},\omega_{\dimz-1}}(y),\ldots,\Phi_{\text{mult}}(\Phi_{m_2,\omega_2}(y),\Phi_{m_1,\omega_1}(y)))\hdots ).
   \end{align*}
   Finally, we compose these networks to approximate $\phi_m(y)(y-\frac{m}{N})^{\omega}$
   \begin{align*}
       \Phi_{m,\omega}(y):=\Phi^{\text{mult}}(\Phi_{m,\omega}'(y),\Phi_{m,\omega}''(y)),
   \end{align*}
   which by Lemmas~\ref{lemma:multiplication approximation} and \ref{lemma:MLP Composition} satisfies
   \begin{align*}
       \sup_{z\in [0,1]^{\dimz}}\left|\phi_m(y)(y-\frac{m}{N})^{\omega}-\Phi_{m,\omega}(y)\right|\leq (\dimz+\alpha)\eps_1
   \end{align*}
   with \begin{align*}
       \mathscr{D}(\Phi_{m,\omega})&=r \bigl(\log(\dimz+\alpha)+\log(\eps_1^{-1})\bigr),\\
        \mathscr{W}(\Phi_{m,\omega})&=5,\\
        \mathscr{S}(\Phi_{m,\omega})&\leq 3N,
   \end{align*}
   where $r$ is a constant depending on $\dimz$ and $\alpha$ that may change line by line.  We now consider the full approximation given by
   \begin{align*}
       \Phi_g(y)=\sum_{m\in \{0,\hdots , N\}^{\dimz}}\sum_{\omega:|\omega|\leq \alpha}a_{m,\omega} \Phi_{m,\omega}(y),
   \end{align*}
   which, by the same argument as in \cite{yarotsky2017error}, has error bounded by
   \begin{align*}
       \sup_{y\in [0,1]^{\dimz}}|\Phi_g(y)-g_1(y)|\leq 2^{\dimz}\dimz^{\alpha}(\dimz+\alpha)\eps_1,
   \end{align*}
   and since $\Phi_g$ is a linear combination of neural networks with coefficients $a_{m,\omega}\leq 2^{\alpha}\mlpBound^{\alpha}\coefficientBound$, Lemma~\ref{lemma:MLP linear combination} ensures that $\Phi_g$ can be implemented by a single neural network with
   \begin{align*}
       \mathscr{D}(\Phi_{g})&= r\bigl(\log(\dimz+\alpha)+\log(\eps_1^{-1})\bigr),\\
        \mathscr{W}(\Phi_{g})&=r(N+1)^{\dimz},\\
        \mathscr{S}(\Phi_{g})&= 3N2^{\alpha}\mlpBound^{\alpha}C.
   \end{align*}
   Taking $\eps_1=\frac{\eps}{2^{\dimz+1}\dimz^{\alpha}(\dimz+\alpha)},$ and using the triangle inequality guarantees that 
   \begin{align*}
       \sup_{y\in [0,1]^{\dimz}}\bigl|\Phi_g(y)-g(y)\bigr|\leq \eps,
   \end{align*}
   where plugging in the choice of $\eps_1$ and $N$ yields
   \begin{align*}
       \mathscr{D}(\Phi_{g})& = r \bigl(\log(\dimz+\alpha)+\alpha\dimz\log(\eps^{-1})\bigr),\\
        \mathscr{W}(\Phi_{g})&=r\mlpBound^{\dimz}\eps^{-\frac{\dimz}{\alpha}}\coefficientBound^{\dimz},\\
        \mathscr{S}(\Phi_{g})&= r\eps^{-\frac{1}{\alpha}}\mlpBound^{\alpha+1}\coefficientBound^2.
   \end{align*}
Finally, by Lemma~\ref{lemma:MLP Composition} there is a network $\Phi_h\in \Nc_{\dimz,1}$ implementing the composition $\Phi_h(z)=\Phi_g(\frac{z+\mlpBound}{2\mlpBound})$ for all $z\in [-\mlpBound,\mlpBound]^{\dimz}$ that satisfies
   \begin{align*}
       \sup_{[-\mlpBound,\mlpBound]^{\dimz}}\bigl|\Phi_h(z)-h(z)\bigr|=\sup_{[-\mlpBound,\mlpBound]^{\dimz}}\left|\Phi_g\Bigl(\frac{z+\mlpBound}{2\mlpBound} \Bigr)-h(z)\right|=\sup_{y\in [0,1]^{\dimz}}\bigl|\Phi_g(y)-g(y)\bigr|\leq \eps
   \end{align*}
   with
   \begin{align*}
       \mathscr{D}(\Phi_{h})& = r\bigl(\log(\dimz+\alpha)+\alpha\dimz\log(\eps^{-1})\bigr),\\
        \mathscr{W}(\Phi_{h})&=r\mlpBound^{\dimz}\eps^{-\frac{\dimz}{\alpha}}\coefficientBound^{\dimz},\\
        \mathscr{S}(\Phi_{h})&= r\eps^{-\frac{1}{\alpha}}\mlpBound^{\alpha+1}\coefficientBound^2.
   \end{align*}
   This completes the proof.
  \end{proof}
    \end{lemma}

\subsection{FNO Background Lemmas}

We have the following FNO analogs of the preceding lemmas for constructing MLPs.
\begin{lemma}[FNO Composition]\label{lemma:FNO Composition}
    Let $d_1,d_2,d_3\in \N$, $\Psi_1\in \tilde{\Nc}_{d_1,d_2}$, and $\Psi_2\in \tilde{\Nc}_{d_2,d_3}.$ Then, there exists an FNO $\Psi_3\in \tilde{\Nc}_{d_1,d_3}$ 
    such that
    \begin{align*}
        \Psi_3(u_N)=(\Psi_2\circ \Psi_1)(u_N)=\Psi_2\bigl(\Psi_1(u_N)\bigr), \quad \forall\: u_N\in L^2_N(\T^d,\R^{d_1}),
    \end{align*}
    and moreover
       \begin{align*}
        \mathscr{D}(\Psi_3)&=\mathscr{D}(\Psi_1)+\mathscr{D}(\Psi_2),\\
        \mathscr{W}(\Psi_3) & =\max\bigl\{2d_2,\mathscr{W}(\Psi_1),\mathscr{W}(\Psi_2)\bigr\},\\
        \mathscr{S}(\Psi_3)&=\max\bigl\{\mathscr{S}(\Psi_1),\mathscr{S}(\Psi_2)\bigr\},\\
        \cutoff(\Psi_3)&=\max\bigl\{\cutoff(\Psi_1),\cutoff(\Psi_2)\bigr\}.
    \end{align*}
    \begin{proof}
        The proof of  \cite[Lemma II.3]{elbrachter2021deep} for MLPs carries over directly to FNOs.
    \end{proof}
\end{lemma}

\begin{lemma}[FNO Depth Augmentation]
Let $d_1,d_2,K\in \N$ and $\Psi\in \tilde{\Nc}_{d_1,d_2}$ with $\mathscr{D}(\Psi)\leq K$. Then, there exists an FNO $\Psi'\in \tilde{\Nc}_{d_1,d_2}$ such that
\begin{align*}
    \Psi'(u_N)=\Psi(u_N), \quad \forall u_N\in L^2_N(\T^d,\R^{d_1}),
\end{align*}
and moreover
\begin{align*}
    \mathscr{D}(\Psi')&=K,\\
    \mathscr{W}(\Psi')&=\max\bigl\{2d_2,\mathscr{W}(\Psi)\bigr\},\\
    \mathscr{S}(\Psi')&=\max \bigl\{1,\mathscr{S}(\Psi)\bigr\},\\
    \cutoff(\Psi') & =\cutoff(\Psi).
\end{align*}

\begin{proof}
    The same construction as in \cite[Lemma II.4]{elbrachter2021deep} applies by taking $\hat{P}=0$ in the added FNO layers.
\end{proof}

\end{lemma}
\begin{lemma}[FNO Parallelization]\label{lemma:FNO parallelization lemma}
    Let $n,L\in \N$ and for $i\in \{1,\hdots,n\}$ let $d_i,d_i'\in \N$ and $\Psi_i\in \tilde{\Nc}_{d_i,d_i'}$ with $\mathscr{D}(\Psi_i)=L$. Then, there exists an FNO $\Psi \in \tilde{\Nc}_{\sum_{i=1}^n d_i,\sum_{i=1}^n d_i'}$ such that, for all $u_N=(u_N^1,\ldots,u_N^n)\in L^2_N(\T^d,\R^{\sum_{i=1}^n d_i})$ with $u_N^i\in L^2_N(\T^d,\R^{d_i}),$ 
        \begin{align*}
        \Psi(u_N)=\bigl(\Psi_1(u_N^1),\hdots,\Psi_n(u_N^n)\bigr)\in L^2_N(\T^d,\R^{\sum_{i=1}^n d_i'}),
    \end{align*} 
    and moreover
    \begin{align*}
        \mathscr{D}(\Psi)&=L,\\
        \mathscr{W}(\Psi)&=\sum_{i=1}^n\mathscr{W}(\Psi_i),\\
        \mathscr{S}(\Psi)&=\max_i\mathscr{S}(\Psi_i),\\
        \cutoff(\Psi) & = \max_i\cutoff(\Psi_i).
    \end{align*}
    \begin{proof}
        The same construction as in \cite[Lemma II.5]{elbrachter2021deep} goes through as written.
    \end{proof}
\end{lemma}

\begin{lemma}[FNO Linear Combination]\label{lemma:FNO linear combination}
    Let $n,L,\doutput\in \N$ and, for $i\in \{1,\hdots,n\},$ let $d_i\in \N$, $a_i\in \R$, and $\Psi_i\in \tilde{\mathcal{N}}_{d_i,\doutput}$ with $\mathscr{D}(\Psi_i)=L.$ Then, there exists an FNO $\Psi\in \tilde{\mathcal{N}}_{\sum_{i=1}^nd_i,\doutput}$ such that, 
    for all $u_N=(u_N^1,\ldots,u_N^n)\in L^2_N(\T^d,\R^{\sum_{i=1}^n d_i})$ with $u_N^i\in L^2_N(\T^d,\R^{d_i}),$ 
    \begin{align*}
         \Psi (u_N)=\sum_{i=1}^na_i  \Psi_i (u_N^i)\in L^2_N(\T^d,\R^{\doutput}),
    \end{align*}
      and moreover 
    \begin{align*}
        \mathscr{D}(\Psi)&=L,\\
        \mathscr{W}(\Psi)&\leq \sum_{i=1}^n\mathscr{W}(\Psi_i),\\
        \mathscr{S}(\Psi)&=\max_{i}\{|a_i|\mathscr{S}(\Psi_i)\},\\
        \mathscr{C}(\Psi)&=\max_i\mathscr{C}(\Psi_i).
    \end{align*}
    \begin{proof}
        The same construction as in \cite[Lemma II.6]{elbrachter2021deep} applies, using Lemma~\ref{lemma:FNO parallelization lemma} in place of \cite[Lemma II.5]{elbrachter2021deep}.
    \end{proof}
\end{lemma}

The following result \cite[Lemma 44]{kovachki2021universal} shows that pointwise MLPs can be realized as FNOs without increasing their depth, width, or weight magnitude.

\begin{lemma}[MLP-to-FNO Representation]\label{lemma: MLP to FNO lemma}
    Let $\Phi:\R^{\doutput}\to \R^{\doutput}$ be an MLP. Then, the mapping $a(x)\mapsto \mathcal{I}_N\Phi(a(x))$ can be represented by an FNO $\Psi$ with $\mathscr{D}(\Psi)=\mathscr{D}(\Phi)$, $\mathscr{W}(\Psi)=\mathscr{W}(\Phi)$, and $\mathscr{S}(\Psi)=\mathscr{S}(\Phi)$ by taking $P_k=0$ for all $k$.
\end{lemma}

\begin{lemma}[FNO Approximation of Affine Maps]\label{lemma:affine FNO layer}
    Let $\L:L^2_N(\T^d,\R^{\doutput})\to L^2_N(\T^d,\R^{\doutput})$ be an affine mapping of the form
    \begin{align}\label{eq:affine FNO eqn}
        \L(v)(x)= \mathcal{W} v(x)+ \mathcal{K} v(x)+b(x),
    \end{align}
    where $\mathcal{W}:L^2_N(\T^d,\R^{\doutput})\to L^2_N(\T^d,\R^{\doutput})$ is given by $\mathcal{W}(v)(x)={\bf W}v(x)$, $\mathcal{K}:L^2_N(\T^d,\R^{\doutput})\to L^2_N(\T^d,\R^{\doutput})$ is given by $\mathcal{K}(v)(x)=\Bigl(\F_N^{-1}\bigl(\hat{P}(\F v)\bigl)\Bigl)(x)$, and $b(x)\in L^2_N(\T^d,\R^{\doutput}).$
    Then, there exists an FNO $\Psi:L^2_N(\T^d,\R^{\doutput})\to L^2_N(\T^d,\R^{\doutput})$ such that
    \begin{align*}
        \Psi(u_N)(x)=\L(u_N)(x), \qquad \forall u_N\in L^2_N(\T^d,\R^{\doutput}),
    \end{align*}
    and moreover
    \begin{align*}
        \mathscr{D}(\Psi)&=1,\\
        \mathscr{W}(\Psi)&=2\doutput,\\
        \mathscr{S}(\Psi)&=\max\bigl\{\|{\bf W}\|_{\infty},\|\hat{P}\|_{\infty},\|\hat{b}\|_{\infty}\bigr\},\\
        \mathscr{C}(\Psi)&=N.
    \end{align*}
    \begin{proof}
         Consider a depth one FNO $\Psi:L^2_N(\T^d,\R^{\doutput})\to L^2_N(\T^d,\R^{\doutput})$ with lifting layer $\mathcal{R}:L^2_N(\T^d,\R^{\doutput})\to L^2_N(\T^d,\R^{2\doutput})$ parameterized by $${\bf R}=\begin{bmatrix}
            I_{\doutput}\\
            -I_{\doutput}
        \end{bmatrix}\in \R^{2\doutput\times \doutput},$$ a point-wise mapping $\mathcal{W}_1:L^2_N(\T^d,\R^{2\doutput})\to L^2_N(\T^d,\R^{2\doutput})$ parameterized by $${\bf W}_1=\begin{bmatrix}
            {\bf W} & 0\\ 0 & {\bf W}
        \end{bmatrix}\in \R^{2\doutput\times 2\doutput},$$ a nonlocal mapping $\mathcal{K}_1:L^2_N(\T^d,\R^{2\doutput})\to L^2_N(\T^d,\R^{2\doutput})$ parameterized by $$\hat{P}_{1}(k)=\begin{bmatrix}
            \hat{P}(k) & 0\\
            0 & \hat{P}(k)
        \end{bmatrix},$$ for all $k\in \findices_N$, bias function $b_1 \in L^2_N(\T^d,\R^{\doutput})$ given by $b_1(x)=\begin{bmatrix}
            b(x)\\ -b(x)
        \end{bmatrix}$, and projection layer $\mathcal{Q}:L^2_N(\T^d,\R^{2\doutput})\to L^2_N(\T^d,\R^{\doutput})$ parameterized by 
        $$
        {\bf Q}=\begin{bmatrix}
            I_{\doutput} & -I_{\doutput}
        \end{bmatrix}.
        $$
        Then, for any $u_N\in L^2_N(\T^d,\R^{\doutput})$, using the identity $x=\sigma(x)-\sigma(-x)$ for the ReLU activation function,
        \begin{align*}
            \Psi(u_N)(x)=\mathcal{I}_N\sigma\Bigl(\mathcal{W}u_N+ \mathcal{K} u_N +b\Bigr)(x)-\mathcal{I}_N\sigma\Bigl(-\mathcal{W}u_N - \mathcal{K}u_N-b\Bigr)(x)=\L(u_N)(x),
        \end{align*}
        as claimed. 
    \end{proof}
\end{lemma}

\section{FNO Approximation of the Building-Block Operators}
This appendix contains FNO approximation bounds for the operators involved in our classes of single-step spectral methods. 
\begin{lemma}\label{lemma:B FNO implementation}
    Let $\B\in \linearOpClass(\dissipation,\gamma,\power,\inverseBound)$ be as defined in Subsection~\ref{ssec:spectralclasspoly} and let $N'\geq N$. Then, there exists an FNO $\Psi_{\linearOp}:L^2_{N'}(\T^d,\R^{\doutput})\to L^2_{N}(\T^d,\R^{\doutput})\subset L^2_{N'}(\T^d,\R^{\doutput})$ such that
    \begin{align*}
        \Psi_{\linearOp}(u)=\B u, \qquad \forall u\in L^2_{N'}(\T^d,\R^{\doutput}),
    \end{align*}
and moreover
\begin{align*}
        \mathscr{D}(\Psi_{\linearOp})&=1,\\
        \mathscr{W}(\Psi_{\linearOp})&=2\doutput,\\
        \mathscr{S}(\Psi_{\linearOp})&=1,\\
        \mathscr{C}(\Psi_{\linearOp})&=N'.
    \end{align*}
    \begin{proof}
        Since $\mathcal{P}_N:L^2_{N'}(\T^d,\R^{\doutput})\to L^2_{N}(\T^d,\R^{\doutput})$ is an $L^2$-orthogonal projection onto a shift-invariant subspace $\mathcal{V}\subset L^2_{N}(\T^d,\R^{\doutput})$, it is diagonalizable in Fourier space, meaning 
        \begin{align*}
            (\mathcal{P}_Nu)(x)=\sum_{k\in \findices_N}\hat{P}_{\mathcal{P}_N}(k)\hat{u}(k)e^{i\langle k,x\rangle},
        \end{align*}
       for some orthogonal projectors $\hat{P}_{\mathcal{P}_N}(k)\in \C^{\doutput\times \doutput}$. Similarly, by our assumption that $\mathcal{A}$ is diagonalizable in Fourier space, we have
       \begin{align*}
            (\mathcal{A}u)(x)=\sum_{k\in \findices_N}\hat{P}_{\mathcal{A}}(k)\hat{u}(k)e^{i\langle k,x\rangle},
        \end{align*}
        for positive  semi-definite \nc 
        matrices $\hat{P}_{\mathcal{A}}(k)\in \C^{\doutput\times \doutput}$. Consequently, we can write
        \begin{align*}
            (\B u)(x)= \sum_{k\in \findices_N}\bigl(\frac{1}{\timestep}I+\gamma \hat{P}_{\mathcal{A}}(k)^{\power}+\hat{P}_{\mathcal{A}}(k)\bigr)^{-1}\hat{P}_{\mathcal{P}_N}(k)\hat{u}(k)e^{i\langle k,x\rangle},
        \end{align*}
        which is precisely of the form in \eqref{eq:affine FNO eqn}. By the requirement that $\timestep\leq 1$ and the fact that each $\hat{P}_{\mathcal{A}}(k)$ is 
         positive semi-definite 
        and each $\hat{P}_{\mathcal{P}_N}(k)$ is an orthogonal projector, it holds that $\|\bigl(\frac{1}{\timestep}I+\gamma \hat{P}_{\mathcal{A}}(k)^{\power}+\hat{P}_{\mathcal{A}}(k)\bigr)^{-1}\hat{P}_{\mathcal{P}_N}(k)\|_{\infty}\leq 1$ for each $k\in \findices_N$. Consequently, Lemma~\ref{lemma:affine FNO layer} guarantees that an FNO of the size given can implement $\B$, as claimed.
    \end{proof}
\end{lemma}
\begin{lemma}\label{lemma:affine part FNO implementation}
    Let $f_N\in \forcingClass(\forcingBound)$ and $\mathcal{A}$ be as defined in Subsection~\ref{ssec:spectralclasspoly}. Then, there exists an FNO $\Psi_{\textup{affine}}:L^2_{N}(\T^d,\R^{\doutput})\to L^2_{N}(\T^d,\R^{\doutput})$ such that
    \begin{align*}
        \Psi_{\textup{affine}}(u)=\bigl(\frac{1}{\timestep}I+\gamma \mathcal{A}^{\power}\bigr)u+f_N, \qquad \forall u\in L^2_{N}(\T^d,\R^{\doutput}), 
    \end{align*}
   and moreover
    \begin{align*}
                \mathscr{D}(\Psi_{\textup{affine}})&=1,\\
        \mathscr{W}(\Psi_{\textup{affine}})&=2\doutput,\\
        \mathscr{S}(\Psi_{\textup{affine}})&=\max \bigl\{\timestep^{-1},r\gamma N^{\dissipation \power},\forcingBound \bigr\},\\
        \mathscr{C}(\Psi_{\textup{affine}})&=N.
    \end{align*}
    \begin{proof}
       Since $\mathcal{A}$ is diagonalizable in Fourier space, we have that
        \begin{align*}
            \Bigl(\bigl(\frac{1}{\timestep}I+ \gamma \mathcal{A}^{\power}\bigr)u\Bigr)(x)+f_N(x)=\sum_{k\in \findices_N}\bigl(\frac{1}{\timestep}I+ \gamma \hat{P}_{\mathcal{A}}(k)^{\power}\bigr)\hat{u}(k)e^{i\langle k,x\rangle}+f_N(x),
        \end{align*}
        which is exactly of the form in \eqref{eq:affine FNO eqn}. 
         By the assumption that $\inverseBound\|k\|_2^{\dissipation}I\preceq \hat{P}_{\mathcal{A}}(k)\preceq \|k\|_2^{\dissipation}I$ for all $k\in \findices_N$, we have that
$\|\hat{P}_{\mathcal{A}}(k)^{\power}\|_{\infty}\leq rN^{\dissipation \power}.$
The result then follows from the affine FNO Lemma~\ref{lemma:affine FNO layer}. 
    \end{proof}
\end{lemma}

\begin{lemma}\label{lemma:Ds FNO implementation}
For any $\mathcal{D}_s\in \differentialClass(s,\DcoefficientBound),$ there exists an FNO $\Psi_s:L^2_N(\T^d,\R^{\ndim})\to L^2_N(\T^d,\R^{\doutput})$ such that  
\begin{align*}
    \Psi_s(v_N)=\mathcal{D}_s(v_N), \qquad \forall v_N\in L^2_N(\T^d, \R^{\ndim}),
\end{align*}
and moreover 
    \begin{align*}
        \mathscr{D}(\Psi_s)&= 1,\\
        \mathscr{W}(\Psi_s)&=2(\ndim+\doutput),\\
        \mathscr{S}(\Psi_s)&=d\DcoefficientBound N^s ,\\
        \mathscr{C}(\Psi_s)&=N.
    \end{align*}
    \begin{proof}
        Fix some $\mathcal{D}_s:L^2_{N}(\T^d,\R^{\ndim})\to L^2_{N}(\T^d,\R^{\doutput})$. For any $v\in L^2_{N}(\T^d,\R^{\ndim}),$
    \begin{align*}
        \mathcal{D}_s(v)= \sum_{k\in \findices_{N}}\sum_{\ell=1}^{\ndim}\sum_{j=1}^{d}\begin{bmatrix}
           c_{j\ell 1}(ik_j)^{s_{j\ell 1}}\hat{v}_{\ell}(k)e^{i\langle k,x\rangle}\\
            \vdots \\
           c_{j\ell \doutput}(ik_j)^{s_{j\ell \doutput}}\hat{v}_{\ell}(k)e^{i\langle k,x\rangle}
        \end{bmatrix}.
    \end{align*}
Here we adopt the convention $0^0=1,$ so that $\partial^0=I$. This mapping can be implemented exactly by an FNO $\Psi_s:L^2_{N}(\T^d,\R^{\ndim})\to L^2_{N}(\T^d,\R^{\doutput})$, which we now construct. Consider a lifting layer $\mathcal{R}$ parameterized by 
$${\bf R}_s=
\begin{bmatrix}
I_{\ndim}\\
0_{\doutput\times \ndim}\\
I_{\ndim}\\
0_{\doutput\times \ndim}
\end{bmatrix}
\in \R^{2(\ndim+\doutput)\times \ndim}. $$ Let
\begin{align*}
    \hat{P}_1'(k)=\begin{bmatrix}
                \sum_{j=1}^dc_{j11}(ik_j)^{s_{j11}} & \hdots & \sum_{j=1}^dc_{j\ndim 1}(ik_j)^{s_{j\ndim 1}} & 0_{1\times \doutput}\\
                 \vdots & \ddots &\vdots  & \vdots \\
                 \sum_{j=1}^d c_{j 1 \doutput}(ik_j)^{s_{j1\doutput}}& \hdots & \sum_{j=1}^d c_{j\ndim \doutput}(ik_j)^{s_{j\ndim\doutput}} & 0_{1\times \doutput}\\
                 0_{\ndim \times 1} & \hdots & 0_{\ndim \times 1} & 0
            \end{bmatrix}\in \C^{\ndim+\doutput\times \ndim+\doutput} \quad \forall k\in \findices_{N},
\end{align*}
and consider a hidden FNO layer $\mathcal{L}_1:L^2_{N}(\T^d,\R^{2(\ndim+\doutput)})\to L^2_{N}(\T^d,\R^{2(\ndim+\doutput)})$ with block-diagonal Fourier multipliers 
\begin{align*}
            \hat{P}_1(k)=\begin{bmatrix}
                \hat{P}_1'(k) \\
                & -\hat{P}_1'(k)
            \end{bmatrix} \in \C^{2(\ndim+\doutput)\times 2(\ndim+\doutput)},
        \end{align*}
    ${\bf W_1}=0$, and $b(x)=0.$ Taking a projection layer $\mathcal{Q}$ parameterized by 
    \begin{align*}
        {\bf Q}_s=\begin{bmatrix}
            I_{\doutput}& 0_{\doutput\times \ndim} & -I_{\doutput}& 0_{\doutput\times \ndim}\\
        \end{bmatrix}\in \R^{\doutput\times 2(\ndim+\doutput)},
    \end{align*}
    we have that
    \begin{align*}
        \Psi_s(v)(x)&=\mathcal{Q} \circ \mathcal{L}_1 \circ \mathcal{R}(v)(x)\\
        & \hspace{-0.5cm}=\sigma\left( \sum_{k\in \findices_{N}}\sum_{\ell=1}^{\ndim}\sum_{j=1}^{d}\begin{bmatrix}
           c_{j\ell 1}(ik_j)^{s_{j\ell 1}}\hat{v}_{\ell}(k)e^{i\langle k,x\rangle}\\
            \vdots \\
           c_{j\ell \doutput}(ik_j)^{s_{j\ell \doutput}}\hat{v}_{\ell}(k)e^{i\langle k,x\rangle}
        \end{bmatrix}\right)-\sigma\left( - \sum_{k\in \findices_{N}}\sum_{\ell=1}^{\ndim}\sum_{j=1}^{d}\begin{bmatrix}
           c_{j\ell 1}(ik_j)^{s_{j\ell 1}}\hat{v}_{\ell}(k)e^{i\langle k,x\rangle}\\
            \vdots \\
            c_{j\ell \doutput}(ik_j)^{s_{j\ell \doutput}}\hat{v}_{\ell}(k)e^{i\langle k,x\rangle}
        \end{bmatrix}\right)\\
        & \hspace{-0.5cm} = \sum_{k\in \findices_{N}}\sum_{\ell=1}^{\ndim}\sum_{j=1}^{d}\begin{bmatrix}
            c_{j\ell 1}(ik_j)^{s_{j\ell 1}}\hat{v}_{\ell}(k)e^{i\langle k,x\rangle}\\
            \vdots \\
           c_{j\ell \doutput}(ik_j)^{s_{j\ell \doutput}}\hat{v}_{\ell}(k)e^{i\langle k,x\rangle}
        \end{bmatrix}
        =\mathcal{D}_s(v)(x),
    \end{align*}
    where we have used the identity $x=\sigma(x)-\sigma(-x).$ Note that $\Psi_s$ satisfies
    \begin{align*}
        \mathscr{D}(\Psi_s)&= 1,\\
        \mathscr{W}(\Psi_s)&=2(\ndim+\doutput),\\
        \mathscr{S}(\Psi_s)&=d\DcoefficientBound N^s ,\\
        \mathscr{C}(\Psi_s)&=N,
    \end{align*}
    completing the proof.
    \end{proof}
\end{lemma}

 In the following results, $\mathcal{A}$ and $\mathcal{P}_N$ are as in Subsection \ref{ssec:spectralclasspoly}, and in particular, $\mathcal{A}$ satisfies that for all $k\in \findices_N\backslash\{0\},$
\begin{align*}
    0 \prec \|k\|_{2}^{\beta}cI \preceq \hat{P}_{\mathcal{A}}(k)\preceq \|k\|_{2}^{\beta}I,
\end{align*}
for some $c\in (0,1].$

    In the following lemma, we construct an FNO to approximate the operator $\mathcal{D}_s\bigl(\nonlinearity(\cdot)\bigl)$. We quantify the mean-zero part of the approximation error in the $L^2$ norm weighted by $\mathcal{A}^{-1/2}$, which is most convenient for the multi-step approximation bound in Subsection \ref{ssec:polynomial approximation theorem proof} and allows for faster approximation by exploiting the dissipation due to the operator $\mathcal{A}$. We also quantify the size of the mean in the $L^2$ norm, identifying the mean with the constant function taking that value on $\T^d$.
    
\begin{lemma}\label{lemma:polynomial part single step}
Let $\nonlinearity\in \nonlinearClass(p,\coefficientBound)$ and $\mathcal{D}_s
\in \differentialClass(s,\DcoefficientBound)$ be as defined in Subsection~\ref{ssec:spectralclasspoly}, and let $\eps_{\nonlinearity}>0$. Then, there exists an FNO $\Psi_{\nonlinearity}:L^2_N(\T^d,\R^{\doutput})\to L^2_N(\T^d,\R^{\doutput})$ such that, for any $u_N\in L^2_N(\T^d,\
\R^{\doutput})$ with $\|u_N\|_{L^2}\leq \inputBall,$
\begin{align*}
\left\|
\mathcal A^{-1/2}
\left(
\Psi_{\nonlinearity}(u_N)
-
\mathcal P_N\mathcal D_s(\nonlinearity(u_N))
-
\bar E(u_N)
\right)
\right\|_{L^2}
&\le
\eps_{\nonlinearity},\\
 \|\bar{E}(u_N)\|_{L^2} &\leq \eps_{\nonlinearity},
\end{align*}
\nc
where 
\begin{equation}\label{eq:mean-appendix}
   \bar E(u_N)
:=
\frac{1}{(2\pi)^d}\int_{\T^d}
\Bigl[
\Psi_{\nonlinearity}(u_N)(x)
-
\mathcal{P}_N\mathcal D_s(\nonlinearity(u_N))(x)
\Bigr]\,dx . 
\end{equation}
Moreover,  $\Psi_{\nonlinearity}$ can be chosen to satisfy 
\begin{align*}
        \mathscr{D}(\Psi_{\nonlinearity})&=r\log(\eps^{-1}_{\nonlinearity}N\inputBall),\\
        \mathscr{W}(\Psi_{\nonlinearity})&=r\ndim{p+\doutput-1 \choose \doutput-1},\\
        \mathscr{S}(\Psi_{\nonlinearity})&=rN^s,\\
        \mathscr{C}(\Psi_{\nonlinearity})&=pN,
    \end{align*}
    where $r\geq 1$ is a constant depending on  $d,s,\doutput,p,\inverseBound$ and $\coefficientBound.$    
    \begin{proof}    
    By Lemma~\ref{lemma: polynomial approximation}, for each component $g_{j}$ for $1\leq j\leq \ndim$ (a polynomial of degree at most $p$), there exists a ReLU network $\Phi_{j}:\R^{\doutput}\to \R$ such that
    \begin{align*}
        \sup_{z\in [-\mlpBound,\mlpBound]^{\doutput}}|\Phi_{j}(z)-g_{j}(z)|\leq \eps_1,
    \end{align*}
    with 
    \begin{align*}
        \mathscr{D}(\Phi_{j})&= rp^2\log(\mlpBound\eps_1^{-1}),\\
        \mathscr{W}(\Phi_{j})&=r{p+\doutput-1 \choose \doutput-1},\\
        \mathscr{S}(\Phi_{j})&=\coefficientBound,
    \end{align*}
    where $\eps_1$ and $\mlpBound$ will be chosen later. By Lemma~\ref{lemma:MLP parallelization}, there exists a ReLU network $\Phi:\R^{\doutput}\to \R^{\ndim}$ implementing the $\Phi_{j}$ in parallel such that
    \begin{align*}
        \sup_{z\in [-\mlpBound,\mlpBound]^{\doutput}} \left\|\Phi(z)-\begin{bmatrix}
            g_{1}(z)\\
            \vdots \\
            g_{\ndim}(z)
        \end{bmatrix} \right\|_2\leq \sqrt{\ndim}\eps_1,
    \end{align*}
    with 
    \begin{align*}
        \mathscr{D}(\Phi)&= rp^2\log(\mlpBound\eps_1^{-1}),\\
        \mathscr{W}(\Phi)&=r\ndim{p+\doutput-1 \choose \doutput-1},\\
        \mathscr{S}(\Phi)&=\coefficientBound,
    \end{align*}
    and in turn there exists an FNO $\Psi':L^2_{N}(\T^d,\R^{\doutput})\to L^2_{pN}(\T^d,\R^{\ndim})$ such that
    \begin{align*}
        \Psi'(u_N)=\mathcal{I}_{pN}\Phi(u_N),
    \end{align*}
    with \begin{align*}
        \mathscr{D}(\Psi')&= rp^2\log(\mlpBound\eps^{-1}_1),\\
        \mathscr{W}(\Psi')&=r\ndim{p+\doutput-1 \choose \doutput-1},\\
        \mathscr{S}(\Psi')&=\coefficientBound  ,\\
        \mathscr{C}(\Psi')&=pN.
    \end{align*}
    Note that we have taken the Fourier cutoff 
    to be $pN$ and not $N$. This can be interpreted as ``de-aliasing'' the FNO by computing on a finer grid of size $pN$. This allows us to avoid aliasing errors since for any $u_N\in L^2_N(\T^d,\R^{\doutput}),$ it holds that $\nonlinearity(u_N)\in L^2_{pN}(\T^d,\R^{\ndim}),$ and thus $\nonlinearity(u_N)=\mathcal{I}_{pN}\nonlinearity(u_N).$ In Lemma~\ref{lemma:Ds FNO implementation}, we have constructed an FNO $\Psi_s:L^2_{pN}(\T^d,\R^{\ndim})\to L^2_{pN}(\T^d,\R^{\doutput})$ such that 
    \begin{align*}
        \mathcal{D}_s(v)=\Psi_s(v), \qquad \forall v\in L^2_{pN}(\T^d,\R^{\ndim}), 
    \end{align*}
  and moreover
    \begin{align*}
        \mathscr{D}(\Psi_s)&= 1,\\
        \mathscr{W}(\Psi_s)&=2(\ndim+\doutput),\\
        \mathscr{S}(\Psi_s)&=d\DcoefficientBound p^sN^s ,\\
        \mathscr{C}(\Psi_s)&=pN.
    \end{align*}
         Finally, we remark that by the same argument as in Lemma~\ref{lemma:B FNO implementation}, the operator $\mathcal{P}_N:L^2_{pN}(\T^d,\R^{\doutput})\to L^2_N(\T^d,\R^{\doutput})$ can be implemented by a single affine FNO layer, and thus by the composition
          Lemma~\ref{lemma:FNO Composition}, 
         the mapping $u_N\mapsto  \mathcal{P}_N\mathcal{D}_s\bigl(\mathcal{I}_{pN}\Phi(u_N)\bigr)=\mathcal{P}_N \Psi_s\bigl(\Psi'(u_N)\bigr)$ can be implemented exactly by an FNO $\Psi_{\nonlinearity}:L^2_{N}(\T^d,\R^{\doutput})\to L^2_{N}(\T^d,\R^{\doutput})$ of size
    \begin{align*}
        \mathscr{D}(\Psi_{\nonlinearity})&= rp^2\log(\mlpBound\eps_1^{-1}),\\
        \mathscr{W}(\Psi_{\nonlinearity})&=r\ndim{p+\doutput-1 \choose \doutput-1},\\
        \mathscr{S}(\Psi_{\nonlinearity})&=\max\bigl\{\coefficientBound,\doutput\DcoefficientBound(pN)^s \bigr\}  ,\\
        \mathscr{C}(\Psi_{\nonlinearity})&=pN.
    \end{align*}
    We now choose $\eps_1$ and $\mlpBound$ to guarantee that $\Psi_{\nonlinearity}(\cdot)$ can approximate $\mathcal{D}_s\bigl(\nonlinearity(\cdot)\bigr)$ to accuracy $\eps_\nonlinearity.$ First, Nikolskii's inequality \cite[Equation 5.1]{guo1998spectral} implies that, for any $u_N\in L^2_{N}(\T^d,\R^{\doutput})$ with $\|u_N\|_{L^2}\leq \inputBall,$ there exists a constant $r>0$ independent of $N$ such that
    \begin{align*}
        \sup_{x\in \T^d}\|u_N(x)\|_{2}\leq r N^{\frac{d}{2}}\inputBall.
    \end{align*}
    Taking $\mlpBound=rN^{\frac{d}{2}} \inputBall$, we have that, for each $j \in \mathsf{J}_{pN}$,
\begin{align*}
    \|\Phi(u_N)(x_j)-\nonlinearity(u_N)(x_j)\|_2\leq \sqrt{\ndim}\eps_1,
\end{align*}
from which it follows that, for all $k\in \findices_{pN},$ 
                  \begin{align}\label{eq: fourier coefficient bound 1}
                  \begin{split}
                    & \Bigl\| \bigl(\mathcal{I}_{pN}\Phi(u_N)\bigr)(k)-\bigl(\F_{pN}\nonlinearity(u_N)\bigr)(k)\Bigr\|_2\\
                      &:=\Bigl \|\Bigl({\bf F} _{pN}\{\Phi(u_N)(x_j)\}_{j\in \indices_{pN}}\Bigr)(k)-\Bigl(\F_{pN}\nonlinearity(u_N)\Bigr)(k) \Bigr\|_2\\
                      & =\biggl\|\frac{1}{(2pN+1)^d}\sum_{j\in \indices_{pN}}\Bigl(\Phi(u_N(x_j))-\nonlinearity(u_N(x_j))\Bigr)e^{-i\langle x_j,k\rangle} \biggr\|_2\\
                      & \leq \frac{1}{(2pN+1)^d}\sum_{j\in \indices_{pN}} \bigl\|\Phi(u_N(x_j))-\nonlinearity(u_N(x_j)) \bigr\|_2\\
                      & \leq \sqrt{\ndim}\eps_1.
                      \end{split}
                  \end{align}
Recall the mean error $\bar{E}(u_N)$ introduced in \eqref{eq:mean-appendix}. By the assumption that $s\leq \dissipation/2$ and the scaling of $\mathcal{A}$ given in \eqref{eq:dissipation scaling}, which implies that $\|\hat{P}_{\mathcal{A}^{-1/2}}(k)\|_{op}\lesssim \inverseBound^{-1/2}\|k\|_2^{-\dissipation/2}$, 

\begin{small}
\begin{align*}
\Bigl\|\mathcal{A}^{-1/2}\Bigl(\Psi_{\nonlinearity}(u_N)-\mathcal{P}_N\mathcal{D}_s\nonlinearity(u_N)- \bar{E}(u_N)\Bigr)\Bigr\|_{L^2}&=\Bigl\| \mathcal{A}^{-1/2}\Bigl(\mathcal{P}_N\mathcal{D}_s \mathcal{I}_{pN}\Phi(u_N) -\mathcal{P}_N\mathcal{D}_s\nonlinearity(u_N)- \bar{E}(u_N)\Bigr)\Bigr\|_{L^2}\\ 
    &\hspace{-7.4cm} =\left\|\sum_{k\in\findices_N\backslash\{0\}}\hat{P}_{\mathcal{A}^{-1/2}}(k)\hat{P}_{\mathcal{P}_N}(k)\sum_{\ell=1}^{\ndim}\sum_{\ell'=1}^{d}\begin{bmatrix}
         c_{\ell' \ell 1}(ik_{\ell'})^{s_{\ell' \ell 1}}\Bigl(\Bigl({\bf F} _{pN}\{\Phi_{\ell}(u_N)(x_j)\}_{j\in \indices_{pN}}\Bigr)(k)-\Bigl(\F_{pN}g_{\ell}(u_N)\Bigr)(k)\Bigr)\\
         \vdots \\
         c_{\ell' \ell \doutput}( ik_{\ell'} )^{s_{\ell' \ell \doutput}}\Bigl(\Bigl({\bf F} _{pN}\{\Phi_{\ell}(u_N)(x_j)\}_{j\in \indices_{pN}}\Bigr)(k)-\Bigl(\F_{pN}g_{\ell}(u_N)\Bigr)(k)\Bigr)
     \end{bmatrix}e^{i\langle k,x\rangle }\right\|_{L^2}\\ 
     & \hspace{-7.4cm} \leq r \DcoefficientBound \doutput \sum_{k\in \findices_N}\inverseBound^{-1/2}\|k\|_2^{-\dissipation/2+s}\Bigl\| \bigl(\mathcal{I}_{pN}\Phi(u_N)\bigr)(k)-\bigl(\F_{pN}\nonlinearity(u_N)\bigr)(k)\Bigr\|_2\\
     & \hspace{-7.4cm} \leq r \DcoefficientBound \inverseBound^{-1/2} N^{\max\{d+s-\dissipation/2,0\}}\log(N)\eps_1.
\end{align*}
\end{small}
Then, taking $\eps_1=\frac{\eps_{\nonlinearity}\inverseBound^{1/2}}{2r\DcoefficientBound N^{\max\{d+s-\dissipation/2,0\}}\log(N)}$ yields the desired result. To bound $\|\bar{E}(u_N)\|_{L^2}$, note that
\begin{align*}
    \|\bar{E}(u_N)\|_{L^2}&=  (2\pi)^{d/2}  \left\|\F\Bigl( \mathcal{P}_N\mathcal{D}_s\mathcal{I}_{pN}\Phi(u_N)-\mathcal{P}_N\mathcal{D}_s \nonlinearity(u_N)\Bigr)(0) \right\|_{2}\\
    &\leq  (2\pi)^{d/2}  \left\|\sum_{\ell=1}^{\ndim}\sum_{\ell'=1}^{d}\begin{bmatrix}
         c_{\ell' \ell 1}\Bigl(\Bigl({\bf F} _{pN}\{\Phi_{\ell}(u_N)(x_j)\}_{j\in \indices_{pN}}\Bigr)(0)-\Bigl(\F_{pN}g_{\ell}(u_N)\Bigr)(0)\Bigr)\\
         \vdots \\
         c_{\ell' \ell \doutput}\Bigl(\Bigl({\bf F} _{pN}\{\Phi_{\ell}(u_N)(x_j)\}_{j\in \indices_{pN}}\Bigr)(0)-\Bigl(\F_{pN}g_{\ell}(u_N)\Bigr)(0)\Bigr)
     \end{bmatrix}\right\|_{L^2}\\
&\leq r\DcoefficientBound \doutput \eps_1 \leq \eps_{\nonlinearity},
\end{align*}
where the last line follows from the bound in \eqref{eq: fourier coefficient bound 1} for $k=0$ along with the given choice of $\eps_1$.
\end{proof}
\end{lemma}

\begin{lemma}\label{lemma: polynomial nonlinearity lipschitz constant}
    Let $\nonlinearity\in \nonlinearClass(p,\coefficientBound)$ be as defined in Subsection \ref{ssec:spectralclasspoly}. Then, for any $u_N,v_N\in L^2_N(\T^d,\R^{\doutput})$ with $\|u_N\|_{L^2},\|v_N\|_{L^2}\leq \inputBall$, it holds that
    \begin{align*}
        \|\nonlinearity(u_N)-\nonlinearity(v_N)\|_{L^2}\leq rN^{\frac{d}{2}(p-1)}\inputBall^{p-1}\|u_N-v_N\|_{L^2}
    \end{align*}
    for a constant $r>0$ depending on $d,p,\coefficientBound,\doutput,\ndim, \DcoefficientBound$ but independent of $N$ and $\inputBall.$ Further, defining $$\bar{E}'(u_N,v_N) :=\frac{1}{(2\pi)^d}\int_{\T^d} \Bigl[\mathcal{D}_s\bigl(\nonlinearity(u_N)\bigr)(x)-\mathcal{D}_s\bigl(\nonlinearity(v_N)\bigr)(x) \Bigr] \, dx,$$ we have that
    \begin{align*}
        \|\bar{E}'(u_N,v_N)\|_{L^2}\leq r N^{\frac{d}{2}(p-1)}\inputBall^{p-1}\|u_N-v_N\|_{L^2}.
    \end{align*}
    \begin{proof}
        By the mean value theorem,
        \begin{align}\label{eq: MVT}
            \bigl\|\nonlinearity(u_N)(x)-\nonlinearity(v_N)(x) \bigr\|_2\leq \left\|\begin{bmatrix}
                \max_{\|z\|_2\leq Z}\|\nabla g_{1}(z)\|_2\|u_N(x)-v_N(x)\|_2\\
                \vdots \\
                \max_{\|z\|_2\leq Z}\|\nabla g_{\ndim}(z)\|_2\|u_N(x)-v_N(x)\|_2
            \end{bmatrix}\right\|_2,
        \end{align}
        where $Z=\max \bigl\{\|u_N\|_{L^{\infty}},\|v_N\|_{L^\infty} \bigr\}$. By Nikolskii's inequality, we have that $Z\leq rN^{\frac{d}{2}}\inputBall.$ Since each $g_i$ is a degree-$p$ polynomial, it holds that 
        \begin{align*}
            \max_{\|z\|_2\leq Z}\|\nabla g_i(z)\|_2\leq rN^{\frac{d}{2}(p-1)}\inputBall^{p-1},
        \end{align*}
        for $1\leq i \leq \ndim$.  Consequently, integrating the square of the inequality given in \eqref{eq: MVT} over $\T^d$ yields the first result. For the second, we have that
 \begin{align*}
            \|\bar{E}'(u_N,v_N)\|_{L^2}^2&= (2\pi)^{d}\left\|\F\Bigl( \mathcal{P}_N\mathcal{D}_s\nonlinearity(u_N)-\mathcal{P}_N\mathcal{D}_s \nonlinearity(v_N)\Bigr)(0) \right\|_{2}^2\\
    &\leq   \frac{(2\pi)^{d}(d\ndim \DcoefficientBound)^2}{(2pN+1)^d}\sum_{j\in \indices_{pN}} \bigl\|\nonlinearity(u_N(x_j))-\nonlinearity(v_N(x_j)) \bigr\|_2^2\\
     & \leq  \frac{(2\pi)^d(d\ndim \DcoefficientBound)^2}{(2pN+1)^d}\sum_{j\in \indices_{pN}} \left\|\begin{bmatrix}
                \max_{\|z\|_2\leq Z}\|\nabla g_{1}(z)\|_2 \|u_N(x_j)-v_N(x_j)\|_2\\
                \vdots \\
                \max_{\|z\|_2\leq Z}\|\nabla g_{\ndim}(z)\|_2 \|u_N(x_j)-v_N(x_j)\|_2
            \end{bmatrix}\right\|_2^2\\
     & \leq (2\pi)^d(d\ndim \DcoefficientBound)^2 r N^{d(p-1)}\inputBall^{2p-2} \frac{1}{(2pN+1)^d}\sum_{j\in \indices_{pN}} \left\|\begin{bmatrix}
                \|u_N(x_j)-v_N(x_j)\|_2\\
                \vdots \\
                \|u_N(x_j)-v_N(x_j)\|_2
            \end{bmatrix}\right\|_2^2\\
            & \leq (2\pi)^d(d\ndim \DcoefficientBound)^2 r N^{d(p-1)}\inputBall^{2p-2}\|u_N-v_N\|_{L^2}^2.
        \end{align*}

\nc

    \end{proof}
\end{lemma}

\begin{lemma}\label{lemma:smooth part single step}
Let $\nonlinearity\in \nonlinearClass(\alpha,\coefficientBound)$ be as defined in Subsection~\ref{ssec:spectralclasssmooth} and let $\eps_\nonlinearity>0$. Then, there exists an FNO $\Psi_{\nonlinearity}:L^2_N(\T^d,\R^{\doutput})\to \textup{Range}(\mathcal{P}_N)\subset L^2_N(\T^d,\R^{\doutput})$ such that, for any $u_N\in L^2_N(\T^d,\
\R^{\doutput})$ with $\|u_N\|_{L^2}\leq \inputBall,$
\begin{align*}
    \left\|\mathcal{A}^{-1/2}\bigl(\Psi_{\nonlinearity}(u_N)-\mathcal{P}_N\mathcal{D}_{s}\nonlinearity(u_N)\bigr) \right\|_{L^2} &\leq \eps_{\nonlinearity}, \\
     \frac{1}{(2\pi)^d}\int_{\T^d}\Psi_{\nonlinearity}(u_N)(x) \, dx & =0, 
\end{align*}
and moreover
\begin{align*}
               \mathscr{D}(\Psi_{\nonlinearity})& =r_{\alpha,\doutput,s,\dissipation}\log(\eps_{\nonlinearity}^{-1}N),\\
        \mathscr{W}(\Psi_{\nonlinearity})&=r_{\alpha,\doutput,\coefficientBound}\inputBall^{\doutput}N^{\frac{\doutput d}{2}}N^{\frac{\doutput\max\{d+s-\dissipation/2\}}{\alpha}}{\log(N)^{\frac{\doutput}{\alpha}}}\eps_{\nonlinearity}^{-\frac{\doutput}{\alpha}},\\
        \mathscr{S}(\Psi_{\nonlinearity})&=\max \Bigl\{r_{\alpha,\doutput,\coefficientBound}\inputBall^{\alpha-1}N^{\frac{\max\{d+s-\dissipation/2,0\}}{\alpha}}{\log(N)^{\frac{1}{\alpha}}}\eps_{\nonlinearity}^{-\frac{1}{\alpha}},\eps_{\nonlinearity}^{-\frac{s}{\alpha}}N^{\frac{ds}{2}+s}\inputBall^{s} \Bigr\},\\
        \mathscr{C}(\Psi_{\nonlinearity})&=\max \Bigl\{N,r_{d,\doutput,\alpha,\coefficientBound,\inverseBound}\eps_{\nonlinearity}^{-1/\alpha} \inputBall N^{1+\frac{d}{2}} \Bigr\}.
    \end{align*}
    \begin{proof}
        Let $\eps_1>0.$ By Lemma~\ref{lemma: sobolev approximation}, for each component $g_{j}$ for $1\leq j\leq \doutput$ (a smooth function with $\alpha$ derivatives bounded by $\coefficientBound$), there exists a ReLU network $\Phi_j:\R^{\doutput}\to \R$ such that
    \begin{align*}
        \sup_{z\in [-\mlpBound,\mlpBound]^{\doutput}}|\Phi_{j}(z)-g_{j}(z)|\leq \eps_1,
    \end{align*}
    with 
    \begin{align*}
        \mathscr{D}(\Phi_{j})&= r_{\alpha,\doutput}\log(\eps_1^{-1}),\\
        \mathscr{W}(\Phi_{j})&=r_{\alpha,\doutput}\mlpBound^{\doutput}\coefficientBound^{\doutput}\eps_1^{-\frac{\doutput}{\alpha}},\\
        \mathscr{S}(\Phi_{j})&=r_{\alpha,\doutput}\mlpBound^{\alpha+1}\coefficientBound^2\eps_1^{-\frac{1}{\alpha}},
    \end{align*}
    where $\eps_1$ and $\mlpBound$ will be chosen later. By Lemma~\ref{lemma:MLP parallelization}, there exists a ReLU network $\Phi:\R^{\doutput}\to \R^{\ndim}$ implementing the $\Phi_j$ in parallel such that
    \begin{align*}
        \sup_{z\in [-\mlpBound,\mlpBound]^{\doutput}} \left\|\Phi(z)-\begin{bmatrix}
            g_{1}(z)\\
            \vdots \\
            g_{\ndim}(z)
        \end{bmatrix} \right\|_2\leq \sqrt{\ndim}\eps_1,
    \end{align*}
    with 
    \begin{align*}
        \mathscr{D}(\Phi)&= r_{\alpha,\doutput}\log(\eps_1^{-1}),\\
        \mathscr{W}(\Phi)&=r_{\alpha,\doutput}\ndim\mlpBound^{\doutput}\coefficientBound^{\doutput}\eps_1^{-\frac{\doutput}{\alpha}},\\
        \mathscr{S}(\Phi)&=r_{\alpha,\doutput}\mlpBound^{\alpha+1}\coefficientBound^2\eps_1^{-\frac{1}{\alpha}}.
    \end{align*}
    In turn, for any integer $\mathscr{C}(N)\geq N,$
     there exists an FNO $\Psi':L^2_{N}(\T^d,\R^{\doutput})\to L^2_{\mathscr{C}(N)}(\T^d,\R^{\ndim})$ such that 
    \begin{align*}
        \Psi'(u_N)=\mathcal{I}_{\mathscr{C}(N)}\Phi(u_N), \qquad \forall u_N \in L^2_{N}(\T^d,\R^{\doutput}),
    \end{align*}
    with \begin{align*}
        \mathscr{D}(\Psi')&= r_{\alpha,\doutput}\log(\eps_1^{-1}),\\
        \mathscr{W}(\Psi')&=r_{\alpha,\doutput}\ndim\mlpBound^{\doutput}\coefficientBound^{\doutput}\eps_1^{-\frac{\doutput}{\alpha}},\\
        \mathscr{S}(\Psi')&=r_{\alpha,\doutput}\mlpBound^{\alpha+1}\coefficientBound^2\eps_1^{-\frac{1}{\alpha}},\\
        \mathscr{C}(\Psi')&=\mathscr{C}(N).
    \end{align*}
    By Lemma~\ref{lemma:Ds FNO implementation}, there exists an FNO $\Psi_s:L^2_{\mathscr{C}(N)}(\T^d,\R^{\ndim})\to L^2_{\mathscr{C}(N)}(\T^d,\R^{\doutput})$ such that
    \begin{align*}
        \mathcal{D}_s(v)=\Psi_s(v), \qquad  \forall v_N\in L^2_{\mathscr{C}(N)}(\T^d,\R^{\ndim}),
    \end{align*} 
     with 
    \begin{align*}
        \mathscr{D}(\Psi_s)&= 1,\\
        \mathscr{W}(\Psi_s)&=2(\ndim+\doutput),\\
        \mathscr{S}(\Psi_s)&=d\DcoefficientBound \mathscr{C}(N)^s ,\\
        \mathscr{C}(\Psi_s)&=\mathscr{C}(N).
    \end{align*}
    Finally, by the same argument as in Lemma~\ref{lemma:B FNO implementation}, the operator
     $\mathcal{P}_N:L^2_{\mathscr{C}(N)}(\T^d,\R^{\doutput})\to L^2_N(\T^d,\R^{\doutput})$ 
    can be implemented by a single affine FNO layer, and thus by the  composition Lemma~\ref{lemma:FNO Composition},
    the mapping 
     $u_N\mapsto  \mathcal{P}_N\mathcal{D}_s\bigl(\mathcal{I}_{\mathscr{C}(N)}\Phi(u_N)\bigr)=\mathcal{P}_N \Psi_s\bigl(\Psi'(u_N)\bigr)$ 
    can be implemented exactly by an FNO $\Psi_{\nonlinearity}:L^2_{N}(\T^d,\R^{\doutput})\to L^2_{N}(\T^d,\R^{\doutput})$ of size
    \begin{align*}
        \mathscr{D}(\Psi_{\nonlinearity})&= r_{\alpha,\doutput}\log(\eps_1^{-1})+2,\\
        \mathscr{W}(\Psi_{\nonlinearity})&=r_{\alpha,\doutput}\ndim\mlpBound^{\doutput}\coefficientBound^{\doutput}\eps_1^{-\frac{\doutput}{\alpha}},\\
        \mathscr{S}(\Psi_{\nonlinearity})&=\max\Bigl\{r_{\alpha,\doutput}\mlpBound^{\alpha+1}\coefficientBound^2\eps_1^{-\frac{1}{\alpha}},d\DcoefficientBound \mathscr{C}(N)^s\Bigr\},\\
        \mathscr{C}(\Psi_{\nonlinearity})&=\mathscr{C}(N).
    \end{align*}
Since $\mathcal{P}_N$ is implemented exactly and $\widehat{P}_{\mathcal{P}_N}(0)=0$ by assumption, it holds that 
\begin{align*}
    \frac{1}{(2\pi)^d}\int_{\T^d}\Psi_{\nonlinearity}(u_N)(x) \, dx & =0.
\end{align*}
    We now choose $\eps_1$ and $\mathscr{C}(N)$ to guarantee accuracy $\eps_{\nonlinearity}.$
    First, note that Nikolskii's inequality yields that
    \begin{align*}
        \sup_{x\in \T^d}\|u_N(x)\|_2\leq rN^{\frac{d}{2}}\inputBall.
    \end{align*}
    Taking $\mlpBound=r N^{\frac{d}{2}}\inputBall$ we have that, for each $j \in \mathsf{J}_{\mathscr{C}(N)}$,
\begin{align*}
    \left\|\Phi(u_N)(x_j)-\nonlinearity(u_N)(x_j)\right\|_2\leq \sqrt{\ndim} \eps_1,
\end{align*}
from which it follows that, for all $k\in \findices_{\mathscr{C}(N)},$ 
                  \begin{align*}
                      &\Bigl \|\Bigl({\bf F} _{\mathscr{C}(N)}\{\Phi(u_N)(x_j)\}_{j\in \indices_{\mathscr{C}(N)}}\Bigr)(k)-\Bigl(\F_{\mathscr{C}(N)}\nonlinearity(u_N)\Bigr)(k) \Bigr\|_2\\
                      & =\biggl\|\frac{1}{(2\mathscr{C}(N)+1)^d}\sum_{j\in \indices_{\mathscr{C}(N)}}\Bigl(\Phi\bigl(u_N(x_j)\bigr)-\nonlinearity\bigl(u_N(x_j)\bigr)\Bigr)e^{-i\langle x_j,k\rangle} \biggr\|_2\\
                      & \leq \frac{1}{(2\mathscr{C}(N)+1)^d}\sum_{j\in \indices_{\mathscr{C}(N)}} \left\|\Phi\bigl(u_N(x_j)\bigr)-\nonlinearity\bigl(u_N(x_j) \bigr) \right\|_2\\
                      & \leq  \sqrt{\ndim} \eps_1.
                  \end{align*}
We then upper bound the error by two terms, one capturing the neural network approximation error and the other capturing the trigonometric polynomial interpolation error.
    \begin{align*}
    \left\|\mathcal{A}^{-1/2}\bigl(\Psi_{\nonlinearity}(u_N)-\mathcal{P}_N\mathcal{D}_s\nonlinearity(u_N)\bigr) \right\|_{L^2}
    & = \left\|\mathcal{A}^{-1/2}\bigl(\mathcal{P}_N\mathcal{D}_s\mathcal{I}_{\mathscr{C}(N)}\Phi(u_N)-\mathcal{P}_N\mathcal{D}_s\nonlinearity(u_N)\bigr)\right\|_{L^2}\\
    & \leq \left\|\mathcal{A}^{-1/2}\Bigl(\mathcal{P}_N\mathcal{D}_s\mathcal{I}_{\mathscr{C}(N)}\bigl(\Phi(u_N)-\nonlinearity(u_N)\bigr)\Bigr)\right\|_{L^2}\\ 
&+\Bigl\|\mathcal{A}^{-1/2}\mathcal{D}_s\mathcal{P}_N \bigl(I-\mathcal{I}_{\mathscr{C}(N)}\bigr)\nonlinearity(u_N) \Bigr\|_{L^2}.
\end{align*}
   To control the first of these terms, we use the assumption that $s\leq \dissipation/2$ and the scaling of $\mathcal{A}$ given in \eqref{eq:A scaling assumption}, which implies that $\|\hat{P}_{\mathcal{A}^{-1/2}}(k)\|_{op}\lesssim \inverseBound^{-1/2}\|k\|_2^{-\dissipation/2}$, to get that 
    
   { \scriptsize
\begin{align*}
&\left\|\mathcal{A}^{-1/2}\Bigl(\mathcal{P}_N\mathcal{D}_s\mathcal{I}_{\mathscr{C}(N)}\bigl(\Phi(u_N)-\nonlinearity(u_N)\bigr)\Bigr)\right\|_{L^2}\\ 
    &\hspace{0cm} =\left\|\sum_{k\in\findices_N\backslash\{0\}}\hat{P}_{\mathcal{A}^{-1/2}}(k)\hat{P}_{\mathcal{P}_N}(k)\sum_{\ell=1}^{\ndim}\sum_{\ell'=1}^{d}\begin{bmatrix}
         c_{\ell' \ell 1}(ik_{\ell'})^{s_{\ell' \ell 1}}\Bigl(\Bigl({\bf F}_{\mathscr{C}(N)}\{\Phi_{\ell}(u_N)(x_j)\}_{j\in \indices_{\mathscr{C}(N)}}\Bigr)(k)-\Bigl(\F_{\mathscr{C}(N)}g_{\ell}(u_N)\Bigr)(k)\Bigr)\\
         \vdots \\
         c_{\ell' \ell \doutput}(ik_{\ell'})^{s_{\ell' \ell \doutput}}\Bigl(\Bigl({\bf F}_{\mathscr{C}(N)}\{\Phi_{\ell}(u_N)(x_j)\}_{j\in \indices_{\mathscr{C}(N)}}\Bigr)(k)-\Bigl(\F_{\mathscr{C}(N)}g_{\ell}(u_N)\Bigr)(k)\Bigr)
     \end{bmatrix}e^{i\langle k,x\rangle }\right\|_{L^2}\\ 
     & \hspace{0cm} \leq r \DcoefficientBound \doutput \sum_{k\in \findices_N}\inverseBound^{-1/2}\|k\|_2^{-\dissipation/2+s}\Bigl\| \bigl(\mathcal{I}_{\mathscr{C}(N)}\Phi(u_N)\bigr)(k)-\bigl(\F_{\mathscr{C}(N)}\nonlinearity(u_N)\bigr)(k)\Bigr\|_2\\
     & \hspace{0cm} \leq r \DcoefficientBound \inverseBound^{-1/2} N^{\max\{d+s-\dissipation/2,0\}}\log(N)\eps_1.
\end{align*}
}

It remains to bound the interpolation error term. By the scaling of the eigenvalues of $\mathcal{A}$ given in  \eqref{eq:A scaling assumption} and the assumption that $s\leq \dissipation/2$, we have that $\|\mathcal{A}^{-1/2}\mathcal{D}_s\|_{op}\leq r\inverseBound^{-1/2}$ for some constant $r>0$. This, combined with the fact that  $\mathcal{P}_N$ is an $L^2$-orthogonal projection gives that
\begin{align*}
    \Bigl\|\mathcal{A}^{-1/2}\mathcal{D}_s\mathcal{P}_N \bigl(I-\mathcal{I}_{\mathscr{C}(N)}\bigr)\nonlinearity(u_N) \Bigr\|_{L^2}\leq r \inverseBound^{-1/2}\Bigl\|\bigl(I-\mathcal{I}_{\mathscr{C}(N)}\bigr)\nonlinearity(u_N)\Bigr\|_{L^2}.
\end{align*}
Since $\alpha>\frac{d}{2}$ by assumption, we have by  \cite[Equation 5.5]{guo1998spectral} combined with the multivariate Faà di Bruno formula (chain rule) and Bernstein and Nikolskii's inequalities, that 
\begin{align*}
    \left\|\bigl(I-\mathcal{I}_{\mathscr{C}(N)}\bigr)\nonlinearity(u_N)\right\|_{L^2}&\leq  \mathscr{C}(N)^{-\alpha}\|\nonlinearity(u_N)\|_{H^\alpha}\\
    & \leq  \mathscr{C}(N)^{-\alpha} r_{d,\doutput,\alpha,\coefficientBound}\bigl(\|u_N\|_{H^\alpha}+\|u_N\|_{L^{\infty}}^{\alpha-1}\|u_N\|_{H^{\alpha}}\bigr)\\
    & \leq  \mathscr{C}(N)^{-\alpha} r_{d,\doutput,\alpha,\coefficientBound}N^{\alpha}N^{\frac{d}{2}(\alpha-1)}\inputBall^{\alpha}.
\end{align*}

Taking
\begin{align*}
    \eps_1 =\frac{\eps_{\nonlinearity}\inverseBound^{1/2}}{2r\DcoefficientBound N^{\max\{d+s-\dissipation/2,0\}}\log(N)}, \qquad \quad 
    \mathscr{C}(N)=
\max \left\{N,
\left\lceil
r_{d,\doutput,\alpha,\coefficientBound,\inverseBound}
\eps_{\nonlinearity}^{-1/\alpha}\inputBall\,N^{1+\frac d2}
\right\rceil
\right\},
\end{align*}
yields the desired result.  
\end{proof}
\end{lemma}

\section{Supplementary Materials Section~\ref{sec:polynomialnonlinearity}}
\subsection{Proof of Approximation Bound}\label{ssec:polynomial approximation theorem proof}
\begin{proof}[Proof of Theorem~\ref{thm: polynomial nonlinearity approximation theorem}]
         First note that for any $\PDE\in \pdeClass^{\textup{poly}}(\bar{\vartheta})$, $\Psi\in \Sigma^{\textup{poly}}(\eps,\bar{\vartheta})$, and $u \in \mathsf{U}_a,$  we have
         \begin{align*}
             \|\Psi(u) - \PDE(u)\|_{L^2}&\leq   \|\Psi(\mathcal{I}_Nu) - \spectralStep^j(\mathcal{I}_Nu)\|_{L^2} + \|\spectralStep^j(\mathcal{I}_Nu) - \PDE(u)\|_{L^2}\\
             &\leq  \|\Psi(\mathcal{I}_Nu) - \spectralStep^j(\mathcal{I}_Nu) \|_{L^2} + r_1\exp(r_2T)\Bigl(N^{-a}+\timestep\Bigr),
         \end{align*}
         where we have used that $\Psi(u)=\Psi(\mathcal{I}_Nu)$ for $u\in \mathsf{U}_{\smoothness}\subset C(\T^{d},\R^{\doutput})$ in the first inequality and Condition~\ref{conditionsec3}  \ref{item:spectral method} in the second.
         Taking $N=\lceil(\frac{4r_1}{\eps})^{1/a}\exp(r_2T/a)\rceil$  and $\timestep>0$ to be as large as possible such that $\timestep\leq \min \Bigl\{\frac{\eps}{4r_1}\exp(-r_2T),\frac{1}{4} \Bigr\}$ and $j\timestep = T$, we have that
         \begin{align*}
             \|\Psi(u) -\PDE(u)\|_{L^2}&\leq \|\Psi(\mathcal{I}_Nu) - \spectralStep^j(\mathcal{I}_Nu)\|_{L^2} + \frac{\eps}{2}.
         \end{align*}
         Therefore, it remains to construct a $\Psi\in \Sigma^{\textup{poly}}(\eps,\bar{\vartheta})$ that approximates $\spectralStep^j$ to accuracy $\eps/2.$ To that end, we first construct an FNO that approximates the single-step spectral method $\spectralStep.$ Recall that, for any $u_N\in L^2(\T^d,\R^{\doutput})$, 
         \begin{align*}
             \spectralStep(u_N)=\linearOp  \Bigl(\bigl(\frac{1}{\timestep}I+\gamma \mathcal{A}^{\power}\bigr)u_N- \mathcal{D}_s\bigl(\nonlinearity(u_N)\bigr)+\forcing_N\Bigr),
         \end{align*}
         for some $\linearOp\in \linearOpClass(\dissipation,\gamma,\power,\inverseBound),$ $\nonlinearity\in \nonlinearClass(p,\coefficientBound),$ and $\forcing_N\in \forcingClass(\forcingBound)$. In Lemmas~\ref{lemma:B FNO implementation}, \ref{lemma:affine part FNO implementation}, \ref{lemma:Ds FNO implementation} and \ref{lemma:polynomial part single step},  we have constructed FNOs to implement or approximate each part of this mapping. In particular, for any $\linearOp\in \linearOpClass(\dissipation,\gamma,\power,\inverseBound)$, we construct an FNO $\Psi_{\linearOp}:L^2_{pN}(\T^d,\R^{\doutput})\to L^2_{N}(\T^d,\R^{\doutput})$ such that $\Psi_{\linearOp}(u)=\linearOp u$. This FNO has size given by
         \begin{align*}
        \mathscr{D}(\Psi_{\linearOp})&=1,\\
        \mathscr{W}(\Psi_{\linearOp})&=2\doutput,\\
        \mathscr{S}(\Psi_{\linearOp})&=1,\\
        \mathscr{C}(\Psi_{\linearOp})&=pN.
    \end{align*}
    Similarly, for any positive semi-definite $\mathcal{A}:L^2_{N}(\T^d,\R^{\doutput})\to L^2_{N}(\T^d,\R^{\doutput})$ diagonalizable in Fourier space satisfying the scaling \eqref{eq:A scaling assumption}, and any $f_N\in \forcingClass(\forcingBound)$, Lemma~\ref{lemma:affine part FNO implementation} constructs an FNO $\Psi_{\textup{affine}}:L^2_{N}(\T^d,\R^{\doutput})\to L^2_{N}(\T^d,\R^{\doutput})$ such that $ \Psi_{\textup{affine}}(u)=(\frac{1}{\timestep}I+\gamma \mathcal{A}^{\power})u+f_N,$ satisfying 
         \begin{align*}
                \mathscr{D}(\Psi_{\textup{affine}})&=1,\\
        \mathscr{W}(\Psi_{\textup{affine}})&=2\doutput,\\
        \mathscr{S}(\Psi_{\textup{affine}})&=\max \bigl\{\timestep^{-1},r\gamma N^{\dissipation \power},\forcingBound \bigr\},\\
        \mathscr{C}(\Psi_{\textup{affine}})&=N.
    \end{align*}
         Finally, for any $\nonlinearity\in \nonlinearClass(p,\coefficientBound)$, Lemma~\ref{lemma:polynomial part single step} constructs an FNO $\Psi_{\nonlinearity}:L^2_{N}(\T^d,\R^{\doutput})\to L^2_{N}(\T^d,\R^{\doutput})$  that can approximate the mean-zero part of $\D_s(\nonlinearity\cdot)$
to accuracy $\eps_{\nonlinearity}$ in the $\mathcal{A}^{-1/2}$ weighted $L^2$-norm. Bounding the error in this weighted norm will be most natural for controlling the propagation of the error over the $j$ compositions of this FNO required to approximate the PDE solution. Precisely, Lemma~\ref{lemma:polynomial part single step} shows that, for any $u_N\in L^2_{N}(\T^d,\R^{\doutput})$ with $\|u_N\|_{L^2}\leq \inputBall'$,
\begin{align*}
\left\|
\mathcal A^{-1/2}
\left(
\Psi_{\nonlinearity}(u_N)
-
\mathcal P_N\mathcal D_s(\nonlinearity(u_N))
-
\bar E(u_N)
\right)
\right\|_{L^2}
\le
\eps_{\nonlinearity},
\end{align*}
where by subtracting $\bar E(u_N)$ we remove the mean contribution, given by
\[
\bar E(u_N)
:=
\frac{1}{(2\pi)^d}\int_{\T^d}
\Bigl[
\Psi_{\nonlinearity}(u_N)(x)
-
\mathcal P_N\mathcal D_s(\nonlinearity(u_N))(x)
\Bigr]\,dx .
\]
This FNO has a size that depends on the desired accuracy, with depth scaling logarithmically in $\eps_{\nonlinearity}$:
\begin{align*}
        \mathscr{D}(\Psi_{\nonlinearity})&=r\log(\eps^{-1}_{\nonlinearity}N\inputBall'),\\
        \mathscr{W}(\Psi_{\nonlinearity})&=r\ndim{p+\doutput-1 \choose \doutput-1},\\
        \mathscr{S}(\Psi_{\nonlinearity})&=rN^s,\\
        \mathscr{C}(\Psi_{\nonlinearity})&=pN.
    \end{align*}
         This FNO also approximates the mean of $\mathcal{D}_s\bigl(\nonlinearity(\cdot)\bigr)$ to accuracy $\eps_{\nonlinearity},$
         \begin{align*}
        \|\bar{E}(u_N)\|_{L^2}\leq \eps_{\nonlinearity}.
    \end{align*}
         To approximate $\spectralStep,$ we consider the composition of $\Psi_{\linearOp}$ with the sum of $\Psi_{\nonlinearity}$ and $\Psi_{\textup{affine}}$,
         \begin{align*}
             \Psi_{\spectralStep}(u_N):=\Psi_{\linearOp}\bigl(\Psi_{\textup{affine}}(u_N)-\Psi_{\nonlinearity}(u_N)\bigr),
         \end{align*}
         which by the FNO linear combination Lemma~\ref{lemma:FNO linear combination} and the FNO composition Lemma~\ref{lemma:FNO Composition}, can be implemented by an FNO $\Psi_{\spectralStep}:L^2_N(\T^d,\R^{\doutput}) \to L^2_N(\T^d,\R^{\doutput})$ of size 
         \begin{align*}
              \mathscr{D}(\Psi_{\spectralStep})&=r\log(\eps^{-1}_{\nonlinearity}N\inputBall'),\\
        \mathscr{W}(\Psi_{\spectralStep})&=r\ndim{p+\doutput-1 \choose \doutput-1},\\
        \mathscr{S}(\Psi_{\spectralStep})&=\max \bigl\{\timestep^{-1},r\gamma N^{\dissipation\power},N^s,\forcingBound \bigr\},\\
        \mathscr{C}(\Psi_{\spectralStep})&=pN.
         \end{align*}
         To approximate $\spectralStep^j$, we compose $\Psi_{\spectralStep}$ with itself $j$ times. By the FNO composition Lemma~\ref{lemma:FNO Composition}, there exists an FNO $\Psi:L^2_N(\T^d,\R^{\doutput}) \to L^2_N(\T^d,\R^{\doutput})$ implementing 
         \begin{align*}
\Psi(u_N)=\Psi_{\spectralStep}^j(u_N):=\underbrace{\Psi_{\spectralStep}\circ \Psi_{\spectralStep}\circ \cdots \circ\Psi_{\spectralStep}}_{j \text{ times}}(u_N),
         \end{align*}
         with 
         \begin{align*}
              \mathscr{D}(\Psi)&=rj\log(\eps^{-1}_{\nonlinearity}N\inputBall'),\\
        \mathscr{W}(\Psi)&=r{p+\doutput-1 \choose \doutput-1},\\
        \mathscr{S}(\Psi)&=\max \bigl\{\timestep^{-1},r\gamma N^{\dissipation\power},N^s,\forcingBound \bigr\},\\
        \mathscr{C}(\Psi)&=pN.
         \end{align*}
In the remainder of this proof, we argue that by choosing $\eps_{\nonlinearity}$ small enough, we can guarantee that $\|\Psi(\mathcal{I}_Nu) - \spectralStep^j(\mathcal{I}_Nu)\|_{L^2}\leq\frac{\eps}{2}$, from which it follows that $\Psi$ can approximate $\PDE$ to accuracy $\eps.$

Note that we have constructed $\Psi$ such that 
\begin{align*}
    \Psi(u_N)
    =
    \linearOp\Bigl(
        \bigl(\frac{1}{\timestep}I+\gamma \mathcal{A}^{\power}\bigr)
        \Psi_{\spectralStep}^{j-1}(u_N)
        -\Psi_{\nonlinearity}\bigl(\Psi_{\spectralStep}^{j-1}(u_N)\bigr)
        +f_N
    \Bigr).
\end{align*}
   Plugging in $\linearOp= \Bigl(\frac{1}{\timestep}I+\gamma \mathcal{A}^{\power} +
    \mathcal{A}\Bigr)^{-1} \mathcal{P}_N$ and rearranging, we get that
\begin{align}\label{eq:polynomial FNO iterates}
    \frac{\Psi(u_N)-\Psi_{\spectralStep}^{j-1}(u_N)}{\timestep}+ \mathcal{A}\Psi(u_N)+\gamma \mathcal{A}^{\power}\bigl(\Psi(u_N)-\Psi_{\spectralStep}^{j-1}(u_N)\bigr)+\Psi_{\nonlinearity}\bigl(\Psi_{\spectralStep}^{j-1}(u_N)\bigr) = f_N.
\end{align}
On the other hand, the iterates of the single-step spectral method $\spectralStep$ satisfy
\begin{align}\label{eq:polynomial spectral iterates}
    \frac{\spectralStep^j(u_N)-\spectralStep^{j-1}(u_N)}{\timestep}+  \mathcal{A}\spectralStep^{j}(u_N)+\gamma \mathcal{A}^{\power}\bigl(\spectralStep^j(u_N)-\spectralStep^{j-1}(u_N)\bigr)+\mathcal{P}_N \mathcal{D}_s\bigl(\nonlinearity(\spectralStep^{j-1}(u_N))\bigr) = f_N.
\end{align}
For notational brevity, we write $u^{j}=\Psi_{\spectralStep}^{j}(u_N)$ and $v^{j}=\spectralStep^j(u_N)$ for $j\in \N$, and denote the error between the FNO iterate and the spectral method iterate as 
$e^j=u^j-v^j=\Psi_{\spectralStep}^j(u_N)-\spectralStep^j(u_N).$
Since $\Psi_{\linearOp}$ includes the projection $\mathcal P_N$ and powers of $\mathcal{A}$ commute with $\mathcal{P}_N$, the FNO iterates $u^j=\Psi_{\spectralStep}^j(u_N)$ belong to $\operatorname{Range}(\mathcal P_N)$. Similarly, the spectral iterates $v^j=\spectralStep^j(u_N)$ belong to $\operatorname{Range}(\mathcal P_N)$. 
Subtracting \eqref{eq:polynomial spectral iterates} from \eqref{eq:polynomial FNO iterates} yields
\begin{align*}
    \frac{e^j-e^{j-1}}{\timestep}+ \mathcal{A} e^{j}+\gamma \mathcal{A}^{\power}(e^j-e^{j-1})+\Psi_{\nonlinearity}(u^{j-1})-\mathcal{P}_N\mathcal{D}_s\bigl(\nonlinearity(v^{j-1})\bigr)=0.
\end{align*}
Taking the inner product of this expression with $e^j$, we have
\begin{equation}\label{eq:error evolution}
\begin{split}
    &\frac{1}{2\timestep}\bigl(\|e^j\|_{L^2}^2-\|e^{j-1}\|_{L^2}^2+\|e^j-e^{j-1}\|_{L^2}^2\bigr)+  \|\mathcal{A}^{1/2}e^j\|_{L^2}^2 \\
    &\hspace{2cm} +\dfrac{\gamma}{2}(\| \mathcal{A}^{\power/2}e^j\|_{L^2}^2-\|\mathcal{A}^{\power/2}e^{j-1}\|_{L^2}^2+\|\mathcal{A}^{\power/2}(e^j-e^{j-1})\|_{L^2}^2)\\
    & \hspace{2cm} = \left\langle
\mathcal{P}_N\mathcal{D}_s\bigl(\nonlinearity(v^{j-1})\bigr)
-\Psi_{\nonlinearity}(u^{j-1}), e^j
\right\rangle  
    :=\mathcal{E}.
    \end{split}
\end{equation}

We now pursue an upper bound on the inner product term on the right-hand side. Denoting  
\begin{align*}
    \bar{E}(u^{j-1})
    &:=
    \frac{1}{(2\pi)^d}\int_{\T^d}
    \Bigl[
        \Psi_{\nonlinearity}(u^{j-1})(x)
        -\mathcal{P}_N\mathcal{D}_s\bigl(\nonlinearity(u^{j-1})\bigr)(x)
    \Bigr]\,dx, \\
    \bar{E}'(u^{j-1},v^{j-1})
    &:=
    \frac{1}{(2\pi)^d}\int_{\T^d}
    \Bigl[
        \mathcal{D}_s\bigl(\nonlinearity(u^{j-1})\bigr)(x)
        -\mathcal{D}_s\bigl(\nonlinearity(v^{j-1})\bigr)(x)
    \Bigr]\,dx ,
\end{align*}
 Young's inequality gives
\begin{align*}
|\mathcal{E}|
&\leq
\Bigl|
\Bigl\langle
    \mathcal{P}_N\mathcal{D}_s\bigl(\nonlinearity(v^{j-1})\bigr)
    -\mathcal{P}_N\mathcal{D}_s\bigl(\nonlinearity(u^{j-1})\bigr)
    -\bar{E}'(v^{j-1},u^{j-1}),
    e^j
\Bigr\rangle
\Bigr|  \\
& +
\Bigl|
\bigl\langle
    \bar{E}'(v^{j-1},u^{j-1}),
    e^j
\bigr\rangle
\Bigr| \\
& +
\Bigl|
\Bigl\langle
    \Psi_{\nonlinearity}(u^{j-1})
    -\mathcal{P}_N\mathcal{D}_s\bigl(\nonlinearity(u^{j-1})\bigr)
    -\bar{E}(u^{j-1}),
    e^j
\Bigr\rangle
\Bigr| \\
& +
\Bigl|
\bigl\langle
    \bar{E}(u^{j-1}),
    e^j
\bigr\rangle
\Bigr| \\
&\leq
\frac{1}{2}
\Bigl\|
\mathcal{A}^{-1/2}
\Bigl(
    \mathcal{P}_N\mathcal{D}_s\bigl(\nonlinearity(v^{j-1})\bigr)
    -\mathcal{P}_N\mathcal{D}_s\bigl(\nonlinearity(u^{j-1})\bigr)
    -\bar{E}'(v^{j-1},u^{j-1})
\Bigr)
\Bigr\|_{L^2}^2
+
\frac{1}{2}
\|\mathcal{A}^{1/2}e^j\|_{L^2}^2 \\
& +
\frac{1}{2}
\Bigl\|
\mathcal{A}^{-1/2}
\Bigl(
    \Psi_{\nonlinearity}(u^{j-1})
    -\mathcal{P}_N\mathcal{D}_s\bigl(\nonlinearity(u^{j-1})\bigr)
    -\bar{E}(u^{j-1})
\Bigr)
\Bigr\|_{L^2}^2
+
\frac{1}{2}
\|\mathcal{A}^{1/2}e^j\|_{L^2}^2 \\
& +
\frac{1}{2}\|\bar{E}'(v^{j-1},u^{j-1})\|_{L^2}^2
+
\frac{1}{2}\|e^j\|_{L^2}^2
+
\frac{1}{2}\|\bar{E}(u^{j-1})\|_{L^2}^2
+
\frac{1}{2}\|e^j\|_{L^2}^2 ,
\end{align*}
    Since $s\leq \dissipation/2$, we have that $\|\mathcal{A}^{-1/2}\mathcal{D}_s\|_{op}\leq d\ndim\DcoefficientBound\inverseBound^{-\frac{1}{2}}$. This, combined with Lemma \ref{lemma: polynomial nonlinearity lipschitz constant} yields 
    \begin{align*}
     |\mathcal{E}|&\leq \frac{1}{2}r_N \|e^{j-1}\|_{L^2}^2+\frac{1}{2}\|\mathcal{A}^{1/2}e^j\|_{L^2}^2\\
     & + \frac{1}{2}\left\|\mathcal{A}^{-1/2}\Bigl(\Psi_{\nonlinearity}(u^{j-1})-\mathcal{P}_N\mathcal{D}_s\bigl(\nonlinearity(u^{j-1})\bigr)-\bar{E}(u^{j-1})\Bigr)\right\|_{L^2}^2+\frac{1}{2}\|\mathcal{A}^{1/2}e^j\|_{L^2}^2\\
    &+ \frac{1}{2}r_N\|e^{j-1}\|_{L^2}^2+\frac{1}{2}\|e^j\|_{L^2}^2+\frac{1}{2}\|\bar{E}(u^{j-1})\|_{L^2}^2+\frac{1}{2}\|e^j\|_{L^2}^2 \, ,
\end{align*}
where we have defined
\[
r_N
=
r\,\DcoefficientBound^2\inverseBound^{-1}
N^{d(p-1)}(\inputBall')^{2p-2}.
\]
\nc
 
 Assuming that $\inputBall'$ is taken to be large enough such that $\|u^j\|_{L^2}\leq \inputBall'$ (we will specify a specific choice of $\inputBall'$ shortly below),  Lemma~\ref{lemma:polynomial part single step} yields
\begin{align*}
\Bigl\|
\mathcal{A}^{-1/2}
\Bigl(
    \Psi_{\nonlinearity}(u^{j-1})-\mathcal{P}_N\mathcal{D}_s\bigl(\nonlinearity(u^{j-1})\bigr)
    -\bar{E}(u^{j-1})
\Bigr)
\Bigr\|_{L^2}^2
\leq
\eps_{\nonlinearity}^2,
\end{align*}
and 
\begin{align*}
    \Bigl\| \bar{E}(u^{j-1})\|_{L^2}^2\leq \eps_{\nonlinearity}^2.
\end{align*}
Consequently,
\begin{align*}
    |\mathcal{E}|&\leq 
      \|\mathcal{A}^{1/2}e^j\|_{L^2}^2 +  r_N\|e^{j-1}\|_{L^2}^2
 + \eps_{\nonlinearity}^2
+\|e^j\|_{L^2}^2.
\end{align*}
Plugging this bound into \eqref{eq:error evolution}, we have
\begin{align*}
    &\frac{1}{2\timestep}
    \bigl(
        \|e^j\|_{L^2}^2
        -\|e^{j-1}\|_{L^2}^2
        +\|e^j-e^{j-1}\|_{L^2}^2
    \bigr) \\
    &
    +\frac{\gamma}{2}
    \bigl(
        \|\mathcal{A}^{\power/2}e^j\|_{L^2}^2
        -\|\mathcal{A}^{\power/2}e^{j-1}\|_{L^2}^2
        +\|\mathcal{A}^{\power/2}(e^j-e^{j-1})\|_{L^2}^2
    \bigr) \\
    &\leq
    r_N\|e^{j-1}\|_{L^2}^2
    +\eps_{\nonlinearity}^2
    +\|e^j\|_{L^2}^2 ,
\end{align*}
    from which it follows that 
    \begin{align*}
        \|e^j\|_{L^2}^2+2\timestep \gamma \|\mathcal{A}^{\power/2}e^j\|_{L^2}^2\leq \bigl( \timestep r_N+2\bigr)\bigl(\|e^{j-1}\|_{L^2}^2+2\timestep\gamma\|\mathcal{A}^{\power/2}e^{j-1}\|_{L^2}^2\bigr)+4\timestep \eps_{\nonlinearity}^2.
    \end{align*}
Applying Lemma~\ref{lemma:grownall lemma 1} with  $y_j=\|e^j\|_{L^2}^2+2\timestep\gamma\|\mathcal{A}^{\power/2}e^j\|_{L^2}^2$, $\alpha=\timestep r_N+2$, and $\beta=4\timestep\eps_{\nonlinearity}^2,$ we have that
    \begin{align*}
        \|e^j\|_{L^2}^2\leq \|e^j\|_{L^2}^2+2\timestep\gamma\|\mathcal{A}^{\power/2}e^j\|_{L^2}^2\leq\frac{1-(\timestep r_N+2)^j}{1-(\timestep r_N+2)}4\timestep\eps_{\nonlinearity}^2\leq  4\bigr(\timestep r_N+2\bigl)^j\timestep \eps_{\nonlinearity}^2.
    \end{align*}

Recalling the definition of $e^j$, plugging in the choices
$N=\lceil(\frac{4r_1}{\eps})^{1/a}\exp(r_2T/a)\rceil$ and
$\timestep\leq\frac{\eps}{4r_1}\exp(-r_2T)$, and using $j\timestep=T$, we have
\begin{align*}
    \|\Psi(\mathcal{I}_Nu)-\spectralStep^j(\mathcal{I}_Nu)\|_{L^2}
    \leq
    r
    \bigl(r\eps^{1-\frac{d(p-1)}{\smoothness}}+2\bigr)^{r'\eps^{-1}}
    \eps^{1/2}
    \eps_{\nonlinearity},
\end{align*}
for $r,r'>0$ independent of $\eps.$
Thus, taking
\[
\eps_{\nonlinearity}
=
r\eps^{1/2}
\bigl(r\eps^{1-\frac{d(p-1)}{\smoothness}}+2\bigr)^{-r'\eps^{-1}}.
\]
and $\inputBall'= \max \bigl\{\inputBall,\outputBall \bigr\}+\eps$ yields that $\Psi$ is an FNO of the size specified in $\Sigma^{\text{poly}}(\eps,\bar{\vartheta})$ that can approximate $\PDE$ to $\eps$ accuracy, as claimed. 

Finally, we remark that 
    \begin{align*}
    \sup_{u\in \inputBall_{\smoothness}}\|\Psi(u)\|_{L^2}\leq\sup_{u\in \inputBall_{\smoothness}}\|\Psi(u)-\PDE(u)\|_{L^2}+\sup_{u\in \inputBall_{\smoothness}}\|\PDE(u)\|_{L^2}\leq \eps+\outputBall\leq 2\outputBall, 
\end{align*}
so $\Psi\in \Sigma^{\textup{poly}}(\eps,\bar{\vartheta})$, completing the proof.
     \end{proof}

\subsection{Proof of Learning Bound}\label{ssec: polynomail learning proof}
This section contains the proof of Theorem~\ref{thm: learning}. Recall that we train FNOs by minimizing the empirical squared $L^2$ prediction risk
\begin{align*}
    \hat{\risk}(\Psi,\PDE)=\frac{1}{\samples}\sum_{\sample=1}^{\samples}\|\Psi(u^{(\sample)})-\PDE(u^{(\sample)})\|_{L^2}^2.
\end{align*}
We want to show that a global minimizer of the empirical risk over $\Sigma^{\textup{poly}}(\eps,\bar{\vartheta})$, which we denote by $\hat{\Psi}_{\PDE}$ has small population risk,
\begin{align*}
    \risk(\Psi,\PDE)=\mathbb{E}_{u\sim\mu}\|\Psi(u)-\PDE(u)\|_{L^2}^2,
\end{align*}
with high probability. Our proof is adapted from \cite{kovachki2024data}. First, in Lemma~\ref{lemma:approximation estimation bound} we bound the population risk of the ERM by an approximation error term, which uses the approximation result in Theorem~\ref{thm: polynomial nonlinearity approximation theorem} and an estimation error term which we need to control. To control this estimation error term, we prove a concentration inequality for the population risk $\risk(\hat{\Psi}_{\PDE},\PDE)$ in Proposition~\ref{prop:risk bound} in terms of the metric entropy of $\Sigma^{\textup{poly}}(\eps,\bar{\vartheta})$ and $\pdeClass^{\textup{poly}}(\bar{\vartheta})$. Lemmas~\ref{lemma: risk deviation lemma}, \ref{lemma: risk lipschitz}, and \ref{lemma: f lipschitz} are used to prove this proposition. The major difference between our proof and that in \cite{kovachki2024data} is that they require $\|\mathcal{S}(u)\|_{L^2}\leq 1$, and our result holds for a generic upper bound $\outputBall$. 
Throughout this section, we occasionally write $\pdeClass$ and $\Sigma$ for $\pdeClass^{\textup{poly}}(\bar{\vartheta})$ and $\Sigma^{\textup{poly}}(\eps,\bar{\vartheta})$ to streamline the notation.
\begin{lemma}\label{lemma:approximation estimation bound}
    Let $\PDE\in \pdeClass^{\textup{poly}}(\bar{\vartheta})$. Given $u^{(1)},\hdots,u^{(\samples)}\in \sobolevBall$, let
    \begin{align*}
        \hat{\Psi}_{\PDE}\in \argmin_{\Psi \in \Sigma^{\textup{poly}}(\eps,\bar{\vartheta})}\hat{\risk}(\Psi,\PDE).
    \end{align*}
   For any $\eps\in (0,1/2),$ it holds that 
   \begin{align*}
       \risk(\hat{\Psi}_{\PDE},\PDE)\leq \eps^2+ \sup_{\PDE'\in \pdeClass,\Psi'\in \Sigma^{\textup{poly}}(\eps,\bar{\vartheta})}\bigl[ \risk(\Psi',\PDE')-\hat{\risk}(\Psi',\PDE')\bigr].
   \end{align*}
\begin{proof}
      Observe that
       \begin{align*}
           \risk(\hat{\Psi}_{\PDE},\PDE)
           &=\hat{\risk}(\hat{\Psi}_{\PDE},\PDE)
           +\risk(\hat{\Psi}_{\PDE},\PDE)-\hat{\risk}(\hat{\Psi}_{\PDE},\PDE)\\
           & \leq \hat{\risk}(\hat{\Psi}_{\PDE},\PDE)
           + \sup_{\PDE'\in \pdeClass,\Psi'\in \Sigma^{\textup{poly}}(\eps,\bar{\vartheta})}
           \bigl[ \risk(\Psi',\PDE')-\hat{\risk}(\Psi',\PDE')\bigr].
       \end{align*}
       By Theorem~\ref{thm: polynomial nonlinearity approximation theorem}, there exists
       $\Psi_{\PDE}^{\ast}\in \Sigma^{\textup{poly}}(\eps,\bar{\vartheta})$ such that
       \begin{align*}
           \sup_{u\in \sobolevBall}
           \|\Psi_{\PDE}^{\ast}(u)-\PDE(u)\|_{L^2}\leq \eps.
       \end{align*}
       Hence, since $u^{(1)},\hdots,u^{(\samples)}\in \sobolevBall$, we have
       \begin{align*}
           \hat{\risk}(\Psi_{\PDE}^{\ast},\PDE)
           =\frac{1}{\samples}\sum_{\sample=1}^{\samples}
           \|\Psi_{\PDE}^{\ast}(u^{(\sample)})-\PDE(u^{(\sample)})\|_{L^2}^2
           \leq \eps^2.
       \end{align*}
       Since $\hat{\Psi}_{\PDE}$ minimizes the empirical risk over
       $\Sigma^{\textup{poly}}(\eps,\bar{\vartheta})$, it follows that
       \begin{align*}
           \hat{\risk}(\hat{\Psi}_{\PDE},\PDE)
           \leq \hat{\risk}(\Psi_{\PDE}^{\ast},\PDE)
           \leq \eps^2.
       \end{align*}
       Combining the preceding bounds yields
       \begin{align*}
           \risk(\hat{\Psi}_{\PDE},\PDE)
           \leq
           \eps^2+
           \sup_{\PDE'\in \pdeClass,\Psi'\in \Sigma^{\textup{poly}}(\eps,\bar{\vartheta})}
           \bigl[ \risk(\Psi',\PDE')-\hat{\risk}(\Psi',\PDE')\bigr],
       \end{align*}
       proving the desired result.
\end{proof}
\end{lemma}
\nc

Using Lemmas~\ref{lemma: risk deviation lemma}, \ref{lemma: risk lipschitz}, and \ref{lemma: f lipschitz} from Appendix~\ref{ssec:auxiliary lemmas} along with a union-bound argument, we will show a concentration inequality for the risk uniformly over $\Sigma^{\textup{poly}}(\eps,\bar{\vartheta})\times \pdeClass.$ To do so, we first recall the definitions of covering number and metric entropy.
\begin{definition}[Covering Number and Metric Entropy]
    Let $X$ be a set. A $\delta$-cover of $X$ with respect to a metric $\metric$ is a set $\{x_1,\hdots,x_\mathsf{M}\}\subset X$ such that for any $x\in X$ there exists $i\in \{1,\hdots,\mathsf{M}\}$ such that $\metric(x,x_i)\leq \delta$. The $\delta$-covering number of $X$, $\covering_X(\delta,\metric)$, is the cardinality of the smallest $\delta$-cover of $X.$ The metric entropy of $X$ with respect to $\metric$ is defined to be the logarithm of the covering number, $\log \covering_X(\delta,\metric).$
\end{definition}
Note that the covering number depends on both the set and metric. In what follows, we will consider covering numbers with respect to the supremum norm 
\begin{align*}    \metric(\Psi,\Psi'):=\sup_{u\in \sobolevBall}\|\Psi(u)-\Psi'(u)\|_{L^2(D)}.
\end{align*}

Consequently, we will drop the dependence on the metric in our notation for covering numbers in what follows. 
Throughout the analysis we allow external covers, meaning that the covering centers need not belong to the set being covered, but only to the ambient metric space. This convention is harmless here: any external $\delta$-cover can be converted into an internal $2\delta$-cover by choosing, from each nonempty covering ball, one point of the covered set.

The next result characterizes the concentration of the empirical risk over $\Sigma^{\textup{poly}}(\eps,\bar{\vartheta})\times \pdeClass.$
\begin{lemma}\label{lemma: union bound}
    Let $\delta>0$ and $\alpha\in (0,1].$ Then, for $u_1,\ldots,u_{\samples}\iid \mu,$ we have that
    \begin{align*}
            \prob \Biggl[\sup_{(\Psi,\PDE)\in \Sigma\times \pdeClass}\frac{\risk(\Psi,\PDE)-\hat{\risk}(\Psi,\PDE)}{\risk(\Psi,\PDE)+72\outputBall\delta \alpha^{-1}}\geq \alpha\Biggr]\leq \covering_{\Sigma}(\delta)\covering_{\pdeClass}(2\delta) \exp\Bigl(-\frac{2\alpha \samples \delta }{\outputBall}\Bigr),
        \end{align*}
        where $\covering_{\Sigma}(\delta)$ denotes the $\delta$-covering number of $\Sigma^{\textup{poly}}(\eps,\bar{\vartheta})$ and $\covering_{\pdeClass}(2\delta)$ denotes the $2\delta$-covering number of $\pdeClass.$
    \begin{proof}
    Let $\Psi_1,\ldots,\Psi_{\covering_{\Sigma}(\delta)}$ be a $\delta$-cover of $\Sigma^{\textup{poly}}(\eps,\bar{\vartheta})$ and let $\PDE_1,\ldots,\PDE_{\covering_{\pdeClass}(2\delta)}$ be a $2\delta$-cover of $\pdeClass.$ Consequently, for any pair $(\Psi,\PDE)\in \Sigma^{\textup{poly}}(\eps,\bar{\vartheta})\times\pdeClass$, there exists $1\leq \sample_{\Psi}\leq \covering_{\Sigma}(\delta)$ and $1\leq \sample_{\PDE}\leq \covering_{\pdeClass}(2\delta)$ such that
        \begin{align*}
            \sup_{u\in \sobolevBall}\|\Psi(u)-\Psi_{m_{\Psi}}(u)\|_{L^2}\leq \delta, \quad \sup_{u\in \sobolevBall}\|\PDE(u)-\PDE_{m_{\PDE}}(u)\|_{L^2}\leq 2\delta.
        \end{align*}
        By Lemma~\ref{lemma: f lipschitz} and Lemma~\ref{lemma: risk lipschitz}, we have that 
        \begin{align*}
            &\frac{\risk(\Psi,\PDE)-\hat{\risk}(\Psi,\PDE)}{\risk(\Psi,\PDE)+\chi}\leq f_\chi \bigl(\risk(\Psi,\PDE),\hat{\risk}(\Psi,\PDE) \bigr)\\
            & \hspace{1.5cm} \leq f_\chi \bigl(\risk(\Psi_{\sample_{\Psi}},\PDE_{\sample_{\PDE}}) ,\hat{\risk}(\Psi_{\sample_{\Psi}},\PDE_{\sample_{\PDE}}) \bigr)+\frac{1}{\chi}\bigl|\risk(\Psi,\PDE)-\risk(\Psi_{\sample_{\Psi}},\PDE_{\sample_{\PDE}})\bigr|+\frac{1}{\chi}\bigl|\hat{\risk}(\Psi,\PDE)-\hat{\risk}(\Psi_{\sample_{\Psi}},\PDE_{\sample_{\PDE}})\bigr|\\
            & \hspace{1.5cm} \leq f_\chi \bigl(\risk(\Psi_{\sample_{\Psi}},\PDE_{\sample_{\PDE}}),\hat{\risk}(\Psi_{\sample_{\Psi}},\PDE_{\sample_{\PDE}}) \bigr) + \frac{36V}{\chi}\delta.
        \end{align*}
        Taking the maximum over all $\sample_{\Psi}$ and $\sample_{\PDE}$ of the right-hand side and the supremum over $\Sigma^{\textup{poly}}(\eps,\bar{\vartheta})\times \pdeClass$ of the left-hand side, we have
        \begin{align*}
            \sup_{(\Psi,\PDE)\in \Sigma\times \pdeClass}\frac{\risk(\Psi,\PDE)-\hat{\risk}(\Psi,\PDE)}{\risk(\Psi,\PDE)+\chi} \leq \max_{1\leq \sample_{\Psi}\leq \covering_{\Sigma}(\delta),1\leq \sample_{\PDE}\leq \covering_{\pdeClass}(2\delta)}f_\chi\bigl(\risk(\Psi_{\sample_{\Psi}},\PDE_{\sample_{\PDE}}),\hat{\risk}(\Psi_{\sample_{\Psi}},\PDE_{\sample_{\PDE}}) \bigr) + \frac{36V}{\chi}\delta.
        \end{align*}
        Making the particular choice of $\chi=\frac{72\outputBall\delta}{\alpha}$, it follows that
        \begin{align*}
            & \prob \Biggl[\sup_{(\Psi,\PDE)\in \Sigma\times \pdeClass}\frac{\risk(\Psi,\PDE)-\hat{\risk}(\Psi,\PDE)}{\risk(\Psi,\PDE)+\chi}\geq \alpha\Biggr] \leq \prob \Biggl[\max_{1\leq \sample_{\Psi}\leq \covering_{\Sigma}(\delta),1\leq \sample_{\PDE}\leq \covering_{\pdeClass}(2\delta)}\frac{\risk(\Psi_{\sample_{\Psi}},\PDE_{\sample_{\PDE}})-\hat{\risk}(\Psi_{\sample_{\Psi}},\PDE_{\sample_{\PDE}})}{\risk(\Psi_{\sample_{\Psi}},\PDE_{\sample_{\PDE}})+\chi}\geq \frac{1}{2}\alpha\Biggr]\\
            & \hspace{3.75cm} \leq \covering_{\Sigma}(\delta)\covering_{\pdeClass}(2\delta) \max_{1\leq \sample_{\Psi}\leq \covering_{\Sigma}(\delta),1\leq \sample_{\PDE}\leq \covering_{\pdeClass}(2\delta)}\prob\Biggl[\frac{\risk(\Psi_{\sample_{\Psi}},\PDE_{\sample_{\PDE}})-\hat{\risk}(\Psi_{\sample_{\Psi}},\PDE_{\sample_{\PDE}})}{\risk(\Psi_{\sample_{\Psi}},\PDE_{\sample_{\PDE}})+\chi}\geq \frac{1}{2}\alpha\Biggr],
        \end{align*}
        where the final inequality follows from a union bound. 
        Finally, applying Lemma~\ref{lemma: risk deviation lemma} yields
        \begin{align*}
            \prob \Biggl[\sup_{(\Psi,\PDE)\in \Sigma\times \pdeClass}\frac{\risk(\Psi,\PDE)-\hat{\risk}(\Psi,\PDE)}{\risk(\Psi,\PDE)+72\outputBall\delta \alpha^{-1}}\geq \alpha\Biggr]\leq \covering_{\Sigma}(\delta)\covering_{\pdeClass}(2\delta) \exp\biggl(-\frac{2\alpha \samples \delta }{\outputBall}\biggr),
        \end{align*}
        as desired.
    \end{proof}
\end{lemma}
We now apply the previous lemma to prove a high probability bound on the population risk of an empirical risk minimizer, assuming we have enough samples relative to the metric entropy of the class of FNOs and operators. These metric entropies are controlled in the three subsequent results.
\begin{proposition}\label{prop:risk bound}
    Fix $\confidence \in (0,1]$ and $\eps,\delta>0$. If 
    $
        \samples\geq \outputBall^2\delta^{-2}\log\bigl(\confidence^{-1}\covering_{\Sigma}(\delta)\covering_{\pdeClass}(2\delta)\bigr)/2,
    $
   then 
    \begin{align*}
        \prob \Biggl[    \sup_{\PDE\in \pdeClass} \risk (\hat{\Psi}_{\PDE},\PDE)\leq \eps^2+81\outputBall\delta    \Biggr]  \ge 1-\confidence . 
    \end{align*}
    \begin{proof}
    Taking $\alpha=\delta/\outputBall$ in Lemma~\ref{lemma: union bound} and using the fact that $\risk(\Psi,\PDE)\leq 9\outputBall^2$ for all  $(\Psi,\PDE)\in \Sigma\times \pdeClass$, we have that 
    \begin{align*}
        \prob\Biggl[\sup_{(\Psi,\PDE)\in \Sigma \times \pdeClass}\risk(\Psi,\PDE)-\hat{\risk}(\Psi,\PDE)\geq 81\outputBall\delta\Biggr]\leq \covering_{\Sigma}(\delta)\covering_{\pdeClass}(2\delta)\exp \left(-\frac{2\samples \delta^2}{\outputBall^2}\right).
    \end{align*}
    Consequently, our assumption on $\samples$ guarantees that with probability $1-\confidence$, 
    \begin{align*}
        \sup_{(\Psi,\PDE)\in \Sigma \times \pdeClass}\risk(\Psi,\PDE)-\hat{\risk}(\Psi,\PDE)\leq 81\outputBall\delta.
    \end{align*}
        Combined with Lemma~\ref{lemma:approximation estimation bound}, we have that with probability $1-\confidence,$ 
        \begin{align*}
       \sup_{\PDE\in \pdeClass} \risk(\hat{\Psi}_\PDE,\PDE)\leq \eps^2+81\outputBall\delta,
   \end{align*}
   as claimed. 
    \end{proof}
\end{proposition}
It remains to characterize the metric entropy of $\Sigma$ and $\pdeClass$, which we do in the following three results. The first is a restatement of a result from \cite{kovachki2024data}.
\begin{lemma}[\cite{kovachki2024data} Proposition 3.9]\label{lemma:FNO metric entropy}
    Let $\covering_{\mathsf{FNO}}(\delta)$ denote the $\delta$-covering number of $\mathsf{FNO}(\mathscr{D},\mathscr{W},\mathscr{S},\mathscr{C})$ with respect to $\|\cdot\|_{C(\inputBall_\smoothness,L^2(D))}$. Assume that $\mathscr{S}\geq 1.$ Then, there exists a constant $r$ depending only on $d$ and $\inputBall$ such that
    \begin{align*}
        \log \covering_{\mathsf{FNO}}(\delta)\leq r \mathscr{C}^d\mathscr{D}^2\mathscr{W}^2 \log \left(\frac{\mathscr{S}\mathscr{D}\mathscr{W}\mathscr{C}}{\delta}\right).
    \end{align*}
    \end{lemma}
    As an immediate corollary of Lemma~\ref{lemma:FNO metric entropy}, we can characterize the metric entropy of $\Sigma^{\textup{poly}}(\eps,\bar{\vartheta}).$
    \begin{corollary}\label{corollary: rate of metric entropy}
        Fix $\eps>0$ and let $\covering_{\Sigma}(\delta)$ denote the $\delta$-covering number of $\Sigma^{\textup{poly}}(\eps,\bar{\vartheta})$ with respect to $\|\cdot\|_{C(\sobolevBall,L^2(D))}.$ For any $\delta\in (0,1)$, it holds that
        \begin{align*}
            \log \covering_{\Sigma}(\delta)\leq r\eps^{-4-\frac{d}{a}}\log(\eps^{-1}\delta^{-1})^3,
        \end{align*}
        where $r$ is increasing in $s,d,\doutput,p,F,T,\inverseBound^{-1},\dissipation^{-1}$ and $\gamma$.
    \end{corollary}
    The dependence of the metric entropy on the physical dimension $d$ and the smoothness parameter $a$ is natural. The factor $\eps^{-d/a}$ comes from the Fourier cutoff required to approximate functions in the $H^a$ input class; equivalently, it reflects the number of retained Fourier modes, and hence the number of parameters needed to describe the Fourier multipliers. This number grows rapidly with the physical dimension $d$. The smoothness parameter $\smoothness$ dictates the rate at which the Fourier cutoff $\mathscr{C}$ must grow in order to achieve a prescribed approximation accuracy. The remaining factor $\eps^{-4}$ comes from the depth of the time-stepping construction and the quadratic dependence on depth in the FNO metric entropy bound.
    

    The next lemma transfers a cover of the approximating FNO class to a cover of the PDE solution-operator class. The condition $\eps\leq\delta$ ensures that the approximation error is no larger than the covering scale.
    \begin{lemma}\label{lemma:covering number translation}
        Fix $\delta>0$ and assume that $\eps\leq \delta.$ Let $\covering_{\Sigma}$ and $\covering_{\pdeClass}$ denote the covering numbers of $\Sigma^{\textup{poly}}(\eps,\bar{\vartheta})$ and $\pdeClass^{\textup{poly}}(\bar{\vartheta})$ each with respect to $\|\cdot\|_{C(\sobolevBall,L^2(D))}.$ It holds that
        \begin{align*}
            \covering_{\pdeClass}(2\delta)\leq \covering_{\Sigma}(\delta).
        \end{align*}
        \begin{proof}
            By Theorem~\ref{thm: polynomial nonlinearity approximation theorem}, for any $\mathcal{S}\in \pdeClass^{\textup{poly}}(\bar{\vartheta})$ there exists $\Psi\in\Sigma^{\textup{poly}}(\eps,\bar{\vartheta})$ such that
            \begin{align*}
                \sup_{u\in \inputBall_{\smoothness}}\|\Psi(u)-\PDE(u)\|_{L^2}\leq \eps.
            \end{align*}
            Let $\Psi_1,\ldots,\Psi_{\covering_{\Sigma}}$ be a $\delta$-cover of $\Sigma^{\textup{poly}}(\eps,\bar{\vartheta}).$ Then, 
            \begin{align*}
                \min_{1\leq j \leq \covering_{\Sigma}}\sup_{u\in \inputBall_{\smoothness}}\|\PDE(u)-\Psi_j(u)\|_{L^2} 
                &\leq \sup_{u\in \inputBall_{\smoothness}}\|\PDE(u)-\Psi(u)\|_{L^2}+\min_{1\leq j \leq \covering_{\Sigma}}\sup_{u\in \inputBall_{\smoothness}}\|\Psi(u)-\Psi_j(u)\|_{L^2}\\
                &\leq \eps+\delta\leq 2\delta,
\end{align*}
completing the proof.
        \end{proof}
    \end{lemma}
    The proof of Theorem~\ref{thm: learning} now follows in a straightforward manner from the above lemmas.
    \begin{proof}[Proof of Theorem~\ref{thm: learning}] 
        Let $\eps=r\samples^{-\frac{1}{6+2\frac{d}{a}}}$ for some appropriate $r$ depending on $\outputBall$ and the constant in Corollary~\ref{corollary: rate of metric entropy}, and take $\delta=\eps$ and $t=\samples^{-1}$. This choice of $\eps$ along with Corollary~\ref{corollary: rate of metric entropy} and Lemma~\ref{lemma:covering number translation} guarantees that
        \begin{align*}
\samples= r \eps^{-6-2\frac{d}{a}} \geq     r \eps^{-6-\frac{d}{a}}\log(\eps^{-1})^3\geq  r\delta^{-2}\log\bigl(t^{-1}\covering_{\Sigma}(\delta)\covering_{\pdeClass}(2\delta)\bigr), \end{align*}
provided that $\samples$ is large enough such that $\eps^{-\frac{d}{\smoothness}}\geq \log(\eps^{-1})^3.$
Consequently, we can apply Proposition~\ref{prop:risk bound}, which implies that with probability at least $1 - 1/\samples$ 
\begin{equation*}
    \sup_{\PDE\in \pdeClass}\risk(\PDE,\hat{\Psi}_{\PDE})\leq r\samples^{-\frac{1}{6+\frac{2d}{a}}}.  \hfill \qedhere 
\end{equation*}
    \end{proof}

\subsection{Auxiliary Lemmas}\label{ssec:auxiliary lemmas}
In the proof of Theorem~\ref{thm: polynomial nonlinearity approximation theorem} we use the following classical Gronwall lemma. 
\begin{lemma}\label{lemma:grownall lemma 1}
     Let $\{y_j\}_{j=0}^{\infty}$ be a positive sequence satisfying
    \begin{align*}
        y_j\leq \alpha y_{j-1}+\beta \quad \forall j\geq 1,
    \end{align*}
for $\alpha>0$ with $\alpha\neq 1$. Then, for any integer $J\geq 0,$ it holds that
\begin{align*}
    y_J\leq\frac{\beta}{1-\alpha}\Bigl(1-\alpha^J\Bigr) +y_0\alpha^J.
\end{align*}
\end{lemma}

In the proof of Theorem~\ref{thm: learning} in Appendix~\ref{ssec: polynomail learning proof}, we use the following lemmas. The first of these lemmas relies on the boundedness of any $\PDE\in \pdeClass^{\textup{poly}}(\bar{\vartheta})$ and Bernstein's inequality to obtain a concentration inequality for the excess risk.
\begin{lemma}\label{lemma: risk deviation lemma}
    Let $\PDE\in \pdeClass^{\textup{poly}}(\bar{\vartheta})$, and let $\Psi\in \Sigma^{\textup{poly}}(\eps,\bar{\vartheta}).$ If $u^{(1)},\hdots ,u^{(M)} \iid \mu$, then for any $\zeta,\alpha>0$ and $M\in \N$ it holds that
    \begin{align*}
        \prob \Biggl[\frac{\risk(\Psi,\PDE)-\hat{\risk}(\Psi,\PDE)}{\risk(\Psi,\PDE)+\zeta}\geq \alpha\Biggr]\leq \exp\biggl(-\frac{\alpha^2M\zeta}{9\outputBall^2}\biggr).
    \end{align*}
    \begin{proof}
        For $\PDE\in \pdeClass^{\textup{poly}}(\bar{\vartheta})$ and $\Psi\in \Sigma^{\textup{poly}}(\eps,\bar{\vartheta})$ it holds that
        \begin{align*}
            \sup_{u\in \sobolevBall}\|\Psi(u)-\PDE(u)\|_{L^2}\leq \sup_{u\in \sobolevBall}\|\Psi(u)\|_{L^2}+\sup_{u\in \sobolevBall}\|\PDE(u)\|_{L^2}\leq 3V.
        \end{align*}
        We define i.i.d. random variables $y^{(\sample)}=\|\Psi(u^{(\sample)})-\PDE(u^{(\sample)})\|_{L^2}^2$,  which satisfy $0\leq y^{(\sample)}\leq 9\outputBall^2$ and $\E[y^{(\sample)}]=\risk(\Psi,\PDE).$ Bernstein's inequality \cite[Equation 2.23]{wainwright2019high} yields that, for any $t>0,$
        \begin{align*}
            \prob\Bigl[\risk(\Psi,\PDE)-\hat{\risk}(\Psi,\PDE)\geq t \Bigr]&=\prob \left[\E[y^{(\sample)}]-\frac{1}{\samples}\sum_{\sample=1}^\samples y^{(\sample)}\geq t \right]\\
            & \leq \exp \Biggl(-\frac{t^2\samples}{\frac{2}{\samples}\sum_{\sample=1}^\samples \E[(y^{(\sample)})^2]}\Biggr).
        \end{align*}
Since $0\leq y^{(\sample)}\leq 9\outputBall^2$, we have
$$
\E\bigl[(y^{(\sample)})^2\bigr]
\leq
9\outputBall^2\E[y^{(\sample)}]
=
9\outputBall^2\risk(\Psi,\PDE).
$$
Therefore,
$$
\frac{1}{\samples}
\sum_{\sample=1}^{\samples}
\E\bigl[(y^{(\sample)})^2\bigr]
\leq
9\outputBall^2\risk(\Psi,\PDE).
$$
Substituting this into Bernstein's inequality gives
$$
\prob\Bigl[
\risk(\Psi,\PDE)-\hat{\risk}(\Psi,\PDE)\geq t
\Bigr]
\leq
\exp \Biggl(
-\frac{t^2\samples}
{18\outputBall^2\risk(\Psi,\PDE)}
\Biggr).
$$
\nc
            Taking $t=\alpha \bigl(\risk(\Psi,\PDE)+\zeta\bigr)$ for any $\zeta,\alpha>0$, and using the fact that $\bigl(\risk(\Psi,\PDE)+\zeta\bigr)^2\geq 2\zeta\risk(\Psi,\PDE),$ we have
            \begin{align*}
                \prob \Biggl[\frac{\risk(\Psi,\PDE)-\hat{\risk}(\Psi,\PDE)}{\risk(\Psi,\PDE)+\zeta}\geq \alpha\Biggr] & \leq \exp \biggl(-\frac{\alpha^2 \samples(\risk(\Psi,\PDE)+\zeta)^2}{18V^2\risk(\Psi,\PDE)}\biggr)\\
                & \leq \exp \biggl(-\frac{\alpha^2 \samples\zeta}{9V^2}\biggr),
            \end{align*}
            as claimed.
    \end{proof}
\end{lemma}
The following lemma gives a Lipschitz bound for the risk functionals.
\begin{lemma}\label{lemma: risk lipschitz}
    For any $\Psi,\Psi'\in \Sigma^{\textup{poly}}(\eps,\bar{\vartheta})$ and $\PDE,\PDE'\in \pdeClass$, and any samples $u^{(1)},\hdots,u^{(\samples)}\in \mathsf{U}_{\smoothness}$, we have
    \begin{align*}
        \bigl|\hat{\risk}(\Psi,\PDE)-\hat{\risk}(\Psi',\PDE')\bigr|\leq 6\outputBall\Bigl(\sup_{u\in\sobolevBall}\|\Psi(u)-\Psi'(u)\|_{L^2}+\sup_{u\in\sobolevBall}\|\PDE(u)-\PDE'(u)\|_{L^2}\Bigr),
    \end{align*}
    and
    \begin{align*}
        \bigl|\risk(\Psi,\PDE)-\risk(\Psi',\PDE')\bigr|\leq 6\outputBall\Bigl(\sup_{u\in\sobolevBall}\|\Psi(u)-\Psi'(u)\|_{L^2}+\sup_{u\in\sobolevBall}\|\PDE(u)-\PDE'(u)\|_{L^2}\Bigr).
    \end{align*}
    \begin{proof}
        The proof proceeds exactly as  \cite[Lemma E.8]{kovachki2024data}, but noting that in our setting, $\sup_{u\in \sobolevBall}\|\Psi(u)-\PDE(u)\|_{L^2}\leq 3\outputBall$ instead of $\sup_{u\in \sobolevBall}\|\Psi(u)-\PDE(u)\|_{L^2}\leq 3$.
    \end{proof}
\end{lemma}
\begin{lemma}[\cite{kovachki2024data} Lemma E.9]\label{lemma: f lipschitz}
    For $\chi>0$, define for $a,b>0$
    \begin{align*}
        f_\chi(a,b):=\frac{(a-b)_+}{a+\chi}.
    \end{align*}
    Then $f_{\chi}$ is globally Lipschitz continuous and 
    \begin{align*}
        \sup_{a,b}|\partial_af_{\chi}(a,b)|\leq \frac{1}{\chi}, \quad \sup_{a,b}|\partial_bf_{\chi}(a,b)|\leq \frac{1}{\chi},
    \end{align*}
    almost everywhere.
\end{lemma}

\section{Supplementary Materials Section~\ref{sec:smoothnonlinearity}}
\subsection{Proof of Approximation Bound}\label{ssec:smooth approximation proof}
\begin{proof}[Proof of Theorem~\ref{thm: smooth nonlinearity approximation theorem}]
         First note that for any $\pdeClass^{\text{smooth}}(\bar{\vartheta})$, $\Psi\in \Sigma^{\textup{smooth}}(\eps,\bar{\vartheta})$ and $u \in \mathsf{U}_a,$ we have 
         \begin{align*}
             \|\Psi(u)-\PDE(u)\|_{L^2}&\leq \|\Psi(\mathcal{I}_Nu) - \spectralStep^j(\mathcal{I}_Nu)\|_{L^2} + \|\spectralStep^j(\mathcal{I}_Nu) - \PDE(u)\|_{L^2}\\
             &\leq \|\Psi(\mathcal{I}_Nu) - \spectralStep^j(\mathcal{I}_Nu)\|_{L^2} + r_1\exp(r_2T)\Bigl(N^{-a}+\timestep\Bigr).
         \end{align*}
         The second inequality follows from Condition~\ref{conditionsec4} \ref{item:spectral method II}. Taking $N=\lceil(\frac{ 4  r_1}{\eps})^{1/a}\exp(r_2 T/a)\rceil$ and $\timestep>0$ to be as large as possible such that $\timestep\leq\min \Bigl\{\frac{\eps}{ 4  r_1}\exp(-r_2T),1/4 \Bigr\}$ and $j\timestep=T$, we have that
         \begin{align*}
             \|\Psi(u)-\PDE(u)\|_{L^2}&\leq \|\Psi(\mathcal{I}_Nu) - \spectralStep^j(\mathcal{I}_Nu)\|_{L^2} +  \frac{\eps} {2},
         \end{align*}
         and we see that it remains to show that there exists $\Psi\in \Sigma^{\textup{smooth}}(\eps,\bar{\vartheta})$ that can approximate $\spectralStep^j$ to accuracy $\eps/2.$ To that end, we first construct an FNO that can approximate a single step of the spectral method, $\spectralStep.$ Recall that for any $u_N\in L^2_N(\T^d,\R^{\doutput})$
         \begin{align*}
             \spectralStep(u_N)=\linearOp  \Bigl(\bigl(\frac{1}{\timestep}I+\gamma \mathcal{A}^{\power}\bigr)u_N- \mathcal{D}_s\bigl(\nonlinearity(u_N)\bigr)+\forcing_N\Bigr),
         \end{align*}
         for some $\linearOp\in \linearOpClass(\dissipation,\gamma,\power,\inverseBound),$ $\nonlinearity\in \nonlinearClass(\alpha,\coefficientBound),$ and $\forcing_N\in \forcingClass(\forcingBound)$. 

In Lemmas~\ref{lemma:B FNO implementation}, \ref{lemma:affine part FNO implementation}, \ref{lemma:Ds FNO implementation} and \ref{lemma:smooth part single step},  we have constructed FNOs to implement or approximate each part of this mapping. In particular, for any $\linearOp\in \linearOpClass(\dissipation,\gamma,\power,\inverseBound)$, we construct an FNO  $\Psi_{\linearOp}:L^2_{\mathscr{C}(N)}(\T^d,\R^{\doutput})\to L^2_{N}(\T^d,\R^{\doutput})$ such that $\Psi_{\linearOp}(u)=\linearOp u$ for an integer $\mathscr{C}(N)$ that we will specify shortly. This FNO has size 
         \begin{align*}
        \mathscr{D}(\Psi_{\linearOp})&=1,\\
        \mathscr{W}(\Psi_{\linearOp})&=2\doutput,\\
        \mathscr{S}(\Psi_{\linearOp})&=1,\\
        \mathscr{C}(\Psi_{\linearOp})&=\mathscr{C}(N).
    \end{align*}
         Similarly, for any positive semi-definite $\mathcal{A}:L^2_{N}(\T^d,\R^{\doutput})\to L^2_{N}(\T^d,\R^{\doutput})$ diagonalizable in Fourier space satisfying the scaling \eqref{eq:A scaling assumption}, and any $f_N\in \forcingClass(\forcingBound)$, Lemma~\ref{lemma:affine part FNO implementation} constructs an FNO $\Psi_{\textup{affine}}:L^2_{N}(\T^d,\R^{\doutput})\to L^2_{N}(\T^d,\R^{\doutput})\subset L^2_{\mathscr{C}(N)}(\T^d,\R^{\doutput})$ such that $ \Psi_{\textup{affine}}(u)=(\frac{1}{\timestep}I+\gamma \mathcal{A}^{\power})u+f_N,$ satisfying 
         \begin{align*}
                \mathscr{D}(\Psi_{\textup{affine}})&=1,\\
        \mathscr{W}(\Psi_{\textup{affine}})&=2\doutput,\\
        \mathscr{S}(\Psi_{\textup{affine}})&=\max \bigl\{\timestep^{-1},r\gamma N^{\dissipation \power},\forcingBound \bigr\},\\
        \mathscr{C}(\Psi_{\textup{affine}})&=N.
    \end{align*}
Finally, for any $\nonlinearity\in \nonlinearClass(\alpha,\coefficientBound)$, Lemma~\ref{lemma:smooth part single step} constructs an FNO $\Psi_{\nonlinearity}:L^2_{N}(\T^d,\R^{\doutput})\to L^2_{\mathscr{C}(N)}(\T^d,\R^{\doutput})$ that can approximate ${\mathcal{P}_N\mathcal{D}_s\nonlinearity(\cdot)}$ to accuracy $\eps_{\nonlinearity}$ in the $\mathcal{A}^{-1/2}$ weighted $L^2$-norm. Bounding the error in this weighted norm will be most natural for controlling the propagation of the error over the $j$ compositions of this FNO required to approximate the PDE solution. Precisely, we show that for any $u_N\in L^2_{N}(\T^d,\R^{\doutput})$ with $\|u_N\|_{L^2}\leq \inputBall'$,
         \begin{align*}
    \left\|\mathcal{A}^{-1/2}\Bigl(\Psi_{\nonlinearity}(u_N)-\mathcal{P}_N\mathcal{D}_s\bigl(\nonlinearity(u_N)\bigr)\Bigr)\right\|_{L^2}\leq \eps_{\nonlinearity},
\end{align*}
where $\int_{\T^d}\Psi_{\nonlinearity}(u_N)(x) \, dx=0$ such that $\mathcal{A}^{-1/2}$ is defined. This FNO has a size that depends on $N$ and the desired accuracy:
\begin{align*}
        \mathscr{D}(\Psi_{\nonlinearity})&=r_{\alpha,\doutput,s,\dissipation}\log(\eps_{\nonlinearity}^{-1}N), \\
         \mathscr{W}(\Psi_{\nonlinearity})&=r_{\alpha,\doutput,\coefficientBound}\inputBall'^{\doutput}N^{\frac{\doutput d}{2}}N^{\frac{\doutput\max\{d+s-\dissipation/2,0\}}{\alpha}}{\log(N)^{\frac{\doutput}{\alpha}}}\eps_{\nonlinearity}^{-\frac{\doutput}{\alpha}},\\
        \mathscr{S}(\Psi_{\nonlinearity})&=\max \Bigl\{r_{\alpha,\doutput,\coefficientBound}\inputBall^{\alpha-1}N^{\frac{\max\{d+s-\dissipation/2,0\}}{\alpha}}{\log(N)^{\frac{1}{\alpha}}}\eps_{\nonlinearity}^{-\frac{1}{\alpha}},\eps_{\nonlinearity}^{-\frac{s}{\alpha}}N^{\frac{ds}{2}+s}\inputBall'^{\alpha s} \Bigr\},\\
        \mathscr{C}(\Psi_{\nonlinearity})&=\mathscr{C}(N):=\max\Bigl\{N,r_{d,\doutput,\alpha,C}\eps_{\nonlinearity}^{-1/\alpha}\inputBall'N^{1+\frac{d}{2}} \Bigr\}.
    \end{align*}
    To approximate $\spectralStep,$ we consider the composition of $\Psi_{\linearOp}$ with the sum of $\Psi_{\nonlinearity}$ and $\Psi_{\textup{affine}}$,
         \begin{align*}
             \Psi_{\spectralStep}(u_N):=\Psi_{\linearOp}\bigl(\Psi_{\textup{affine}}(u_N)-\Psi_{\nonlinearity}(u_N)\bigr),
         \end{align*}
         which by the FNO linear combination Lemma~\ref{lemma:FNO linear combination} and the FNO composition Lemma~\ref{lemma:FNO Composition}, can be implemented by an FNO $\Psi_{\spectralStep}:L^2_N(\T^d,\R^{\doutput}) \to L^2_N(\T^d,\R^{\doutput})$ of size 
       \begin{align*}
        \mathscr{D}(\Psi_{\spectralStep})&=r_{\alpha,\doutput,s,\dissipation}\log(\eps_{\nonlinearity}^{-1}N),\\
        \mathscr{W}(\Psi_{\spectralStep})&=r_{\alpha,\doutput,\coefficientBound}\inputBall'^{\doutput}N^{\frac{\doutput d}{2}}N^{\frac{\doutput\max\{d+s-\dissipation/2,0\}}{\alpha}}{\log(N)^{\frac{\doutput}{\alpha}}}\eps_{\nonlinearity}^{-\frac{\doutput}{\alpha}},\\
        \mathscr{S}(\Psi_{\spectralStep})&=\max\Bigl\{r_{\alpha,\doutput,\coefficientBound}\inputBall'^{\alpha-1}N^{\frac{\max\{d+s-\dissipation/2,0\}}{\alpha}}{\log(N)^{\frac{1}{\alpha}}}\eps_{\nonlinearity}^{-\frac{1}{\alpha}},\eps_{\nonlinearity}^{-\frac{s}{\alpha}}N^{\frac{ds}{2}+s}\inputBall'^{\alpha s} \Bigr\},\\
        \mathscr{C}(\Psi_{\spectralStep})&=\max \Bigl\{N,r_{d,\doutput,\alpha,C}\eps_{\nonlinearity}^{-1/\alpha}\inputBall'^{d/\alpha}N^{1+\frac{d}{2}} \Bigr\}.
    \end{align*}  
    To approximate $\spectralStep^j$, we compose $\Psi_{\spectralStep}$ with itself $j$ times. By the FNO composition Lemma~\ref{lemma:FNO Composition}, there exists an FNO $\Psi:L^2_N(\T^d,\R^{\doutput}) \to L^2_N(\T^d,\R^{\doutput})$ implementing 
         \begin{align*}
             \Psi(u_N)=\Psi_{\spectralStep}^j(u_N)=\underbrace{\Psi_{\spectralStep}\circ \Psi_{\spectralStep}\circ \cdots \circ\Psi_{\spectralStep}}_{j \text{ times}}(u_N),
         \end{align*}
         with 
\begin{align*}
        \mathscr{D}(\Psi)&=r_{\alpha,\doutput,s,\dissipation}j\log(\eps_{\nonlinearity}^{-1}N),\\
        \mathscr{W}(\Psi)&=r_{\alpha,\doutput,\coefficientBound}\inputBall'^{\doutput}N^{\frac{\doutput d}{2}}N^{\frac{\doutput\max\{d+s-\dissipation/2,0\}}{\alpha}}{\log(N)^{\frac{\doutput}{\alpha}}}\eps_{\nonlinearity}^{-\frac{\doutput}{\alpha}},\\
        \mathscr{S}(\Psi)&=\max \Bigl\{r_{\alpha,\doutput,\coefficientBound}\inputBall'^{\alpha-1}N^{\frac{\max\{d+s-\dissipation/2,0\}}{\alpha}}{\log(N)^{\frac{1}{\alpha}}}\eps_{\nonlinearity}^{-\frac{1}{\alpha}},\eps_{\nonlinearity}^{-\frac{s}{\alpha}}N^{\frac{ds}{2}+s}\inputBall'^{\alpha s} \Bigr\},\\
        \mathscr{C}(\Psi)&=\max \Bigl\{N,r_{d,\doutput,\alpha,C}\eps_{\nonlinearity}^{-1/\alpha}\inputBall'^{d/\alpha}N^{1+\frac{d}{2}} \Bigr\}.
    \end{align*}  
In the remainder of this proof, we will argue that by choosing $\eps_{\nonlinearity}$ to be small enough, we can guarantee that $\|\Psi(\mathcal{I}_Nu) - \spectralStep^j(\mathcal{I}_Nu)\|_{L^2}\leq\frac{\eps}{2}$, from which it follows that $\Psi$ can approximate $\PDE$ to accuracy $\eps.$
We have constructed $\Psi$ such that
\begin{align*}
    \Psi(u_N)=\linearOp\Bigl((\frac{1}{\timestep}I+\gamma \mathcal{A}^{\power})\Psi_{\spectralStep}^{j-1}(u_N)-\Psi_{\nonlinearity}(\Psi_{\spectralStep}^{j-1}(u_N))+f_N\Bigr).
\end{align*}
Note that this implies that $\int_{\T^d}\Psi_{\spectralStep}^{j}(u_N) \, dx=0$ for all $j.$
   Then, plugging in $\linearOp= \Bigl(\frac{1}{\timestep}I+\gamma \mathcal{A}^{\power} +
    \mathcal{A}\Bigr)^{-1} \mathcal{P}_N$ and rearranging, we get that 
\begin{align}\label{eq:smooth FNO iterates}
    \frac{\Psi(u_N)-\Psi_{\spectralStep}^{j-1}(u_N)}{\timestep}+ \mathcal{A}\Psi(u_N)+\gamma \mathcal{A}^{\power}\bigl(\Psi(u_N)-\Psi_{\spectralStep}^{j-1}(u_N)\bigr)+\Psi_{\nonlinearity}\bigl(\Psi_{\spectralStep}^{j-1}(u_N)\bigr) = f_N.
\end{align}
Here we have used that $\text{Range}(\Psi_{\nonlinearity})\subseteq \text{Range}(\mathcal{P}_N).$

         Recall that the single-step spectral method $\spectralStep$ is similarly defined such that the iterates satisfy
\begin{align}\label{eq:smooth spectral iterates}
    \frac{\spectralStep^j(u_N)-\spectralStep^{j-1}(u_N)}{\timestep}+  \mathcal{A}\spectralStep^{j}(u_N)+\gamma \mathcal{A}^{\power}\bigl(\spectralStep^j(u_N)-\spectralStep^{j-1}(u_N)\bigr)+\mathcal{P}_N \mathcal{D}_s\bigl(\nonlinearity(\spectralStep^{j-1}(u_N))\bigr) = f_N.
\end{align}
For notational brevity, we write $u^{j}=\Psi_{\spectralStep}^{j}(u_N)$ and $v^{j}=\spectralStep^j(u_N)$ for $j\in \N$, and denote the error between the FNO iterate and the spectral method iterate as $e^j=u^j-v^j=\Psi_{\spectralStep}^{j}(u_N)-\spectralStep^j(u_N)$. Subtracting \eqref{eq:smooth spectral iterates} from \eqref{eq:smooth FNO iterates} yields 
\begin{align*}
    \frac{e^j-e^{j-1}}{\timestep}+ \mathcal{A} e^{j}+\gamma \mathcal{A}^{\power}(e^j-e^{j-1})+\Psi_{\nonlinearity}(u^{j-1})-\mathcal{P}_N\mathcal{D}_s\bigl(\nonlinearity(v^{j-1})\bigr)=0.
\end{align*}
Taking the inner product of this expression with $e^j$, we have
\begin{equation}\label{eq:error evolution smooth}
\begin{split}
    &\frac{1}{2\timestep}\bigl(\|e^j\|_{L^2}^2-\|e^{j-1}\|_{L^2}^2+\|e^j-e^{j-1}\|_{L^2}^2\bigr)+  \|\mathcal{A}^{1/2}e^j\|_{L^2}^2 
    \\&\hspace{2cm} +\dfrac{\gamma}{2}\left(\| \mathcal{A}^{\power/2}e^j\|_{L^2}^2-\|\mathcal{A}^{\power/2}e^{j-1}\|_{L^2}^2+\|\mathcal{A}^{\power/2}(e^j-e^{j-1})\|_{L^2}^2 \right)\\
    &\hspace{2cm}=
    \left\langle  \mathcal{P}_N \mathcal{D}_s\bigl(\nonlinearity(v^{j-1})\bigr)- \Psi_{\nonlinearity}(u^{j-1}), e^j \right\rangle
    :=\mathcal{E}.
    \end{split}
\end{equation}
We now pursue a bound on the inner product on the right-hand side. We have from Young's inequality that
\begin{align*}
|\mathcal{E}|
&\leq
\Bigl|
\Bigl\langle
    \mathcal{P}_N\mathcal{D}_s\bigl(\nonlinearity(v^{j-1})\bigr)
    -\mathcal{P}_N\mathcal{D}_s\bigl(\nonlinearity(u^{j-1})\bigr),
    e^j
\Bigr\rangle
\Bigr|  \\
& +
\Bigl|
\Bigl\langle
    \mathcal{P}_N\mathcal{D}_s\bigl(\nonlinearity(u^{j-1})\bigr)
    -\Psi_{\nonlinearity}(u^{j-1}),
    e^j
\Bigr\rangle
\Bigr| \\
& \leq\frac{1}{2}\Bigl\|\mathcal{A}^{-1/2}\Bigl(\mathcal{P}_N\D_{s}\nonlinearity(v^{j-1})- \mathcal{P}_N\mathcal{D}_s\nonlinearity(u^{j-1})\Bigr)\Bigr\|_{L^2}^2+\frac{1}{2}\|\mathcal{A}^{1/2}e^j\|_{L^2}^2\\
&  + \frac{1}{2}\Bigl \|\mathcal{A}^{-1/2}\Bigl(\mathcal{P}_N\mathcal{D}_s\bigl(\nonlinearity(u^{j-1})\bigr)- \Psi_{\nonlinearity}(u^{j-1})\Bigr)\Bigr \|_{L^2}^2+\frac{1}{2}\|\mathcal{A}^{1/2}e^j\|_{L^2}^2.
\end{align*}

     Since $s\leq \dissipation/2$, we have that $\|\mathcal{A}^{-1/2}\mathcal{D}_s\|_{op}\leq d\ndim\DcoefficientBound\inverseBound^{-\frac{1}{2}}$. This, combined with the Lipschitz bound encoded in the definition of the nonlinearity class $\nonlinearClass(\alpha,C_\nonlinearity)$, 
     yields that
     \begin{align*}
         |\mathcal{E}|\leq \frac{(d\ndim\DcoefficientBound \coefficientBound)^2}{2\inverseBound}\|e^{j-1}\|_{L^2}^2
         +\| \mathcal{A}^{1/2} \nc e^{j}\|_{L^2}^2
         +\frac{1}{2}\left\|\mathcal{A}^{-1/2}\Bigl(\mathcal{P}_N\mathcal{D}_s\bigl(\nonlinearity(u^{j-1})\bigr)- \Psi_{\nonlinearity}(u^{j-1})\Bigr)\right\|_{L^2}^2.
     \end{align*}
     Assuming that $\inputBall'$ is chosen to be large enough such that $\|u^j\|_{L^2}\leq \inputBall'$ (we will specify a specific choice of $\inputBall'$ shortly below), Lemma~\ref{lemma:smooth part single step}  yields
     \begin{align*}
         |\mathcal{E}|\leq \frac{(d\ndim\DcoefficientBound\coefficientBound)^2}{2\inverseBound}\|e^{j-1}\|_{L^2}^2+\| \mathcal{A}^{1/2}e^{j} \|_{L^2}^2+\frac{\eps_{\nonlinearity}^2}{2}.
     \end{align*}
     Plugging this bound into \eqref{eq:error evolution smooth}, we get that
   \begin{equation}
\begin{split}
    \frac{1}{2\timestep}\bigl(\|e^j\|_{L^2}^2-\|e^{j-1}\|_{L^2}^2\bigr)+  \dfrac{\gamma}{2}(\| \mathcal{A}^{\power/2}e^j\|_{L^2}^2-\|\mathcal{A}^{\power/2}e^{j-1}\|_{L^2}^2)\leq \frac{(d\ndim\DcoefficientBound\coefficientBound)^2}{2\inverseBound}\|e^{j-1}\|_{L^2}^2+\frac{\eps_{\nonlinearity}^2}{2},
    \end{split}
\end{equation}
which implies
{\small{
\begin{align*}
    \frac{1}{\timestep}\bigl(\|e^j\|_{L^2}^2+\timestep\gamma\|\mathcal{A}^{\power/2}e^j\|_{L^2}^2-\|e^{j-1}\|_{L^2}^2-\timestep\gamma\|\mathcal{A}^{\power/2}e^{j-1}\|_{L^2}^2\bigr)\leq \frac{(d\ndim\DcoefficientBound\coefficientBound)^2}{\inverseBound}\bigl(\|e^{j-1}\|_{L^2}^2+\timestep\gamma\|\mathcal{A}^{\eta/2} e^{j-1} \|_{L^2}^2\bigr)+\eps_{\nonlinearity}^2.
\end{align*}
}}
Applying Lemma~\ref{lemma:gronwall lemma 2} with $y_j=\|e^j\|_{L^2}^2+\timestep\gamma\|\mathcal{A}^{\power/2}e^j\|_{L^2}^2$, $\alpha_{j}=\frac{(d\ndim\DcoefficientBound\coefficientBound)^2}{\inverseBound}$, and $\beta_j=\eps_{\nonlinearity}^2,$ we have
\begin{align*}
    \|\Psi(u_N)-\spectralStep^j(u_N)\|_{L^2}=\|e^j\|_{L^2}\leq \exp \left(T\frac{(d\ndim\DcoefficientBound\coefficientBound)^2}{2\inverseBound}\right)\sqrt{T}\eps_{\nonlinearity}.
\end{align*}
Thus, taking $\eps_{\nonlinearity}=\frac{\eps}{2\sqrt{T}}\exp\bigl(-T\frac{(d\ndim\DcoefficientBound\coefficientBound)^2}{2\inverseBound}\bigr)$ and $\inputBall'= \max \bigl\{\inputBall,\outputBall \bigr\}+\eps$  yields the desired result.
\end{proof}

\subsection{Proof of Learning Bound}\label{sec:smooth learning proof}
     \begin{proof}[Proof of Theorem~\ref{thm: smooth learning bound}]
    Note that Lemmas~\ref{lemma:approximation estimation bound}-\ref{lemma:FNO metric entropy} hold exactly as written with $\Sigma^{\textup{smooth}}(\eps,\bar{\vartheta})$ and $\pdeClass^{\text{smooth}}(\bar{\vartheta})$ in place of $\Sigma^{\textup{poly}}(\eps,\bar{\vartheta})$ and $\pdeClass(\bar{\vartheta}),$ respectively. Corollary~\ref{corollary: smooth rate of metric entropy} is the direct analogue of Corollary~\ref{corollary: rate of metric entropy}. Lemma~\ref{lemma:covering number translation} also follows exactly as written with $\Sigma^{\textup{smooth}}(\eps,\bar{\vartheta})$ and $\pdeClass^{\text{smooth}}(\bar{\vartheta})$ in place of $\Sigma^{\textup{poly}}(\eps,\bar{\vartheta})$ and $\pdeClass^{\text{poly}}(\bar{\vartheta}).$ Then, fixing
    $$
    \eps=r\samples^{\frac{-1}{5+\rho}},
    $$
    where 
    $$ \rho = \frac{d}{\alpha}+\frac{2d+d^2}{\smoothness}+\frac{\doutput d}{\smoothness}+\frac{2\doutput\max\{d+s-\dissipation/2,0\}}{ \alpha \smoothness}+
\frac{2\doutput}{\alpha},$$ and letting $\delta=\eps$ and $t=M^{-1}$, Corollary~\ref{corollary: smooth rate of metric entropy} and Lemma~\ref{lemma:covering number translation} together yield that
    \begin{align*}
        \samples &= r\eps^{-(5+\rho)}\geq r\eps^{-(4+\rho)}\log(\eps^{-1})^{{3+\frac{\doutput}{\alpha}}} \\
        &\geq r\delta^{-2}\log\bigl(t^{-1}\covering_{\Sigma(\eps)}(\delta)\covering_{\pdeClass}(2\delta)\bigr),
    \end{align*}
    provided $\samples$ is large enough such that ${\eps^{-1}\geq \log(\eps^{-1})^{3+\frac{\doutput}{\alpha}}}.$
    Consequently, we can apply Proposition~\ref{prop:risk bound}, which yields the desired result.
\end{proof}

    \subsection{Auxiliary Lemmas}
    In the proof of the approximation bound, we used the following  discrete Gronwall inequality \cite{clark1987short}.
\begin{lemma}\label{lemma:gronwall lemma 2}
For $\timestep >0$, $y_j\geq 0$, $\alpha_j\geq 0$, and $\beta_j\geq 0$ for $j=1,2,3\hdots$. If
\begin{align*}
    \frac{y_j-y_{j-1}}{\timestep}\leq \alpha_{j-1} y_{j-1}+\beta_{j-1}, \quad \forall j\geq 1,
\end{align*}
then, for any integer $J>0,$ 
\begin{align*}
    y_J\leq \exp \left(\timestep\sum_{j=0}^{J-1}\alpha_j\right)\left( y_0 \nc +\timestep \sum_{j=0}^{J-1}\beta_j\right).
\end{align*}
\end{lemma}


    \begin{corollary}\label{corollary: smooth rate of metric entropy}
Fix $\eps>0$ and let $\covering_{\Sigma^{\textup{smooth}}}(\delta)$ denote the $\delta$-covering number of $\Sigma^{\textup{smooth}}(\eps,\bar{\vartheta})$ with respect to $\|\cdot\|_{C(\inputBall{\smoothness},L^2(D))}$. For any $\delta\in (0,1)$, it holds that
$$
\log \covering_{\Sigma^{\textup{smooth}}}(\delta)
\leq
r
\eps^{-2-\rho}
\log(\eps^{-1})^{\frac{\doutput}{\alpha}}
\log(\eps^{-1}\delta^{-1})^3,
$$
where
$$
\rho
:=
\frac{d}{\alpha}
+
\frac{2d+d^2}{\smoothness}
+
\frac{\doutput d}{\smoothness}
+
\frac{2\doutput\max\{d+s-\dissipation/2,0\}}{\alpha\smoothness}
+
\frac{2\doutput}{\alpha}.
$$
Here $r$ is a constant increasing in $d,\doutput,\gamma,\power,\inverseBound^{-1},s,\DcoefficientBound,\coefficientBound,\forcingBound,\outputBall,$ and $T$.
\end{corollary}

\subsection{Logarithmic Cahn--Hilliard Results}
\begin{theorem}\label{thm: FH stability}
    Consider the implicit-explicit scheme in \eqref{eq:FH scheme}. Let $u_0\in L^2(\T^2,\R)$ with $\|u_0\|_{H^5}\leq \inputBall$ and $\|u_0\|_{L^\infty}\leq 1-\delta_0$. Then, there exists $\timestep_0(\inputBall,\delta_0,\nu,\theta,\theta_c)$ and $N_0(\inputBall,\delta_0,\nu,\theta,\theta_c)$ such that for any $\timestep\leq \timestep_0$ and $N\geq N_0$, the following hold for the scheme \eqref{eq:FH scheme}:
    \begin{enumerate}
        \item For the free energy functional 
        \begin{align}
            \mathcal{E}(u) = \int_{\mathbb{T}^2} \Bigl( \frac{\nu}{2} |\nabla u|^2 - \frac{\theta_c}{2} u^2 + \potential(u) \Bigr) \, dx,
        \end{align}
        the iterates are energy stable: for all $j\geq 0,$
        \begin{align}\label{eq:energystabilitystatement}
            \mathcal{E}(u_N^{j+1})\leq \mathcal{E}(u_N^j).
        \end{align}
        \item The iterates $u_N^j$ are well defined for all $j\geq 1$ and there exists $\outputBall(\inputBall,\delta_0,\nu,\theta,\theta_c)$ such that
        \begin{align}
            \sup_{j\geq 1}\|u_N^j\|_{H^5}\leq \outputBall.
        \end{align}
        \item There exists $\delta_1(\inputBall,\delta_0,\nu,\theta,\theta_c)\in (0,1)$ such that
        \begin{align*}
            \sup_{j\geq 1}\|u_N^j\|_{L^{\infty}}\leq 1-\delta_1.
        \end{align*}
    \end{enumerate}
    \begin{proof}
            This theorem is an analogue of the stability result \cite[Theorem 1.1]{li2021stability} for the scheme \eqref{eq:FH scheme}. Note that \cite[Theorem 1.1]{li2021stability} holds for the scheme
            \begin{align*}
                \frac{u^{j+1}-u^j}{\tau} = -\nu \Delta^2 u^{j+1} - \theta_c \Delta u^{j+1} + \Delta  \left(\potentialNonlin(u_N^j)\right),
            \end{align*}
            which, in contrast to \eqref{eq:FH scheme}, is spatially continuous and treats the $\theta_c \Delta u$ term implicitly. As explained in \cite[Remark 1.3]{li2021stability} and \cite[Remark 1.4]{li2021stability}, their analysis can be extended to the case of \eqref{eq:FH scheme} with minor modifications. For completeness, we summarize these modifications here. The overall proof structure proceeds in the same way, making the inductive hypothesis that 
            \begin{align*}
    \biggl\|\frac{\theta}{2}\log\frac{1+u_N^j}{1-u_N^j} \biggr\|_{H^2}\leq A_0<\infty, \quad \|u_N^j\|_{H^5}\leq A_1<\infty
\end{align*}
for constants $A_0$ and $A_1$ that are made clear in \cite{li2021stability}. The argument in \cite[Section 2.1]{li2021stability} to obtain an energy estimate of $u^{j+1}$ goes through nearly exactly as written to obtain the analogous energy estimate of $u_N^{j+1}$. The only difference is that treating the  $\theta_c \Delta u$ term explicitly results in a sign change for the $\frac{\theta_c}{2}\|u^{j+1}-u^j\|_{L^2}^2$ term on the right-hand side. This difference actually results in a milder step-size requirement than for the implicit treatment. Then, the argument in \cite[Section 2.2]{li2021stability} to obtain a preliminary estimate on the chemical potential
$$
K^{j+1}=-\nu \Delta u^{j+1}-\theta_c u^{j+1}+\frac{\theta}{2}\log \frac{1+u^{j+1}}{1-u^{j+1}}
$$ is straightforward to carry out for the iterates given by \eqref{eq:FH scheme} to obtain an analogous estimate on the spatially-discretized chemical potential $$
K_N^{j+1}=-\nu \Delta u_N^{j+1}-\theta_c u_N^{j+1}+\mathcal{P}_N\frac{\theta}{2}\log \frac{1+u_N^{j+1}}{1-u_N^{j+1}}.
$$ 
Here, there are additional lower-order terms that arise from treating $\theta_c \Delta u$ explicitly instead of implicitly; however, these terms are straightforward to control via the energy estimate of $u_N^{j+1}.$ The long-time estimates on $\nabla K^{j+1}$ derived in \cite[Section 2.3]{li2021stability} can also be shown to hold for $\nabla K_N^{j+1}$ by the same argument. Again there are lower-order terms arising from treating $\theta_c \Delta u$ explicitly, but these can similarly be controlled by the energy estimates. To carry out the analysis in \cite[Sections 2.4-2.6]{li2021stability}, it is first necessary to obtain time-uniform bounds on the un-projected chemical potential $$\tilde{K}_N^{j+1}=-\nu \Delta u_N^{j+1}-\theta_c u_N^{j+1}+\frac{\theta}{2}\log \frac{1+u_N^{j+1}}{1-u_N^{j+1}}$$
for the iterates given by \eqref{eq:FH scheme}. Using the inductive hypothesis that $\|u_N^{j}\|_{H^5}\leq A_1,$ it is straightforward to obtain a time-uniform bound on $\nabla \tilde{K}_N^{j+1}$ from the bound on $\nabla {K}_N^{j+1}$ provided $N$ is taken to be large enough relative to $A_1.$ Given this estimate, the arguments in \cite[Sections 2.4-2.6]{li2021stability} hold essentially as written with $\tilde{K}_N^{j+1}$ in place of ${K}^{j+1}$. This distinction is important since in order to obtain the higher-order Sobolev bounds on $u_N^{j+1}$ and $\log \frac{1+u_N^{j+1}}{1-u_N^{j+1}}$ it is necessary to use the fact that $u_N^{j+1}$ and $\log \frac{1+u_N^{j+1}}{1-u_N^{j+1}}$ share the same sign everywhere, which is not true for $u_N^{j+1}$ and $\mathcal{P}_N\log \frac{1+u_N^{j+1}}{1-u_N^{j+1}}$. Finally, the bootstrap argument in \cite[Section 2.7]{li2021stability} holds nearly exactly as written when denoting \begin{align*}
    u_N^{j+1}&=\frac{1}{1+\timestep \nu \Delta^2}u_N^{j}+\frac{\timestep \Delta}{1+\timestep \nu \Delta^2}\mathcal{P}_N \Bigl(\frac{\theta}{2}\log \frac{1+u_N^j}{1-u_N^j}-\theta_c u_N^j\Bigr)\\
    &=: T_0 u_N^{j}+T_1
    \Bigl(\mathcal{P}_N\frac{\theta}{2}\log \frac{1+u_N^j}{1-u_N^j}-\theta_cu_N^{j}\Bigr)\\
    &=: T_0u_N^j+T_1f^j 
\end{align*}
in place of the definition given in \cite[Equation 2.30]{li2021stability}, which completes the proof.
\end{proof}

\end{theorem}

\begin{theorem}\label{thm:FH accuracy}
    Let $\nu>0$, and $u_0\in L^2(\T^2,\R)$ with $\|u_0\|_{H^5}\leq \inputBall$ and $\|u_0\|_{L^{\infty}}\leq 1-\delta_0$ for some $\delta_0\in (0,1).$ Let $u(T)$ be the exact PDE solution to \eqref{eq:FH PDE} corresponding to initial data $u_0.$ Let $\timestep\leq \min \Bigl\{\timestep_0,\frac{\nu}{4\theta_c^2} \Bigr\}$ where $\timestep_0$ is the same as in Theorem~\ref{thm: FH stability}, and suppose that $\timestep j=T$. Let $u_N^j$ be the numerical solution given by \eqref{eq:FH scheme}. Then, 
    \begin{align}\label{eq:FH accuracy bound}
        \|u(T)-u_N^j\|_{L^2}\leq C_1 e^{C_2T} \bigl(\timestep +N^{-5}\bigr),
    \end{align}
    where $C_1,C_2>0$ depend on $\inputBall,\delta_0,\nu, \theta, \theta_c.$
    \begin{proof}
        Consider the scheme given in \eqref{eq:FH scheme}
$$
\frac{u_N^{j+1}-u_N^j}{\tau} = -\nu \Delta^2 u_N^{j+1} - \theta_c \Delta u_N^j + \Delta \mathcal{P}_N \left(\potentialNonlin(u_N^j)\right)
$$
and the spatially continuous scheme analyzed in \cite{li2021stability}, given by
$$
\frac{u^{j+1}-u^j}{\tau} = -\nu \Delta^2 u^{j+1} - \theta_c \Delta u^{j+1} + \Delta  \left(\potentialNonlin(u^j)\right).
$$
From Theorem~\ref{thm: FH stability}, we have that $\sup_{j\geq0}\|u_N^j\|_{L^{\infty}}\leq 1-\delta_1$, where $\delta_1\in(0,1)$ depends on $\inputBall,\delta_0,\nu, \theta, \theta_c$.
Note that this scheme is also implicit in the $-\theta_c \Delta u$ term. Under the condition that $\timestep\leq \frac{\nu}{4\theta_c^2}$,  \cite[Theorem 1.1]{li2021stability} guarantees that the iterates $u^j$ are well defined and also satisfy $\sup_{j\geq0}\|u^j\|_{L^{\infty}}\leq 1-\delta_1$. Writing $e^j=u_N^j-u^j$, the error evolution is given by
$$\frac{e^{j+1} - e^j}{\tau} = -\nu\Delta^2 e^{j+1} + \Delta \left( \mathcal{P}_N \potentialNonlin(u_N^j) - \potentialNonlin(u^j) - \theta_c e^j \right).$$
Taking the inner product with $e^{j+1}$, we have
$$\frac{1}{2\tau} \Bigl( \|e^{j+1}\|_{L^2}^2 - \|e^j\|_{L^2}^2 + \|e^{j+1} - e^j\|_{L^2}^2 \Bigr) + \nu \|\Delta e^{j+1}\|_{L^2}^2 = -\Bigl( \mathcal{P}_N \potentialNonlin(u_N^j) - \potentialNonlin(u^j) - \theta_c e^j, \Delta e^{j+1} \Bigr).$$
Via Young's inequality:
$$\frac{1}{2\tau} \Bigl( \|e^{j+1}\|_{L^2}^2 - \|e^j\|_{L^2}^2 \Bigr) + \nu \|\Delta e^{j+1}\|_{L^2}^2 \le \frac{\nu}{2}\|\Delta e^{j+1}\|_{L^2}^2 + \frac{1}{2\nu}\| \mathcal{P}_N \potentialNonlin(u_N^j) - \potentialNonlin(u^j) - \theta_c e^j \|_{L^2}^2,$$
from which it follows that
$$\frac{1}{2\tau} \Bigl( \|e^{j+1}\|_{L^2}^2 - \|e^j\|_{L^2}^2 \Bigr) \le \frac{1}{2\nu}\| \mathcal{P}_N \potentialNonlin(u_N^j) - \potentialNonlin(u^j) - \theta_c e^j \|_{L^2}^2.$$
By the triangle inequality,
$$\| \mathcal{P}_N \potentialNonlin(u_N^j) - \potentialNonlin(u^j) - \theta_c e^j\|_{L^2} \le \| \potentialNonlin(u_N^j) - \potentialNonlin(u^j) \|_{L^2} + \theta_c \|e^j\|_{L^2} + \| (I - \mathcal{P}_N)\potentialNonlin(u^j) \|_{L^2}.$$
Because both $u_N^j$ and $u^j$ are bounded in $L^{\infty}$ by $1-\delta_1$, we have that  $\| \potentialNonlin(u_N^j) - \potentialNonlin(u^j) \|_{L^2} \le L_{\delta_1} \|e^j\|_{L^2}$. Further, because $u^j \in H^5$, $\potentialNonlin(u^j)$ is also bounded in $H^5$, yielding
$$\| (I - \mathcal{P}_N)\potentialNonlin(u^j) \|_{L^2} \le C N^{-5} \|\potentialNonlin(u^j)\|_{H^5} \le C' N^{-5}.$$
This implies
$$\| \mathcal{P}_N \potentialNonlin(u_N^j) - \potentialNonlin(u^j) - \theta_c e^j \|_{L^2}^2 \le 2(L_{\delta_1} + \theta_c)^2 \|e^j\|_{L^2}^2 + 2(C')^2 N^{-10}.$$
Writing 
$K_1 = \frac{2(L_{\delta_1} + \theta_c)^2}{\nu}$ and $K_2 = \frac{2(C')^2}{\nu}$, we get
$$\frac{\|e^{j+1}\|_{L^2}^2-\|e^j\|_{L^2}^2}{\timestep}\leq K_1\|e^j\|_{L^2}^2+K_2N^{-10}.$$
Applying Lemma~\ref{lemma:gronwall lemma 2} with $y_j=\|e^j\|_{L^2}^2,$ $\alpha_j=K_1$, and $\beta_j=K_2N^{-10},$ we deduce that
\begin{align*}
    \|e^j\|_{L^2}^2&\leq \exp\Bigl(\tau jK_1\Bigr)\Bigl(\|e^0\|_{L^2}^2+\timestep j K_2N^{-10}\Bigr) \\
    &\leq \exp\Bigl(T\frac{2(L_{\delta_1} + \theta_c)^2}{\nu}\Bigr)\Bigl(\|(I-\mathcal{I}_N)u_0\|_{L^2}^2+T \frac{(C')^2}{\nu}N^{-10}\Bigr)\\
    & \leq \exp\Bigl(T\frac{2(L_{\delta_1} + \theta_c)^2}{\nu}\Bigr)\Bigl(CN^{-10}+T \frac{(C')^2}{\nu}N^{-10}\Bigr),
\end{align*}
where in the last line we used the fact that $u_0\in H^5.$ Taking the square root of both sides
\begin{align*}
    \|u_N^j-u^j\|_{L^2}\leq C_1e^{C_2\timestep j}N^{-5}.
\end{align*}
Then, \cite[Theorem 1.2]{li2021stability} ensures that under the assumptions given in our theorem statement,  
\begin{align*}
    \|u(T)-u^j\|_{L^2}\leq C_1' e^{C_2'T}\timestep,
\end{align*}
where $C_1',C_2'>0$ depend on $\|u_0\|_{H^5},\delta_0,\nu, \theta, \theta_c.$ The triangle inequality then gives
\begin{align*}
    \|u(T)-u_N^j\|_{L^2}\leq \|u(T)-u^j\|_{L^2}+\|u_N^j-u^j\|_{L^2}\leq C_1e^{C_2T}\bigl(\timestep +N^{-5}\Bigr),
\end{align*}
completing the proof.
    \end{proof}
\end{theorem}

\end{document}